\documentclass[11pt]{article}

\usepackage[utf8]{inputenc} 
\usepackage[T1]{fontenc}    
\usepackage{hyperref}       
\usepackage{url,authblk}            
\usepackage{booktabs}       
\usepackage{amsfonts}       
\usepackage{nicefrac}       
\usepackage{microtype}      
\usepackage{xcolor}         

\usepackage[round]{natbib}
\usepackage{graphicx}
\usepackage{amsmath,amssymb,mathtools,amsthm}
\usepackage{mathrsfs}
\usepackage{dirtytalk}
\usepackage{tikz}
\usepackage{placeins}
\usepackage{csvsimple}
\usepackage{siunitx}
\usepackage{algorithm}
\usepackage{algorithmic}
\newcommand{\R}{\mathbb{R}}
\newcommand{\E}{\mathbb{E}}

\newcommand{\N}{\mathbb{N}}

\newcommand{\norm}[1]{\left\|#1\right\|}

\newcommand{\X}{\mathcal{X}}
\newcommand{\Rfunc}{\mathcal{R}}

\newcommand{\tf}{J}
\newcommand{\Ham}{\mathscr{H}}

\newcommand{\op}{\mathrm{op}}
\newcommand{\nuc}{\mathrm{nuc}}
\newcommand{\Phieps}{\Phi_{\varepsilon}}
\newcommand{\Psieps}{\Psi_{\varepsilon}}
\newcommand{\Ortheps}{\operatorname{Orth}_{\varepsilon}}
\newcommand{\Orth}{\operatorname{Orth}}
\newcommand{\gradmf}{\mathrm{grad}_{\mf}}
\newcommand{\ip}[2]{\left\langle #1, #2 \right\rangle}
\newcommand{\mf}{\mathrm{avg}}

\newcommand{\mPhieps}{\mathbf{\Phi}_{\varepsilon}}
\newcommand{\mPsieps}{\mathbf{\Psi}_{\varepsilon}}
\newcommand{\Law}{\operatorname{Law}}

\newcommand{\argmin}{\operatorname*{argmin}}

\newcommand{\e}{\mathrm e}
\newcommand{\eps}{\varepsilon}
\newcommand{\Z}{\mathcal{Z}}
\newcommand{\cH}{\mathcal{H}}
\newcommand{\cP}{\mathcal{P}}

\newcommand{\Thetagate}{\Theta_{\mathrm{gate}}}
\newcommand{\Thetaexp}{\Theta_{\mathrm{exp}}}
\newcommand{\dd}{\,\mathrm d}
\newcommand{\diag}{\operatorname{diag}}

\newcommand{\rank}{\operatorname{rank}}
\newcommand{\softmax}{\operatorname{softmax}}

\newcommand{\CE}{\mathrm{CE}}
\newcommand{\Gate}{\mathrm{gate}}

\newcommand{\Hlogit}{\cH_{\mathrm{logit}}}
\newcommand{\Haug}{\cH_{\mathrm{aug}}}
\newcommand{\qtot}{\mathsf q_{\Theta}}

\newcounter{lemmano}
\newcounter{theoremno}
\newcounter{propositionno}

\newcounter{corollarynno}

\newcounter{remarkno}

\newtheorem{theorem}[theoremno]{Theorem}
\newtheorem{lemma}[lemmano]{Lemma}
\newtheorem{proposition}[propositionno]{Proposition}

\newtheorem{corollary}[corollarynno]{Corollary}

\newtheorem{remark}[remarkno]{Remark}

\newenvironment{theorem*}{{\bf Lemma:}}

\newtheorem{asu}{Assumption}

\title{Move on Muon : A Hamiltonian probability gradient flow perspective of Muon optimizer}

\author{Aratrika Mustafi$^{*}$, Soumya Mukherjee and Bharath K. Sriperumbudur}

\affil{Department of Statistics, Pennsylvania State University}

\begin{document}

\maketitle

\begin{abstract}
 We develop a gradient flow on the space of probability measures defined on matrix-valued parameters induced by regularized Muon, an analytically smoothed version of the idealized Muon optimizer. The key observation is that the regularized orthogonalization map is the gradient of a smooth Fenchel-dual smoothing of the nuclear norm. This identifies the (regularized) Muon update as a mirror/prox step in the update variable, with momentum acting as the dual coordinate. We use this structure to lift Muon from a single matrix parameter to finite-particle probability objectives of the form $J(\rho)=R\left(\int F d \rho\right)$, a setting motivated by mean-field descriptions of neural-network training, and derive the inertial continuous-time limit. Using this structure, we derive the finite-particle continuous-time limit under the inertial scaling of step size and momentum, and then pass to a phase-space mean-field equation over probability laws on parameter-momentum pairs. The resulting flow can be shown to be a damped Hamiltonian probability dynamics whose kinetic energy is induced by the regularized Muon mirror potential.  We prove an exact Hamiltonian dissipation identity, showing that the Hamiltonian energy decreases monotonically. While the target objective itself need not be monotone along the inertial Muon dynamics, under additional gradient-dominance, bounded-momentum, and curvature/alignment assumptions, we obtain continuous and discrete-time exponential convergence rates for the objective gap. We also study the well-posedness of the mean-field limit equation and establish propagation of chaos guarantees for the interacting particle system. Finally, we extend the formulation to Hilbert-valued feature maps on product matrix spaces, yielding a blockwise Muon probability flow applicable to smooth transformer mixture-of-experts models. 
\end{abstract}

\section{Introduction}
Optimization methods for deep neural networks often succeed because they exploit structure that is not visible in the scalar-coordinate view of the parameters. As model sizes grow and parameter blocks become highly structured, the geometry imposed by the optimizer can substantially influence both stability and speed. This is particularly visible for matrix-valued parameters, where the Frobenius geometry used by standard gradient methods is only one possible choice.
Standard first-order optimizers such as SGD, Adam, and their variants treat these parameters largely through coordinatewise Euclidean updates. While this has been extremely effective in practice, recent work has shown that matrix geometry can play a more explicit role in optimization, particularly when the update direction is constrained or normalized according to spectral information.

The Muon optimizer introduced by \cite{jordan2024muon} is a recent and influential example of this matrix-geometric viewpoint. At a high level, Muon maintains a momentum variable for each matrix parameter and updates the parameter in an orthogonalized version of that momentum direction. Empirical studies have shown that such orthogonalized matrix updates can be competitive in language-model training, and recent large-scale implementations have investigated the extent to which Muon can be made practical for LLM training through suitable scaling, weight decay, and implementation choices \citep{liu2025muonscalablellmtraining}.
 In the idealized form, as considered in this paper, the orthogonalization map keeps the singular vectors of the momentum matrix and replaces its nonzero singular values by one. Thus, rather than moving directly in the raw momentum direction, Muon moves in a polar-factor direction. This operation makes the update insensitive to the scale of the singular values but highly sensitive to their singular subspaces. Practical implementations approximate this orthogonalization numerically, for instance using Newton-Schulz iterations, but the exact polar-factor map captures the central geometric mechanism of the method.

 These developments motivate a theoretical question : \textit{\textbf{What continuous-time and probability-space dynamics are naturally associated with a momentum optimizer whose update direction is produced by spectral orthogonalization?}
}

A first answer is already suggested by convex geometry. The polar-factor direction is not an arbitrary normalization: it is a steepest descent direction over a spectral-norm unit ball. Equivalently, the ideal Muon step solves a linear minimization problem under an operator-norm constraint. This perspective has been developed in recent work on non-Euclidean trust-region interpretations of gradient orthogonalization and norm-constrained linear minimization oracles \citep{kovalev2025understanding,pethick2025training}. Related viewpoints also connect Muon-type methods to implicit spectral constraints and broader families of spectral optimizers \citep{chen2025muonoptimizes}.

This paper starts from this existing observation but takes a different route. 
We take the spectral-norm trust-region interpretation of hard Muon as a starting point and develop from it a probabilistic gradient-flow perspective on Muon with momentum.

The key technical obstacle is that the hard orthogonalization map on the momentum variable $P$, \(P\mapsto \operatorname{Orth}(P)\) is nonsmooth. Here, $\operatorname{Orth}(P) = U V^{\top}$ with $ P= U \Sigma V^{\top}, \Sigma = \diag(\sigma_1(P), \cdots, \sigma_q(P))$, is the compact/reduced SVD of $P$.

This nonsmoothness is not merely a technical inconvenience: it changes the nature of the continuous-time limit, replacing an ordinary differential equation by a differential inclusion.  To obtain a well-defined smooth flow, we introduce a regularized Muon map
\[
        \operatorname{Orth}_\varepsilon(P)
        =
        U \operatorname{diag}\!\left(
        \frac{\sigma_i(P)}{\sqrt{\sigma_i(P)^2+\varepsilon^2}}
        \right)V^\top,
        \qquad \varepsilon>0,
\]
which keeps the singular directions of \(P\)  but replaces the hard saturation of singular values by a smooth saturation. As \(\varepsilon\downarrow 0\), this map converges pointwise to the hard polar-factor map on fixed-rank matrices. The role of the present regularization is to enable an exact variational representation of the orthogonalization operation needed for the Hamiltonian probability dynamics developed in the next section. The regularized Muon update then becomes a genuine mirror/prox step in the update or velocity variable, with the momentum serving as the corresponding dual coordinate. This type of regularization has been recently used as a technical tool in \cite{kim2026sharpcapacityscalingspectral}.

The second goal of the paper is to lift this mirror interpretation from a single matrix update to a probability-space description. We consider functionals of the form
\[
J(\rho)=R\!\left(\int F(W)\,d\rho(W)\right),
\]
where $\rho$ is a probability measure over matrix-valued parameters (or tuples of matrix-valued parameters).

This includes finite-particle objectives obtained from empirical measures, and it also provides a convenient mean-field perspective for studying populations of parameters. For an $N$-particle empirical law, the regularized Muon update induces a coupled particle system in the phase variables $(W_i,P_i)$. Under the inertial scaling of step-sizes, the discrete dynamics converge to finite-particle ODEs, leading to a McKean-Vlasov continuity equation for $\mu_t=\mathrm{Law}(W_t,P_t)$.
The resulting probability flow  has a damped Hamiltonian flow structure, instead of being an ordinary Wasserstein gradient flow. Hamiltonian-flow formalisms have been used in literature to study momentum-based acceleration on the space of probability measures \citep{wang2022accelerated, chen2025accelerating}. In our setting, the Hamiltonian structure is specific to Muon and is induced by the mirror map associated with regularized nuclear-norm smoothing.
The resulting flow satisfies an exact Hamiltonian dissipation identity, which clarifies why the objective $J$ itself need not decrease monotonically. As in other inertial or accelerated systems, energy can move between the objective and the momentum variable, while the total damped Hamiltonian decreases. Recent work  \cite{peyre2026muondynamicsspectralwasserstein} considers specral Wasserstein flow whose ODE limit corresponds to a momentum-free Muon update with an extra nuclear norm scaling. Our Hamiltonian probability-flow formulation is closer to the version of Muon largely implemented in practice since it preserves the momentum variable, treats momentum as the dual coordinate of a smooth mirror map, and recovers the hard polar Muon update as \(\varepsilon \downarrow 0\).

We then extend the formulation from a single matrix space to finite product spaces of matrix blocks. This allows the same mirror geometry to act on all blocks of structured models. This is important for neural network layers whose parameters consist of several matrices, and especially for mixture-of-experts models with both expert and router parameters. In the extended product space \(\Theta\), the regularized Muon potential is block-separable, and the corresponding orthogonalization (mirror) map \(\operatorname{Orth}_\varepsilon^\Theta\) applies the spectral regularization to each matrix block in a separable manner. This gives a product-space mirror step, a product-space Hamiltonian, and a corresponding dissipation identity. This generalization to product spaces allows us to analyze equally weighted mixture-of-experts where each expert model is a neural network or a transformer, as well as smooth-routing/selection variants of such mixture-of-experts. For transformer mixture-of-experts models, input-dependent routing is encoded inside the Hilbert-valued feature map $F$. Smooth unnormalized gates and softmax-normalized gates fit directly into the framework, while exact hard top-$k$ routing lies outside the smooth theory and can be treated as a nonsmooth limiting case. This distinction is consistent with prior work on sparse mixture-of-experts routing and noisy router smoothing, emphasizing the discontinuity of sparse routing and the smoothing effect of noisy or softened routers \citep{shazeer2017outrageously, fedus2021switch, chen2022towards,vaswani2017attention}.

\section{ Muon as a mirror regularized
trust-region step}

In this section, we isolate the discrete-time Muon update for a single matrix variable and recast it in a form that admits a smooth mirror-map regularization. This reformulation interprets the orthogonalization step in the Muon update as the solution to a variational problem on matrix space. The resulting variational structure then enables the development of a Hamiltonian formulation of the Muon dynamics in later sections.

Let \(\mathcal{X}=\mathbb{R}^{m \times n}\) with Frobenius pairing \(\langle A, B\rangle_{F}= \operatorname{tr}\left(A^{\top} B\right)\), and let \(q=\min (m, n)\). For \(P=U \Sigma V^{\top}\) define, 
\[
\operatorname{Orth}(P)=U V^{\top}, \quad \operatorname{Orth}(0)=0 .
\]

The idealized Muon update for a smooth objective $F$ : $\mathcal{X} \rightarrow \mathbb{R}$ is given by,
\begin{equation}
\begin{aligned}
P_{k+1} & =\beta P_k+(1-\beta) \nabla F\left(W_k\right), \\
W_{k+1} & =W_k-\eta \operatorname{Orth}\left(P_{k+1}\right) ,
\end{aligned}
\end{equation}
where \(P\) denotes the momentum variable and \(W\) denotes the spatial variable. The following proposition demonstrates the convex-analytic structure behind this update.
\begin{proposition}[Spectral trust-region form of Muon] \label{Proposition 1} For every $P \in \mathcal{X}$,
\begin{equation}\begin{aligned}\label{Constrained optimization form of Orthogonalization operation}
\|P\|_{\text {nuc }} =\sup _{\|G\|_{\mathrm{op}} \leq 1}\langle P, G\rangle_F, \quad
-\operatorname{Orth}(P)  \in \underset{\|G\|_{\mathrm{op}} \leq 1}{\operatorname{argmin}}\langle P, G\rangle_F .
\end{aligned}
\end{equation}
Equivalently, if $\Phi_0(G)=\iota_{\left\{\|G\|_{\mathrm{op}} \leq 1\right\}}(G)$ is the $0/\infty$-indicator of the spectral-norm unit ball, then
\[\label{spectral trust-norm of Muon update}
\begin{aligned}
G_{k+1} & \in \underset{G}{\operatorname{argmin}}\left\{\left\langle P_{k+1}, G\right\rangle_F+\Phi_0(G)\right\}, \\
W_{k+1} & =W_k+\eta G_{k+1} .
\end{aligned}
\]
recovers ideal Muon after choosing the canonical minimizer $G_{k+1}=-\operatorname{Orth}\left(P_{k+1}\right)$.
\end{proposition}

The above proposition says that Muon linearizes through the momentum \(P_{k+1}\) and then takes the steepest direction allowed by a spectral norm unit ball. The non-smoothness enters since \(\Phi_0^*=\|\cdot\|_{\text {nuc }}\) is not differentiable at rank-deficient matrices. To regularize this singular structure and obtain a smooth dynamics we use the scalar Fenchel (convex) conjugate pair of functions $\psi_{\varepsilon}(a) =\sqrt{a^2+\varepsilon^2}-\varepsilon$ and $\phi_{\varepsilon}(b) =\varepsilon\left(1-\sqrt{1-b^2}\right)+\iota_{[-1,1]}(b)$
for \(\varepsilon>0\). Lifting through singular values to the matrix space gives the corresponding Fenchel conjugates
\begin{equation}\label{Fenchel conjugates}
\begin{aligned}
\Psi_{\varepsilon}(P)=\sum_{i=1}^q\left(\sqrt{\sigma_i(P)^2+\varepsilon^2}-\varepsilon\right), \qquad \Phi_{\varepsilon}(G)= \begin{cases}\varepsilon \sum_{i=1}^q\left(1-\sqrt{1-\sigma_i(G)^2}\right), & \|G\|_{\mathrm{op}} \leq 1, \\
+\infty, & \text { otherwise } .\end{cases}
\end{aligned}
\end{equation}
Then \(\Phi_{\varepsilon}^*=\Psi_{\varepsilon}\) and \(\Psi_{\varepsilon}^*=\Phi_{\varepsilon}\). If \(P=U \operatorname{diag}\left(\sigma_1, \ldots, \sigma_s\right) V^{\top}\), we define
\begin{equation}\label{regularized Orth}
\operatorname{Orth}_{\varepsilon}(P)=U \operatorname{diag}\left(\frac{\sigma_i}{\sqrt{\sigma_i^2+\varepsilon^2}}\right)_{i=1}^s V^{\top} .
\end{equation}

The map is \(1 / \varepsilon\)-Lipschitz in Frobenius norm, \(\left\|\operatorname{Orth}_{\varepsilon}(P)\right\|_{\mathrm{op}}<1\) and \(\operatorname{Orth}_{\varepsilon}(P) \rightarrow \operatorname{Orth}(P)\) as \(\varepsilon \downarrow 0\) for every fixed \(P\).
The pair \((\phi_{\varepsilon}, \psi_{\varepsilon})\) should be viewed as a spectral analogue of a smooth saturation. At the scalar level, \(a \mapsto \psi_{\varepsilon}^{\prime}(a)=\frac{a}{\sqrt{a^2+\varepsilon^2}}\) maps the real line into \((-1,1)\) and tends pointwise to the sign map as \(\varepsilon \downarrow 0\). Thus \(\operatorname{Orth}_{\varepsilon}\) keeps the singular directions of \(P\) and replaces each singular value by a softened value in \([0,1)\). The conjugate \(\phi_{\varepsilon}\) keeps the spectral-ball domain, but it replaces the hard indicator by a smooth barrier-like penalty inside the ball. This is precisely the structure needed to pass between the nonsmooth trust-region interpretation and a smooth mirror map. Further, due to the identity \(\nabla\Psieps(P)
    =
    U\diag\left(
    \frac{\sigma_1(P)}{\sqrt{\sigma_1(P)^2+\eps^2}},
    \ldots,
    \frac{\sigma_s(P)}{\sqrt{\sigma_s(P)^2+\eps^2}}
    \right)V^\top=\operatorname{Orth}_{\varepsilon}(P)\), the soft-orthogonalization operation is exactly a smooth mirror-map, with its inverse on the open unit spectral-norm ball being exactly \(\nabla\Phieps(G)
    =
    \widetilde U\diag\left(
    \frac{\eps\sigma_1(G)}{\sqrt{1-\sigma_1(G)^2}},
    \ldots,
    \frac{\eps\sigma_s(G)}{\sqrt{1-\sigma_s(G)^2}}
    \right)\widetilde V^\top\), with reduced SVD \(G=\widetilde U\diag(\sigma_1(G),\ldots,\sigma_s(G))\widetilde V^\top\) where $s=\rank(G)$. Hence the regularized direction is not chosen by an ad hoc smoothing of \(\operatorname{Orth}\); it is the primal minimizer associated with the Fenchel-dual kinetic potential \(\Psi_{\varepsilon}\). The same conjugate pair becomes the kinetic term and the mirror geometry in the Hamiltonian probability flow later on.
\begin{proposition}[Smooth regularized Muon step]\label{Proposition 2}
    For every \(\varepsilon>0\), the problem
\begin{equation}\label{regularized G}
G_{\varepsilon}(P)=\underset{G}{\operatorname{argmin}}\left\{\langle P, G\rangle_F+\Phi_{\varepsilon}(G)\right\}
\end{equation}
has the unique solution \(G_{\varepsilon}(P)=-\operatorname{Orth}_{\varepsilon}(P)\). Hence the regularized Muon update is
\begin{equation}
\begin{aligned}
P_{k+1} & =\beta P_k+(1-\beta) \nabla F\left(W_k\right), \\
W_{k+1} & =W_k-\eta \operatorname{Orth}_{\varepsilon}\left(P_{k+1}\right) .
\end{aligned}
\end{equation}

Moreover, if \(G_{\varepsilon, k}=-\operatorname{Orth}_{\varepsilon}\left(P_k\right)\), then
\(
G_{\varepsilon, k+1}=\underset{G}{\operatorname{argmin}}\left\{\left\langle P_{k+1}-P_k, G\right\rangle_F+D_{\Phi_{\varepsilon}}\left(G, G_{\varepsilon, k}\right)\right\} .
\)
\end{proposition}
\begin{remark}[Why regularization matters]
The hard \(\operatorname{map} P \mapsto \operatorname{Orth}(P)\) is a selected element of \(\partial\|P\|_{\text {nuc }}\) and is set-valued at rank-deficient momenta. The smooth family above preserves the spectral trust-region geometry while giving a single-valued Lipschitz vector field. This is the ingredient that makes the ODE, PDE, and propagation-of-chaos arguments standard rather than differential-inclusion arguments.
\end{remark}
\textbf{Three equivalent views of the same update}. For later use it is helpful to keep three interpretations in parallel. The hard step is a spectral trust-region steepest descent direction. The nonsmooth mirror form says that the same step is generated by the indicator of the spectral unit ball, whose conjugate is the nuclear norm. The regularized step says that Muon is the zero-temperature limit of a smooth mirror family. These are not competing descriptions: the trust-region view explains the geometry, the mirror view provides the variational update, and the smooth Fenchel pair provides the analytic regularity needed for continuous-time and mean-field limits.

\textbf{Why not regularize by adding \(\varepsilon\|G\|_F^2\) ?} A Euclidean quadratic regularization would also make the direction unique, but it would change the saturation geometry of Muon. The construction above preserves the spectral unit-ball domain and only smooths the singular-value saturation. Consequently, the limit $\varepsilon \downarrow 0$ returns the canonical polar-factor direction rather than a Euclidean steepest-descent direction. This is the reason the regularization is tied to a Fenchel pair rather than introduced as a generic numerical smoothing.

\section{Probability lift and finite particle Muon}
We now minimize a general class of functionals defined over the space of probability measures on \(\mathcal{X}\) denoted as \(\mathcal{P}(\mathcal{X})\) using the dynamics induced by (regularized) Muon. Let
\[
J(\rho)=R\left(\int_{\mathcal{X}} F(W) \mathrm{d} \rho(W)\right), \quad \rho \in \mathcal{P}(\mathcal{X}),
\]
where \(F: \mathcal{X} \rightarrow \mathbb{R}\) and \(R: \mathbb{R} \rightarrow \mathbb{R}\). For particles \(W= \left(W_1, \ldots, W_N\right)\) define the empirical law \(\rho_W^N=\frac{1}{N} \sum_{i=1}^{N} \delta_{W_i}\) and the empirical particle lift
\[
J_N(W)\coloneqq J\left(\rho_W^N\right)=R\left(\frac{1}{N} \sum_{i=1}^N F\left(W_i\right)\right) .
\]
The natural product geometry over particles is the mean-field pairing
\(
\langle U, V\rangle_{\mathrm{avg}}=\frac{1}{N} \sum_{i=1}^N\left\langle U_i, V_i\right\rangle_F .
\)
The mean-field pairing makes the empirical objective \(J_N=J\left(\rho_W^N\right)\) an intensive energy and makes \(\operatorname{grad}_{\text {avg }} J_N\) coincide with the particle discretization of the Wasserstein force and keeps the probability objective and particle geometry aligned.

The class \(\tf(\rho)=\Rfunc\left(\int F \mathrm{~d} \rho\right)\) is deliberately simple, yet sufficiently expressive to capture population objectives in which a distribution over parameters induces an averaged feature or prediction. For simplicity, we restrict attention in the main text to real-valued functions \(F\). However, ML applications involving neural networks or transformers naturally require extensions to vector-valued outputs and inputs defined on product spaces of matrices. The results and proofs extend to these settings in a standard manner.

For deriving a well-posed dynamics, we need to impose some standard regularity conditions.
\begin{asu}[Basic smoothness]\label{ass:Basic smoothness}\(F \in C^1(\mathcal{X}), \nabla F\) is globally Lipschitz, and \(\|\nabla F(W)\|_F \leq M_F\) for all \(W\). The derivative \(R^{\prime}\) is globally Lipschitz with constant \(L_R\) and is globally bounded with constant $M_R$.
\end{asu}
Under Assumption ~\ref{ass:Basic smoothness}, the probability functional $\tf$ admits a well-defined first variation and Wasserstein gradient that drives the Muon dynamics. 

\begin{proposition}[First variation and particle gradient]\label{Proposition 3}
    Let \(m_\rho=\int F(W) \mathrm{d} \rho(W)\). A valid first variation of \(J\) and its Wasserstein gradient \(\nabla_{W_2} J(\rho)(W)\) are given by
\begin{equation}\label{Wasserstein gradient of J_N}
\begin{aligned}
\frac{\delta J}{\delta \rho}(\rho)(W) =R^{\prime}\left(m_\rho\right) F(W), \qquad
\nabla_{W_2} J(\rho)(W) =R^{\prime}\left(m_\rho\right) \nabla F(W) .
\end{aligned}
\end{equation}
For \(F^N(W)=\frac{1}{N} \sum_{j=1}^{N} F\left(W_j\right)\), the gradient of \(J_N\) under the natural product geometry is
\begin{equation}\label{gradient of J_N}
\begin{aligned}
\operatorname{grad}_{\text {avg }} J_N(W)  =a(W), \quad
a_i(W) =R^{\prime}\left(F^N(W)\right) \nabla F\left(W_i\right) .
\end{aligned}
\end{equation}
\end{proposition}
The regularized Muon mirror potentials on \(\mathcal{X}^N\) are block-separable:
\[
\Phi_{\varepsilon}^N(G)=\frac{1}{N} \sum_{i=1}^{N} \Phi_{\varepsilon}\left(G_i\right), \quad \Psi_{\varepsilon}^N(P)=\frac{1}{N} \sum_{i=1}^{N} \Psi_{\varepsilon}\left(P_i\right) .
\]

With respect to \(\langle\cdot, \cdot\rangle_{\text {avg }}\) they remain Fenchel conjugates. The exact finite-particle regularized Muon scheme is therefore
\begin{equation}\label{Finite particle regularized Muon scheme}
\begin{aligned}
P_{i, k+1} & =\beta P_{i, k}+(1-\beta) R^{\prime}\left(F_k\right) \nabla F\left(W_{i, k}\right), \\
W_{i, k+1} & =W_{i, k}-\eta \operatorname{Orth}_{\varepsilon}\left(P_{i, k+1}\right),
\end{aligned}
\end{equation}
where \(F_k=\frac{1}{N} \sum_{j=1}^{N} F\left(W_{j, k}\right)\).
The scheme is exact for the regularized mirror problem: no approximation has been made beyond the choice of \(\varepsilon>0\). The only coupling among particles is through the empirical scalar \(F_k\), while the orthogonalization is blockwise and particlewise. This separation is central to the mean-field analysis. The interaction enters through the force, while the non-Euclidean geometry enters through the kinetic mirror map \(P \mapsto \operatorname{Orth}_{\varepsilon}(P)\).

A useful way to understand Equation \eqref{Finite particle regularized Muon scheme} is as a relaxation system. The momentum \(P_{i, k}\) is a moving average of the current first-variation force, and the position is transported by the mirror-dual velocity generated by that momentum. The limiting ODE derived below preserves this two-time-scale structure.
\section{Continuous-time and Hamiltonian probability dynamics}
To analyze the continuous-time dynamics, we choose the inertial scaling
\begin{equation}\label{inertial scaling}
\eta_h=h, \quad \beta_h=1-\gamma h+r_h, \quad r_h / h \rightarrow 0,
\end{equation}
with \(\gamma>0\). The state variable is the matrix and momentum pair \(Y=(W, P)\) belonging to the state space \(\mathcal{Z}\coloneqq \mathcal{X} \times \mathcal{X}\). Let \(\mathcal P_1(\mathcal Z)\) denote the space of probability measures on
\(\mathcal Z\) with finite first moment, equipped with the \(W_1\)-Wasserstein topology. Under the stated regularity conditions, we derive the continuous-time evolution both at the particle as well as the distributional law level for the state-space variables.
\begin{theorem}[Finite-particle ODE limit and phase space PDE]\label{McKean-Vlasov formulation}
Under Assumption \ref{ass:Basic smoothness}, for fixed \(N\) and fixed initial data, the piecewise linear interpolation of Equation \eqref{Finite particle regularized Muon scheme} converges uniformly on every $[0, T]$ to the unique global solution of
\begin{equation}\label{Finite particle ODE limit}
\begin{aligned}
\dot{W}_i(t) & =-\operatorname{Orth}_{\varepsilon}\left(P_i(t)\right) \\
\dot{P}_i(t) & =\gamma\left(R^{\prime}\left(F^N(t)\right) \nabla F\left(W_i(t)\right)-P_i(t)\right),
\end{aligned}
\end{equation}
where \(F^N(t)=\frac{1}{N} \sum_{j=1}^N F\left(W_j(t)\right)\). If \(r_h=O\left(h^2\right)\), the convergence rate is \(O(h)\) on finite horizons.

Let \(Y_t=\left(W_t, P_t\right) \in \mathcal{Z}= \mathcal{X} \times \mathcal{X}\), let \(\mu_t=\operatorname{Law}\left(Y_t\right)\), let \(\rho_t=\left(\pi_W\right)_{\#} \mu_t\), and set \(a_t(W)=\Rfunc^{\prime}\left(\int F \mathrm{~d} \rho_t\right) \nabla F(W)\). The mean-field law solves
\begin{equation}\label{Mc-Kean Vlasov equation}
\partial_t \mu_t+\nabla_W \cdot\left(-\operatorname{Orth}_{\varepsilon}(P) \mu_t\right)+\nabla_P \cdot\left(\gamma\left(a_t(W)-P\right) \mu_t\right)=0 .
\end{equation}
The weak solution is unique in \(C\left([0, T] ; \mathcal{P}_1(\mathcal{Z})\right)\) under the corresponding global Lipschitz hypotheses, and under localized hypotheses as long as the trajectory remains in the bounded region on which the constants are finite.
\end{theorem}
Equation \eqref{Mc-Kean Vlasov equation} is the probability-flow counterpart of Muon with momentum. The position velocity is the negative regularized orthogonalized momentum and the momentum velocity relaxes toward the Wasserstein force $a_t$.
\begin{remark}[Role of scaling]
  If \(\beta\) is fixed as \(h \downarrow 0\), the momentum variable relaxes on a fast scale and the second-order structure is lost. The scaling chosen in Equation \eqref{inertial scaling} is therefore part of the model and allows a second-order flow dynamics instead of a singular overdamped limit, not merely a technical convenience.
\end{remark}

We now express the evolution PDE in Theorem \ref{McKean-Vlasov formulation} as an exact regularized Muon Hamiltonian probability flow.
\begin{theorem}[Damped Hamiltonian structure]\label{Damped Hamiltonian structure}
 Define
\begin{equation}\label{Hamiltonian energy}
\Ham_{\varepsilon, \gamma}(\mu)=\int_{\mathcal{Z}} \Psi_{\varepsilon}(P) \mathrm{d} \mu(W, P)+\gamma \Rfunc\left(\int_{\mathcal{Z}} F(W) \mathrm{d} \mu(W, P)\right)
\end{equation}
Then Equation \eqref{Mc-Kean Vlasov equation} is equivalent to
\begin{equation}\label{Hamiltonian equation}
\begin{aligned}
\partial_t \mu_t+\nabla_W \cdot  \left(\mu_t\left[-\nabla_P \frac{\delta \Ham_{\varepsilon, \gamma}}{\delta \mu_t}\right]\right)  +\nabla_P \cdot\left(\mu_t\left[\nabla_W \frac{\delta \Ham_{\varepsilon, \gamma}}{\delta \mu_t}-\gamma P\right]\right)=0
\end{aligned}
\end{equation}
Moreover, along solutions,
\begin{equation}\label{Hamiltonian dissipation}
\frac{\mathrm{d}}{\mathrm{~d} t} \Ham_{\varepsilon, \gamma}\left(\mu_t\right)=-\gamma \int_{\mathcal{Z}}\left\langle P, \operatorname{Orth}_{\varepsilon}(P)\right\rangle_{\mathrm{F}} \mathrm{~d} \mu_t(W, P) \leq 0
\end{equation}
The integrand has the singular-value form \(\left\langle P, \operatorname{Orth}_{\varepsilon}(P)\right\rangle_{F}=\sum_{r=1}^q \frac{\sigma_r(P)^2}{\sqrt{\sigma_r(P)^2+\varepsilon^2}}\)
and vanishes if and only if $P=0$.
\end{theorem}
The theorem explains why the objective \(J\left(\rho_t\right)\) alone need not decrease monotonically, but the Hamiltonian descent drives the \say{potential energy} component $\tf$,  which is the target functional, to decrease as the dynamics progress through the relation in Equation \ref{Hamiltonian energy}. Momentum stores \say{kinetic energy} $\int \Psieps d \mu$, and the dissipated quantity is the Hamiltonian defined in Equation \ref{Hamiltonian energy}.
\begin{remark}
    [Dissipation as a Fenchel coupling] The integrand in Equation \eqref{Hamiltonian dissipation} is 
\[d_{\varepsilon}(P)=\left\langle P, \operatorname{Orth}_{\varepsilon}(P)\right\rangle_F=\Psi_{\varepsilon}(P)+\Phi_{\varepsilon}\left(\operatorname{Orth}_{\varepsilon}(P)\right),\]
where the second equality is  Fenchel equality. Therefore the dissipated quantity is not simply the kinetic energy \(\int \Psi_{\varepsilon} d \mu\), it is the coupling between the momentum and the mirror-dual velocity. For the chosen spectral potential,
\(
0 \leq d_{\varepsilon}(P)-\Psi_{\varepsilon}(P) \leq q \varepsilon, \quad d_{\varepsilon}(P) \leq\|P\|_{\text {nuc. }}
\)
Hence \(d_{\varepsilon}\) is asymptotically equivalent to the nuclear-norm kinetic energy as \(\varepsilon \downarrow 0\), but it has the stronger property that it vanishes if and only if the momentum vanishes.
\end{remark}
\section{Convergence, particle approximation, and hard Muon}\label{Convergence, particle approximation, and hard Muon}
Let \(a_t(W)=R^{\prime}\left(m_t\right) \nabla F(W)\), \(J_{\star}=\inf_{\rho\in\cP_1(\X)}\tf(\rho)\) and define
\[
\begin{aligned}
A_t & =\int\left\|a_t(W)\right\|_F^2 \mathrm{~d} \mu_t, & K_t & =\int \Psi_{\varepsilon}(P) \mathrm{d} \mu_t &D_t & =\int\left\langle P,\operatorname{Orth}_{\varepsilon}(P)\right\rangle_F \mathrm{~d} \mu_t, & \\ & U_t=J\left(\rho_t\right)-J_{\star} 
&H_t & =K_t+\gamma U_t, & C_t & =\int\left\langle a_t(W), P\right\rangle_F \mathrm{~d} \mu_t.
\end{aligned}
\]
The alignment \(C_t\) is the term that records whether momentum is consistent with descent.
The basic Hamiltonian identity gives \(H_t^{\prime}=-\gamma D_t\), but this alone does not control the objective gap because \(D_t\) controls momentum rather than force. The force norm \(A_t\) is connected to the gap by the PL assumption, while the momentum dissipation is connected to the kinetic energy by coercivity. The bridge between them is the alignment \(C_t\). Positive alignment means that the momentum points in a descent-compatible direction, negative alignment records a transient inertial mismatch.
For this reason the convergence proof uses the modified Lyapunov functional
\(
L_t=H_t-\alpha C_t .
\) The role of \(\alpha\) is to reward descent-compatible alignment without letting \(C_t\) dominate the Hamiltonian. The upper-gradient condition and kinetic coercivity imply \(\left|C_t\right| \leq M_C H_t\), so \(L_t\) is equivalent to \(H_t\) when \(\alpha M_C<1\). The curvature assumption then guarantees that the derivative of \(C_t\) introduces the force norm \(A_t\) up to a controlled remainder. This is the mechanism behind the exponential rate below.
\begin{asu}[Trajectory-level convergence hypotheses]\label{ass:Trajectory-level convergence hypotheses} Along the trajectory, \(U_t \geq 0\) and the following hold.\\
(a) Bounded-momentum kinetic coercivity: for constants \(\kappa_K, \kappa_D>0, \chi \geq 1\) and \(L_G<\infty\),
\[
\begin{aligned}
\Psi_{\varepsilon}(P)  &\geq \frac{\kappa_K}{2}\|P\|_F^2, 
&\left\langle P, \operatorname{Orth}_{\varepsilon}(P)\right\rangle_F  &\geq \kappa_D\|P\|_F^2, \\
\Psi_{\varepsilon}(P)  &\leq \chi\left\langle P, \operatorname{Orth}_{\varepsilon}(P)\right\rangle_F, 
&\left\|\operatorname{Orth}_{\varepsilon}(P)\right\|_F &\leq L_G\|P\|_F .
\end{aligned}
\]
For the regularized Muon potential these inequalities hold on \(\|P\|_F \leq B_P\) with \(\kappa_K=\varepsilon^2 /\left(B_P^2+\varepsilon^2\right)^{3 / 2}\), \(\kappa_D=\left(B_P^2+\varepsilon^2\right)^{-1 / 2}, L_G=1 / \varepsilon\), and \(\chi=1\).\\
(b) Functional PL and upper-gradient bounds: \(A_t \geq 2 \lambda U_t\) and \(A_t \leq 2 \Lambda U_t\) for constants \(\lambda>0, \Lambda<\infty\).\\
(c) Curvature control: \(C_t^{\prime}=\gamma A_t-\gamma C_t+S_t\) with \(\left|S_t\right| \leq \sigma D_t\).    
\end{asu}
\begin{theorem}[Exponential convergence under explicit assumptions]\label{Theorem:Exponential convergence under explicit assumptions} Let Assumptions \ref{ass:Basic smoothness} and \ref{ass:Trajectory-level convergence hypotheses} hold true.
 Let \(M_C=\sqrt{\frac{\Lambda}{\gamma \kappa_K}}\).
Choose \(r \in(0,2)\) and \(\alpha>0\) such that \(\alpha M_C<1, \quad d_{\alpha, r}\coloneqq \gamma-\alpha \sigma-\frac{\alpha \gamma}{2 r \kappa_D}>0\). Finally, let us set 
\(c_{\alpha, r} \coloneqq \frac{1}{1+\alpha M_C} \min \left\{\frac{d_{\alpha, r}}{\chi}, 2 \lambda \alpha\left(1-\frac{r}{2}\right)\right\}\). Then,
\begin{equation}\label{Disipation rate}
J\left(\rho_t\right)-J_{\star} \leq \frac{e^{-c_{\alpha, r} t}}{\gamma\left(1-\alpha M_C\right)}\left(H_0-\alpha C_0\right) .
\end{equation}
\end{theorem}
The proof in Appendix \ref{app:proof-thm3-corollary} obtains Theorem \ref{Theorem:Exponential convergence under explicit assumptions} as the one-block scalar-output specialization of the Hilbert-domain Theorem \ref{thm:convergence-gated}.
\begin{remark}[Interpretation of assumptions]
    The theorem separates geometry from objective landscape. Kinetic coercivity is a property of the regularized Muon potential on bounded momentum sets. PL and upper-gradient conditions are properties of \(J\) along the probability trajectory. As observed in practice and demonstrated in our numerical experiments, the exponential convergence behavior holds true. Precise sufficient conditions for the validity of the PL and upper-gradient conditions for $J$ for specific models of interest can be determined.
\end{remark}
\begin{remark}
    The constant \(c_{\alpha, r}\) is the minimum of a momentum-dissipation contribution and a force-gap contribution. Increasing \(\alpha\) helps expose the PL force term, but too large an \(\alpha\) either destroys equivalence of \(L_t\) and \(H_t\) or makes the dissipation tradeoff negative. The parameter \(r\) is the Young-inequality split between force and momentum. These constraints are not artifacts of notation, they express the fact that accelerated dynamics converge only when momentum alignment, curvature, and dissipation are balanced.
\end{remark}
For finite $N$, the particle update Equation \eqref{Finite particle regularized Muon scheme} gives an interacting approximation of the nonlinear characteristic system Equation \eqref{Finite particle ODE limit}. The next theorem states the quantitative mean-field consistency result under a second-moment assumption on the initialization law of \(W,P\).
\begin{theorem}[Propagation of chaos]\label{Propagation of chaos}
    Let Assumption \ref{ass:Basic smoothness} hold.  Equivalently, after identifying
    \(\Theta=\mathcal{X}\) and \(\mathcal{H}=\mathbb{R}\), the scalar data satisfy the global Hilbert smoothness Assumption \ref{ass:global-hilbert}. Let \(\mu_0 \in \mathcal{P}_2(\mathcal{Z})\). Couple the \(N\)-particle ODE Equation \eqref{Finite particle ODE limit}) with i.i.d. nonlinear mean-field copies having law \(\mu_t\) and the same initial data. Then, for every \(T<\infty\), there is \(C_{\mathrm{poc}}(T, \varepsilon)<\infty\) such that, for every fixed \(i\),
\begin{equation}\label{eq:POC}
\sup _{t \leq T} \mathbb{E}\left[\left\|W_i^N(t)-\bar{W}_i(t)\right\|_F^2+\left\|P_i^N(t)-\bar{P}_i(t)\right\|_F^2\right] \leq \frac{C_{\mathrm{poc}}(T, \varepsilon)}{N}
\end{equation}
Consequently every fixed \(k\) particles converge in law to \(\mu_t^{\otimes k}\), uniformly on finite horizons.
\end{theorem}

\begin{remark}[Mean-field relevance] Theorem \ref{Propagation of chaos} justifies using the nonlinear PDE as the large-population limit of the regularized Muon particle system. The constant depends on \(1 / \varepsilon\) through the Lipschitz constant of \(\operatorname{Orth}_{\varepsilon}\), which is expected since the hard Muon map is not Lipschitz.    
\end{remark}
Finally, the regularized flow has a compactness limit as \(\varepsilon \downarrow 0\). Let \(\Psi_0(P)=\|P\|_{\text {nuc }}\). We discuss this in detail in the Appendix (Theorem \ref{Hard Muon limit theorem}).
\section{Transformer mixture-of-experts (MoE) optimization using Hamiltonian probability flow}
In this section, we discuss how the Hamiltonian probability flow formulation can be extended to optimize ML models with tuples of matrix-valued parameters, such as Transformer mixture-of-experts (MoE). A shallow transformer MoE uses a product parameter space
\(
\Theta=\Theta_{\exp } \times \Theta_{\text {gate }}=\prod_{b=1}^B \mathbb{R}^{m_b \times n_b}
\)
where the expert blocks may include \(Q, K, V, O, W_1, W_2\) corresponding to the Query, Key, Value and Output projection matrices in the attention module, together with the weight matrices in the FFN layer and the gate blocks contain router parameters for the expert routing scheme. The Muon mirror map is applied blockwise to the full tuple of matrices, including the router.
For a training input $L$-token sequence \(X \in \mathbb{R}^{L \times d}\), a single-head expert can be written as
\begin{equation}
\begin{aligned}
 A_\omega(X)=\operatorname{softmax}\left(\frac{(X Q)(X K)^{\top}}{\sqrt{d_k}}\right) X V O,
 \quad\psi_\omega(X)=\operatorname{Rout}\left(\sigma\left(A_\omega(X) W_1\right) W_2\right) .
\end{aligned}
\end{equation}
A smooth router score \(s_\phi(X)\) is included in the feature map. For softmax-normalized MoE routing, the empirical output
\(
M_N(X)_{t,:}=\frac{\sum_i e^{s_{\phi_i}(X)_t} \psi_{\omega_i}(X)_{t,:}}{\sum_j e^{s_{\phi_j}(X)_t}}
\)
is represented by an augmented Hilbert-space valued feature map \(F_{\text {soft }}(\omega, \phi)=\left(e^{s_\phi} \psi_\omega, e^{s_\phi}\right)\) and a smooth normalization map applied after averaging. Another gating choice can be considered as well. For an unnormalized non-negative smooth input-dependent gate \(g_\phi(X)\), one may take \(F_{\text {un }}(\omega, \phi)= \left(g_\phi\left(X_r\right) \odot \psi_\omega\left(X_r\right)\right)_{r=1}^n\) in logit Hilbert space. Thus the Transformer MoE model with these two choices of gating fits the form \(J(\rho)=R\left(\int F \mathrm{~d} \rho\right)\).

The theorem below analyzes the dynamics of parameter optimization for Transformer MoE and uses the following concrete conditions, which are the transformer versions of the localized product-space assumptions in the Appendix:
\begin{itemize}
    \item[(T1)] training inputs are bounded, $\left\|X_r\right\|_F \leq B_X$;
    \item[(T2)] the activation and router score maps are $C^2$ in the parameters on bounded sets;
    \item[(T3)] in the normalized-gate case, denominators remain bounded below by a positive constant on the moment set reached by the trajectory;
    \item[(T4)] the parameter and momentum trajectory remains in a bounded region for the time horizon or asymptotic regime under consideration;
    \item[(T5)] for exponential convergence, the product-space PL and upper-gradient assumptions hold along the trajectory.
\end{itemize}

Conditions $(T1)$-$(T4)$ are analytic well-posedness conditions, while condition $(T5)$ is an optimization-landscape condition. Separating them prevents the transformer statement from overstating what follows from smoothness alone.
\begin{theorem}[Transformer MoE consequence]
    Assume bounded training inputs, \(C^2\) expert and router maps on bounded parameter sets, a denominator lower bound for normalized gates, and bounded parameter/momentum trajectories on the time interval considered. Then, under the inertial scaling $\eta_h=h, \beta_h=1-\gamma h+o(h)$, the regularized Muon scheme on expert-router particles,
\[
\begin{aligned}
P_{i, k+1}  =\beta P_{i, k}+(1-\beta) a_i^N\left(\theta_k\right), \quad
\theta_{i, k+1}  =\theta_{i, k}-\eta \operatorname{Orth}_{\varepsilon}^{\Theta}\left(P_{i, k+1}\right) .
\end{aligned}
\]
where \(a_i^N=D F\left(\theta_i\right)^* \nabla R\left(N^{-1} \sum_j F\left(\theta_j\right)\right)\), satisfies the finite-particle ODE limit, phase-space Hamiltonian PDE, dissipation identity and, under the \(P L\) and upper-gradient assumptions, the exponential convergence estimate. The general Hilbert valued product-space PDE and the proof of this theorem are discussed in detail in the Appendix.
\end{theorem}
\section{Numerical Experiments}
We use two deterministic synthetic experiment classes to test the finite-particle dynamics developed above.  The goal is not to benchmark large-scale training, but to isolate the phenomena predicted by the Hamiltonian formulation. The first experiment focuses on a $M$-particle matrix mean matching problem with the finite objective functional on $N$ particles given by $J_N(W_1,\ldots,W_N) =\frac12\norm{\frac1N\sum_{i=1}^N W_i-\bar W_\star}_F^2$ with $W_\star=\frac1M\sum_{j=1}^M W_{j,\star}$. The second experiment is a nonlinear teacher-student problem on a product matrix space with each particle $\theta_i=(A_i,B_i)\in \Theta=\R^{p\times r}\times \R^{r\times d}$ with the $(M,N)$-particle teacher-student tanh neural network objective $J_N((A_i,B_i)_{i=1}^N)=\frac1{2Sp}\sum_{s=1}^S
\norm{\frac1N\sum_{i=1}^N A_i\tanh\!\left(\frac{B_i x_s}{\sqrt d}\right)-y_s}_2^2$ trained on $S$ input-output $(x_s,y_s)$ pairs. In both cases $J^{*} = 0$ is attainable. The experimental results are reported in part in Figure \ref{fig:exp-main paper}, with the complete numerical experiments being discussed in the Appendix. The synthetic experiments support the Hamiltonian interpretation developed in the paper. In the matrix mean-matching problem, regularized Muon avoids the finite-step residual floor exhibited by hard polar and Newton-Schulz updates, allowing the objective and Hamiltonian energies to decay smoothly towards zero. In the nonlinear product-space neural network setting, the regularized dynamics preserve the advantages of spectral Muon geometry while exhibiting the stability predicted by the smooth mirror formulation.
\begin{figure}[!htbp]
\centering
\begin{minipage}{0.44\linewidth}
\centering
\includegraphics[width=\linewidth]{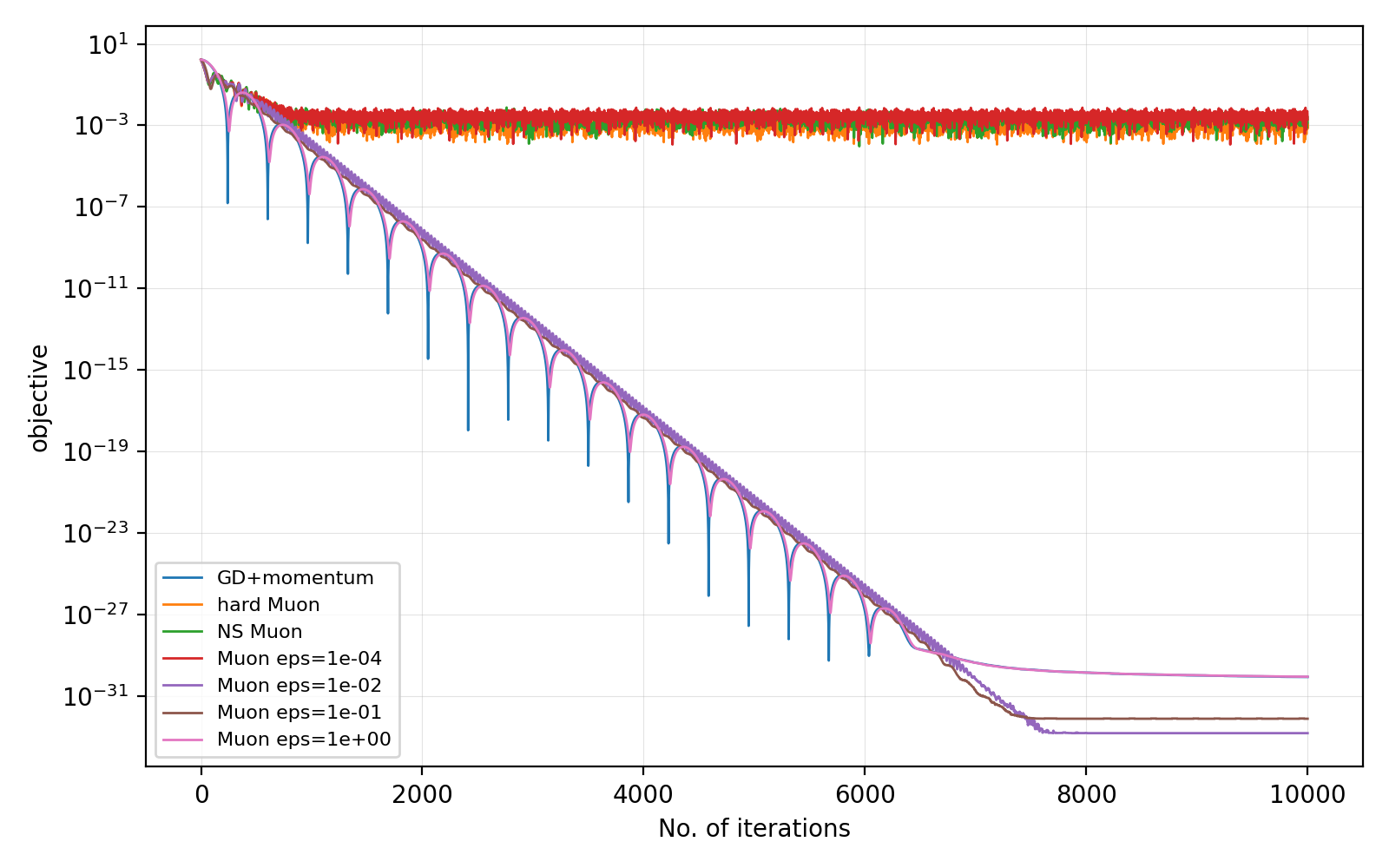}
\end{minipage}\hfill
\begin{minipage}{0.44\linewidth}
\centering
\includegraphics[width=\linewidth]{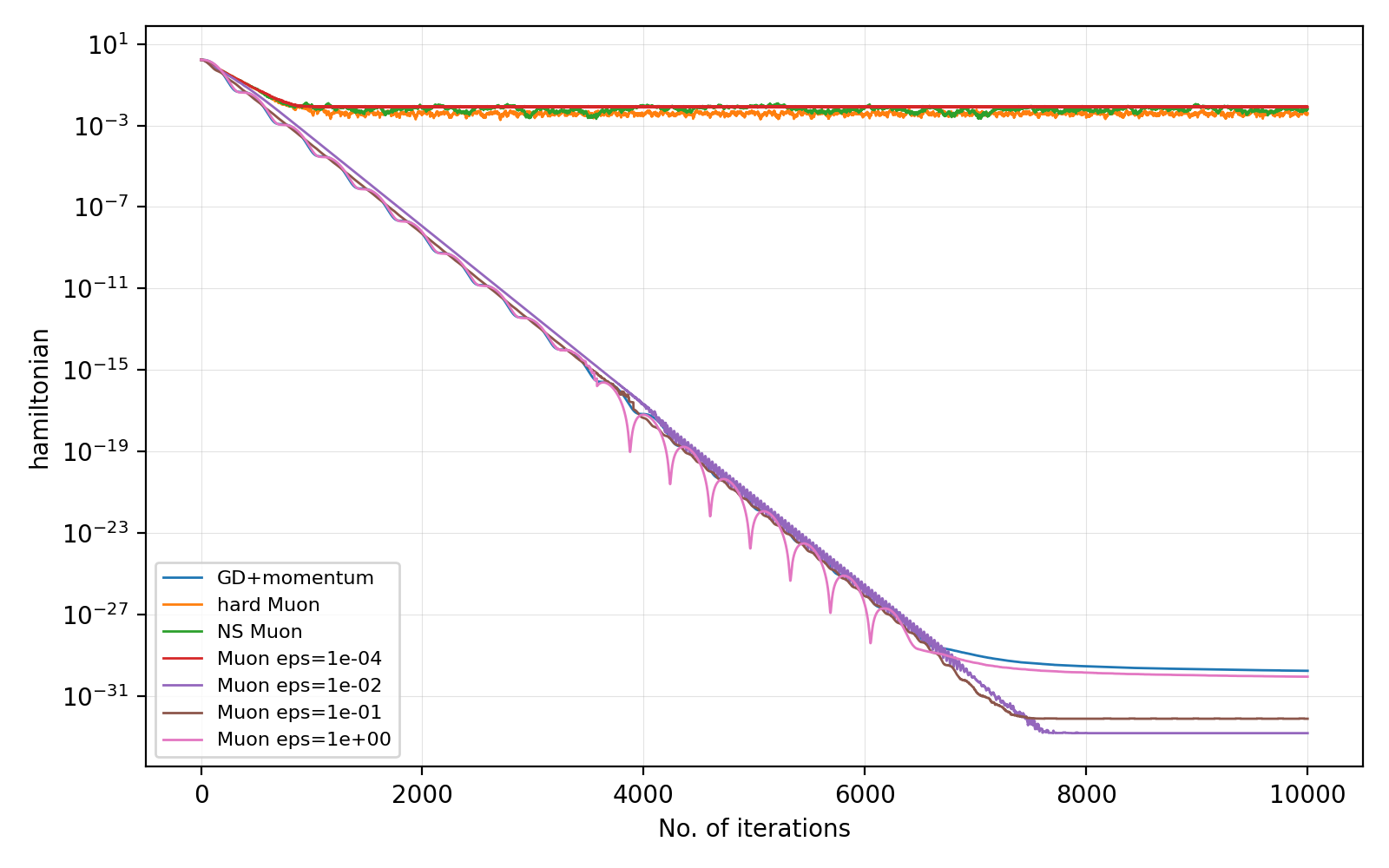}
\end{minipage}
\vspace{0.3em}
\begin{minipage}{0.44\linewidth}
\includegraphics[width=\linewidth]{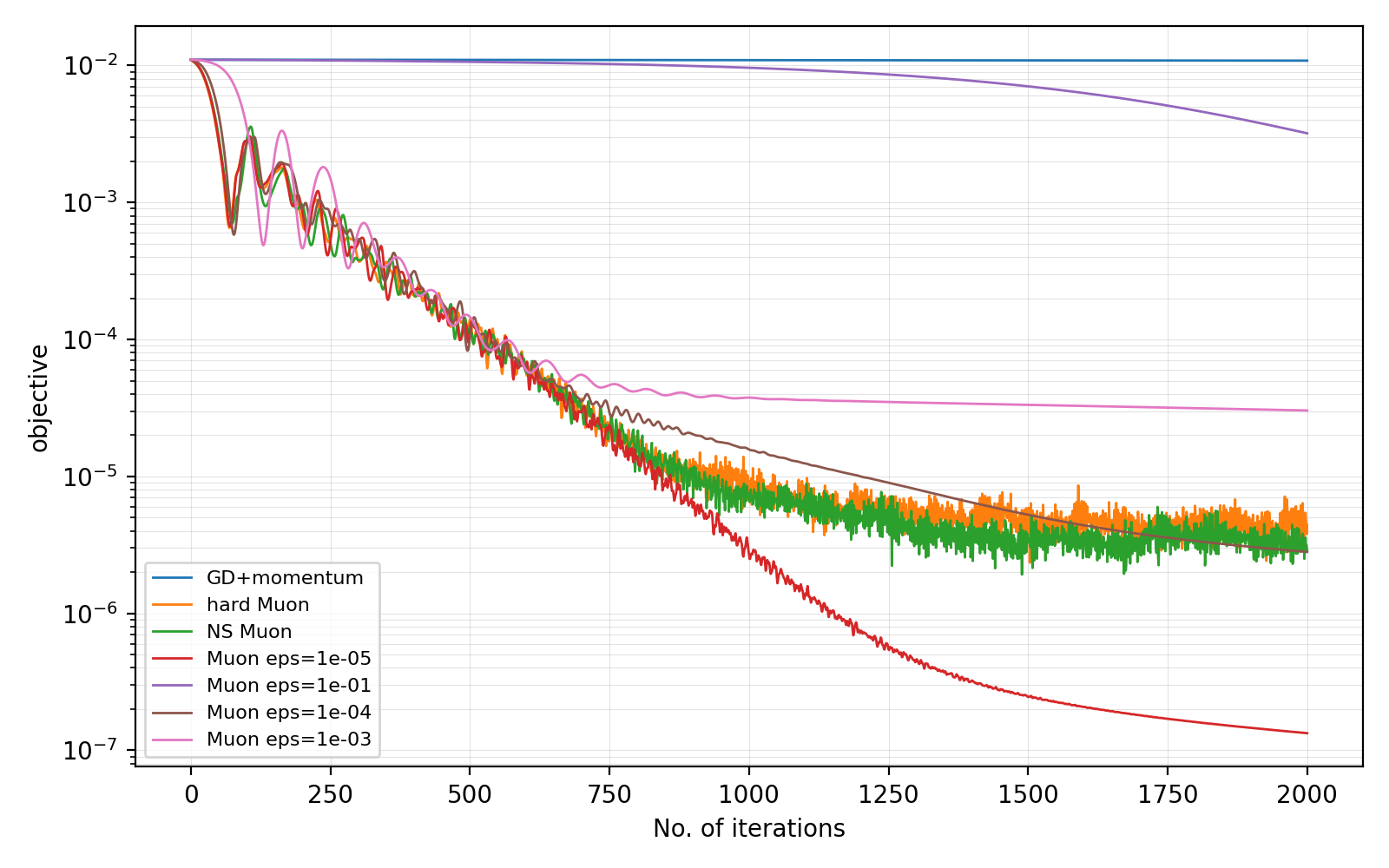}
\end{minipage}\hfill
\begin{minipage}{0.44\linewidth}
\centering
\includegraphics[width=\linewidth]{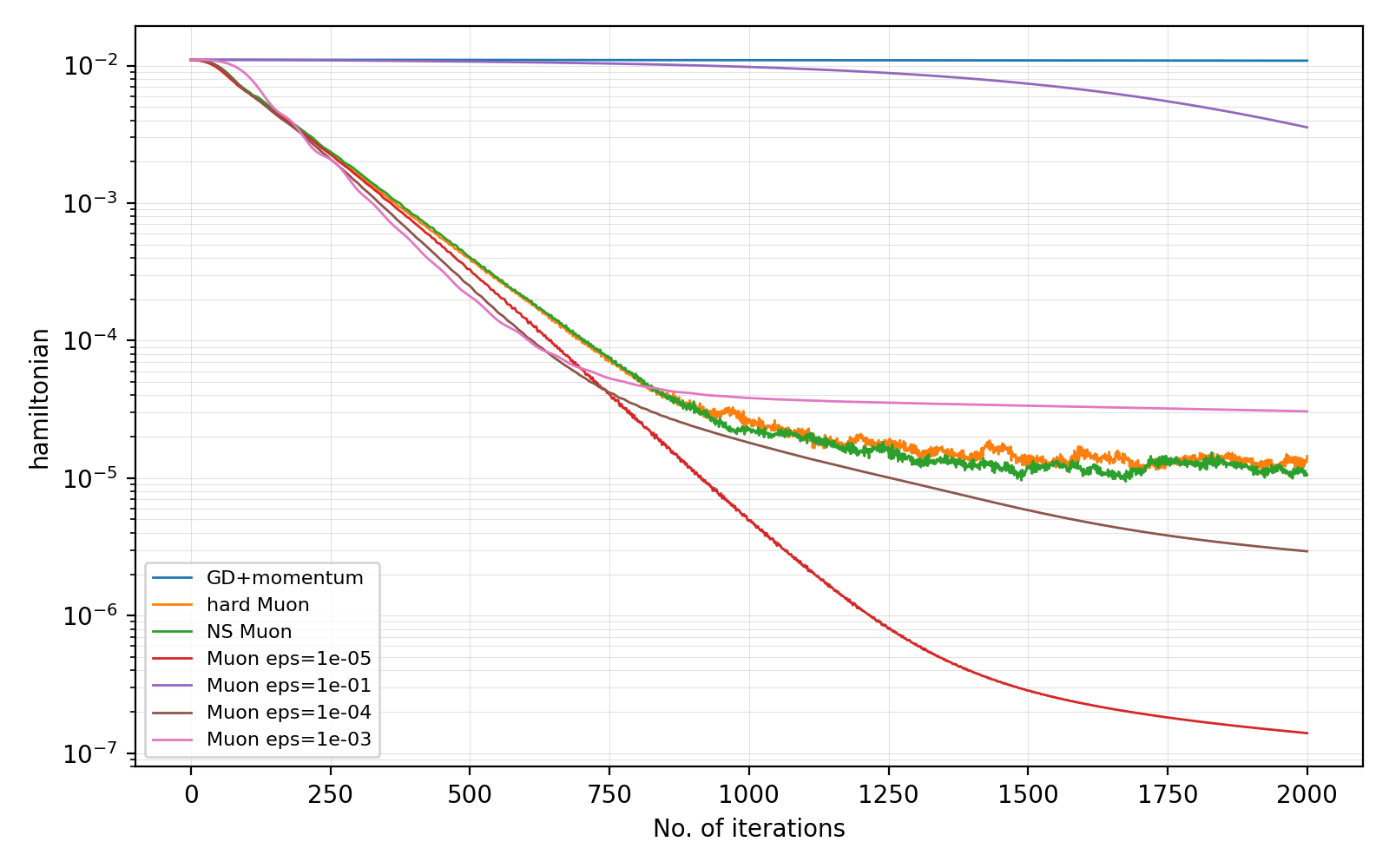}
\end{minipage}
\caption{Top: Experiment~1 on matrix mean matching with $(M,N)=(4,32)$. Bottom: Experiment~2 on product-space teacher-student particles with $(d,r,p)=(10,6,4)$ and $S=320$ training points and $(M,N)=(3,12)$. For both experiments, left panels show $J_N$ and right panels show the Hamiltonian $K+\gamma J_N$ on a loglinear scale.}
\label{fig:exp-main paper}
\end{figure}
\section{Conclusion}
This work identifies Muon as a spectral trust-region mirror method and derives the damped Hamiltonian probability flow induced by its regularized orthogonalization map. It shows that Muon's matrix-level update admits a coherent variational, mean-field, and Lyapunov structure.

\section*{References}
\bibliographystyle{plainnat}
\bibliography{ref.bib}
\nocite{*}

\newpage
\appendix

\section{Technical appendices and supplementary material}
\subsection{Proof of Proposition \ref{Proposition 1}}
Let $P=U \Sigma V^{\top}$ be a reduced SVD with $s=\operatorname{rank}(P)$ and singular values $\sigma_1(P), \ldots, \sigma_s(P)$. By the duality between the operator norm and the nuclear norm,

$$
\|P\|_{\mathrm{nuc}}=\sup _{\|M\|_{\mathrm{op}} \leq 1}\langle P, M\rangle_F .
$$

For completeness, this follows from von Neumann's trace inequality:

$$
\langle P, M\rangle_F \leq \sum_{i=1}^q \sigma_i(P) \sigma_i(M) \leq \sum_{i=1}^q \sigma_i(P)=\|P\|_{\mathrm{nuc}}
$$

whenever $\|M\|_{\mathrm{op}} \leq 1$. Taking $M=U V^{\top}$ gives equality, since the nonzero singular values of $U V^{\top}$ are all one and the singular directions align with those of $P$. Therefore $\operatorname{Orth}(P)$ is a maximizer of the support-function problem, and by sign reversal $-\operatorname{Orth}(P)$ is a minimizer of $\langle P, G\rangle_F$ over $\|G\|_{\mathrm{op}} \leq 1$. Since $\Phi_0$ is the indicator of this spectral unit ball, the constrained problem is exactly Equation \eqref{Constrained optimization form of Orthogonalization operation}. The canonical choice of minimizer gives the ideal Muon step.
When $P$ is rank deficient, the optimizer need not be unique. The subdifferential formula
$$
\partial\|P\|_{\text {nuc }}=\left\{U V^{\top}+Z: U^{\top} Z=0, Z V=0,\|Z\|_{\mathrm{op}} \leq 1\right\}
$$
shows the source of nonuniqueness. Muon chooses the canonical element with $Z=0$.
\subsection{Scalar and spectral Fenchel conjugacy}
For $\varepsilon>0$, define

$$
\psi_{\varepsilon}(a)=\sqrt{a^2+\varepsilon^2}-\varepsilon, \quad \phi_{\varepsilon}(b)=\varepsilon\left(1-\sqrt{1-b^2}\right)+\iota_{[-1,1]}(b) .
$$

Both are proper, closed, convex, and even. On their differentiability domains,

$$
\psi_{\varepsilon}^{\prime}(a)=\frac{a}{\sqrt{a^2+\varepsilon^2}}, \quad \quad \phi_{\varepsilon}^{\prime}(b)=\frac{\varepsilon b}{\sqrt{1-b^2}}
$$

These maps are inverse: if $b=a / \sqrt{a^2+\varepsilon^2}$, then $a=\varepsilon b / \sqrt{1-b^2}$, and conversely. Hence $\phi_{\varepsilon}^*=\psi_{\varepsilon}$ and $\psi_{\varepsilon}^*=\phi_{\varepsilon}$.
The spectral lifts Equation \eqref{Fenchel conjugates} are conjugate by the standard conjugacy theorem for unitarily invariant spectral functions. Equivalently, von Neumann's trace inequality reduces

$$
\sup _G\left\{\langle P, G\rangle_F-\Phi_{\varepsilon}(G)\right\}
$$

to the scalar singular-value optimization, whose value is $\sum_i \psi_{\varepsilon}\left(\sigma_i(P)\right)=\Psi_{\varepsilon}(P)$. The reverse conjugacy is identical. The gradient formula Equation \eqref{regularized Orth} follows from spectral-function calculus. The derivative of $a \mapsto a / \sqrt{a^2+\varepsilon^2}$ is

$$
\frac{\varepsilon^2}{\left(a^2+\varepsilon^2\right)^{3 / 2}} \leq \frac{1}{\varepsilon},
$$
so $\operatorname{Orth}_{\varepsilon}=\nabla \Psi_{\varepsilon}$ is $1 / \varepsilon$-Lipschitz in Frobenius norm. Finally, each nonzero singular value is transformed as $\sigma_i / \sqrt{\sigma_i^2+\varepsilon^2} \rightarrow 1$, proving $\operatorname{Orth}_{\varepsilon}(P) \rightarrow \operatorname{Orth}(P)$.
\subsection{Proof of Proposition \ref{Proposition 2}}
The first-order condition for Equation \eqref{regularized G} is
$$
0=P+\nabla \Phi_{\varepsilon}\left(G_{\varepsilon}(P)\right) .
$$
Since $\nabla \Phi_{\varepsilon}$ and $\nabla \Psi_{\varepsilon}$ are inverse on the interior of the effective domain of $\Phi_{\varepsilon}$,
$$
G_{\varepsilon}(P)=\nabla \Psi_{\varepsilon}(-P) .
$$
The function $\Psi_{\varepsilon}$ is even, so its gradient is odd, and $G_{\varepsilon}(P)=-\nabla \Psi_{\varepsilon}(P)=-\operatorname{Orth}_{\varepsilon}(P)$.
For the Bregman form, use $\nabla \Phi_{\varepsilon}\left(G_{\varepsilon, k}\right)=-P_k$. Then
$$
\begin{aligned}
D_{\Phi_{\varepsilon}}\left(G, G_{\varepsilon, k}\right) & =\Phi_{\varepsilon}(G)-\Phi_{\varepsilon}\left(G_{\varepsilon, k}\right)-\left\langle\nabla \Phi_{\varepsilon}\left(G_{\varepsilon, k}\right), G-G_{\varepsilon, k}\right\rangle_F \\
& =\Phi_{\varepsilon}(G)+\left\langle P_k, G\right\rangle_F+C_k
\end{aligned}
$$
where $C_k$ is independent of $G$. Adding $\left\langle P_{k+1}-P_k, G\right\rangle_F$ gives the same minimizers as $\left\langle P_{k+1}, G\right\rangle_F+\Phi_{\varepsilon}(G)$.
\subsection{Proof of Proposition \ref{Proposition 3}}
Let $\nu$ be a finite signed measure with $\nu(\mathcal{X})=0$ and set $\rho_s=\rho+s \nu$. Then

$$
m_{\rho_s}=m_\rho+s \int_{\mathcal{X}} F(W) \mathrm{d} \nu(W)
$$

Differentiating $J\left(\rho_s\right)=R\left(m_{\rho_s}\right)$ at $s=0$ gives

$$
\left.\frac{\mathrm{d}}{\mathrm{~d} s} J\left(\rho_s\right)\right|_{s=0}=R^{\prime}\left(m_\rho\right) \int_{\mathcal{X}} F(W) \mathrm{d} \nu(W)
$$

This is the defining identity for the first variation, modulo constants independent of $W$. Taking the spatial gradient gives the Wasserstein gradient in Equation \eqref{Wasserstein gradient of J_N}.

For the particle gradient, let $U=\left(U_1, \ldots, U_N\right)$. By the chain rule,

$$
\begin{aligned}
D J_N(W)[U] & =R^{\prime}\left(F^N(W)\right) \frac{1}{N} \sum_{i=1}^N\left\langle\nabla F\left(W_i\right), U_i\right\rangle_F \\
& =\frac{1}{N} \sum_{i=1}^N\left\langle a_i(W), U_i\right\rangle_F=\langle a(W), U\rangle_{\text {avg }} .
\end{aligned}
$$
Hence $\operatorname{grad}_{\text {avg }} J_N(W)=a(W)$.
\subsection{Product-space Fenchel duality for particles}
Define
$$
\Phi_{\varepsilon}^N(G)=\frac{1}{N} \sum_{i=1}^N \Phi_{\varepsilon}\left(G_i\right), \quad \Psi_{\varepsilon}^N(P)=\frac{1}{N} \sum_{i=1}^N \Psi_{\varepsilon}\left(P_i\right) .
$$

With respect to $\langle\cdot, \cdot\rangle_{\text {avg }}$,
$$
\begin{aligned}
\left(\Phi_{\varepsilon}^N\right)^*(P) & =\sup _G\left\{\frac{1}{N} \sum_i\left\langle P_i, G_i\right\rangle_F-\frac{1}{N} \sum_i \Phi_{\varepsilon}\left(G_i\right)\right\} \\
& =\frac{1}{N} \sum_i \sup _{G_i}\left\{\left\langle P_i, G_i\right\rangle_F-\Phi_{\varepsilon}\left(G_i\right)\right\}=\Psi_{\varepsilon}^N(P) .
\end{aligned}
$$

The reverse conjugacy is identical. Thus the finite-particle regularized Muon step is the blockwise minimization stated in Equation \eqref{Finite particle regularized Muon scheme}.
\subsection{Proof of Theorem \ref{McKean-Vlasov formulation} and \ref{Damped Hamiltonian structure}}
\begin{theorem}
    [Finite- $N$ ODE limit]. Assume the vector field $W \mapsto a(W)$ is Lipschitz on the region visited by the discrete and continuous trajectories, and $\operatorname{Orth}_{\varepsilon}$ is $1 / \varepsilon$-Lipschitz. Under $\eta_h=h$ and $\beta_h=1-\gamma h+O\left(h^2\right)$, the piecewise-linear interpolation of the particle scheme converges uniformly on compact intervals to

$$
\dot{W}_i=-\operatorname{Orth}_{\varepsilon}\left(P_i\right), \quad \dot{P}_i=\gamma\left(a_i(W)-P_i\right) .
$$

If the Lipschitz bound is global, the solution is global and the convergence rate is $O(h)$ when $\beta_h=1-\gamma h+O\left(h^2\right)$.
\end{theorem}
\begin{proof}

Let
$$
B_N(W, P)=\left(-\operatorname{Orth}_{\varepsilon}\left(P_i\right), \gamma\left(a_i(W)-P_i\right)\right)_{i=1}^N,
$$
where $a_i(W)=R^{\prime}\left(F^N(W)\right) \nabla F\left(W_i\right)$. The map $P \mapsto \operatorname{Orth}_{\varepsilon}(P)$ is globally $1 / \varepsilon$-Lipschitz. Under Assumption \ref{ass:Basic smoothness}, $a(W)$ is globally Lipschitz in the mean-field norm. Indeed, for two particle configurations $W, \widetilde{W}$,
$$
\begin{aligned}
\left\|a_i(W)-a_i(\widetilde{W})\right\|_F \leq & \left|R^{\prime}\left(F^N(W)\right)\right|\left\|\nabla F\left(W_i\right)-\nabla F\left(\widetilde{W}_i\right)\right\|_F \\
& +\left|R^{\prime}\left(F^N(W)\right)-R^{\prime}\left(F^N(\widetilde{W})\right)\right|\left\|\nabla F\left(\widetilde{W}_i\right)\right\|_F .
\end{aligned}
$$

The first term is controlled by the global bound $\left|R^{\prime}\right| \leq M_R$ and the Lipschitzness of $\nabla F$. The second term is controlled by the Lipschitzness of $R^{\prime}$, the boundedness of $\nabla F$, and the Lipschitzness of $F$, which follows from $\|\nabla F\| \leq M_F$. Thus the ODE has a unique global solution. The global existence follows from bounded $W$-velocity, $\left\|\operatorname{Orth}_{\varepsilon}\left(P_i\right)\right\|_F \leq \sqrt{q}$, and at most linear growth in $P_i$.

Then,

$$
P_{i, k+1}-P_{i, k}=\left(1-\beta_h\right)\left(a_i\left(W_k\right)-P_{i, k}\right)=\left(\gamma h-r_h\right)\left(a_i\left(W_k\right)-P_{i, k}\right),
$$

which is $h \gamma\left(a_i\left(W_k\right)-P_{i, k}\right)+o(h)$ locally uniformly on bounded sets. Also,

$$
\begin{aligned}
W_{i, k+1}-W_{i, k} & =-h \operatorname{Orth}_{\varepsilon}\left(P_{i, k+1}\right) \\
& =-h \operatorname{Orth}_{\varepsilon}\left(P_{i, k}\right)-h\left(\operatorname{Orth}_{\varepsilon}\left(P_{i, k+1}\right)-\operatorname{Orth}_{\varepsilon}\left(P_{i, k}\right)\right) .
\end{aligned}
$$

The last term is $O\left(h^2\right)+o(h) h$ on bounded sets by the Lipschitzness of Orth ${ }_{\varepsilon}$. The one-step consistency error therefore tends to zero after division by $h$, and it is $O(h)$ when $r_h=O\left(h^2\right)$. The standard Euler convergence estimate with discrete Gronwall gives uniform convergence on $[0, T]$ and rate $O(h)$ in the second-order-consistent case.
\end{proof}
\begin{remark}
    The statement is intentionally phrased in a local form. For neural-network parameterizations, global Lipschitz constants are often unavailable, while finite-horizon bounded-trajectory constants are natural.
\end{remark} 
Let $\mu_t \in \mathcal{P}_1(\mathrm{Z})$ with $\mathrm{Z}=\mathrm{X} \times \mathrm{X}$ and $\rho_t=\left(\pi_W\right)_{\#} \mu_t$. Define

$$
b_{\mu_t}(W, P)=\left(-\operatorname{Orth}_{\varepsilon}(P), \gamma\left(a_{\rho_t}(W)-P\right)\right) .
$$

\begin{theorem}[Well-posed McKean-Vlasov equation] Under global Lipschitz assumptions on $b_\mu$ in ( $W, P$ ) and in $\mu$ with respect to $W_1$, for every $\mu_0 \in \mathcal{P}_1(\mathrm{Z})$ there is a unique solution $\mu \in C\left([0, \infty) ; \mathcal{P}_1(\mathrm{Z})\right)$ of

$$
\partial_t \mu_t+\nabla_W \cdot\left(-\operatorname{Orth}_{\varepsilon}(P) \mu_t\right)+\nabla_P \cdot\left(\gamma\left(a_{\rho_t}(W)-P\right) \mu_t\right)=0 .
$$

Equivalently, for every $\zeta \in C_c^{\infty}(\mathrm{Z})$,

$$
\frac{\mathrm{d}}{\mathrm{~d} t} \int \zeta \mathrm{~d} \mu_t=\int\left[\left\langle\nabla_W \zeta,-\operatorname{Orth}_{\varepsilon}(P)\right\rangle+\left\langle\nabla_P \zeta, \gamma\left(a_{\rho_t}(W)-P\right)\right\rangle\right] \mathrm{d} \mu_t
$$
\end{theorem}

\begin{proof}
The nonlinear characteristic system is

$$
W_t=W_0-\int_0^t \operatorname{Orth}_{\varepsilon}\left(P_s\right) \mathrm{d} s, \quad P_t=P_0+\int_0^t \gamma\left(a_{\rho_s}\left(W_s\right)-P_s\right) \mathrm{d} s
$$

where $m_s=\mathbb{E} F\left(W_s\right)$. The drift is Lipschitz in ( $W, P$ ) and in the law with respect to $W_1$ under the preceding estimates, with at most linear growth in $P$. Picard iteration for McKean-Vlasov ODEs gives existence and uniqueness. Setting $\mu_t=\operatorname{Law}\left(W_t, P_t\right)$ and applying the chain rule to any $\zeta \in C_c^{\infty}(\mathcal{Z})$ gives
 \begin{equation*}
\begin{aligned}
\frac{d}{d t} \zeta\left(W_t, P_t\right)
=&\left\langle\nabla_W \zeta\left(W_t, P_t\right), \dot{W}_t\right\rangle
+\left\langle\nabla_P \zeta\left(W_t, P_t\right), \dot{P}_t\right\rangle \\
=& \left\langle\nabla_W \zeta,-\operatorname{Orth}_{\varepsilon}\left(P_t\right)\right\rangle
+\left\langle\nabla_P \zeta, \gamma\left(a_i(W)-P_i\right)\right\rangle .
\end{aligned}
\end{equation*}
   
Conversely, the superposition principle for Lipschitz continuity equations represents any weak solution by characteristics, and uniqueness of the nonlinear characteristic equation gives uniqueness of the PDE solution.
\end{proof}
\begin{remark} For the scalar functional, Lipschitz dependence on the measure follows from Lipschitzness of $R^{\prime}$ and $F$ together with boundedness of $\nabla F$ and the boundedness/local boundedness of $R^{\prime}$ needed for the spatial Lipschitz term.\end{remark}

\begin{theorem}
    [Hamiltonian form and dissipation] Let
$$
\Ham_{\varepsilon, \gamma}(\mu)=\int \Psi_{\varepsilon}(P) \mathrm{d} \mu(W, P)+\gamma R\left(\int F(W) \mathrm{d} \mu(W, P)\right)
$$
A first variation is
$$
\frac{\delta \Ham_{\varepsilon, \gamma}}{\delta \mu}(\mu)(W, P)=\Psi_{\varepsilon}(P)+\gamma R^{\prime}\left(m_\mu\right) F(W) .
$$

Therefore Equation \eqref{Mc-Kean Vlasov equation} is the damped Hamiltonian equation. Along weak solutions for which the cutoff argument is justified,

$$
\frac{\mathrm{d}}{\mathrm{~d} t} \Ham_{\varepsilon, \gamma}\left(\mu_t\right)=-\gamma \int\left\langle P, \operatorname{Orth}_{\varepsilon}(P)\right\rangle_{\mathrm{F}} \mathrm{~d} \mu_t \leq 0 .
$$
\end{theorem}
\begin{proof}
     Let $\nu$ be a finite signed measure on $\Z$ with
    $\nu(\Z)=0$ and finite first moment. For small $s$, formally set
    $$
    \mu_s=\mu+s \nu .
    $$
    Then
    $$
    m_{\mu_s}=\int_\Z F(W) d(\mu+s \nu)(W, P)=m_\mu+s \int_\Z F(W) d \nu(W, P)
    $$
    Note that,
    $$
    \Ham_{\varepsilon, \gamma}\left(\mu_s\right)=\int_\Z \Psi_{\varepsilon}(P) d(\mu+s \nu)+\gamma R\left(m_{\mu_s}\right) = \int_\Z \Psi_{\varepsilon}(P) d \mu+s \int_\Z \Psi_{\varepsilon}(P) d \nu+\gamma R\left(m_\mu+s \int_\Z F(W) d \nu\right).
    $$
    Differentiating at $s=0$, we have that
    $$
    \left.\frac{d}{d s} \Ham_{\varepsilon, \gamma}\left(\mu_s\right)\right|_{s=0}=\int_Z\left[\Psi_{\varepsilon}(P)+\gamma R^{\prime}\left(m_\mu\right) F(W)\right] d \nu(W, P).
    $$
    Hence a valid first variation is
    $$
    \frac{\delta \Ham_{\varepsilon, \gamma}}{\delta \mu}(\mu)(W, P)=\Psi_{\varepsilon}(P)+\gamma R^{\prime}\left(m_\mu\right) F(W),
    $$
    up to an additive constant depending on $\mu$ but independent of $(W, P)$. Such constants do not affect $\nabla_W$ or $\nabla_P$. At a fixed $\mu = \mu_t$, the scalar $m_t \coloneqq m_{\mu_t}$ is constant with respect to the variables $(W, P)$. Therefore
$$
\nabla_P \frac{\delta \Ham_{\varepsilon, \gamma}}{\delta \mu}(\mu_t)(W, P)=\nabla \Psi_{\varepsilon}(P)
$$
and
$$
\nabla_W \frac{\delta \Ham_{\varepsilon, \gamma}}{\delta \mu}(\mu_t)(W, P)=\gamma R^{\prime}\left(m_t\right) \nabla F(W).
$$

   Let
$ G_{\varepsilon}(P)\coloneqq \nabla \Psi_{\varepsilon}(P),
$ and define $a_t(W)\coloneqq R^{\prime}\left(m_t\right) \nabla F(W)$. Then the PDE in Equation \eqref{Mc-Kean Vlasov equation} is the continuity equation with velocity field
$$
b_t(W, P)=\left(b_W(t, W, P), b_P(t, W, P)\right),
$$
where
$$
b_W(t, W, P)=-G_{\varepsilon}(P) \quad\textrm{ and }\quad b_P(t, W, P)=\gamma\left(a_t(W)-P\right).
$$
Thus
$$
\partial_t \mu_t+\nabla_W \cdot\left(\mu_t b_W\right)+\nabla_P \cdot\left(\mu_t b_P\right)=0.
$$
In weak form, for every smooth compactly supported test function $\zeta \in C_c^{\infty}(\Z)$,
\begin{equation}\label{weak form of Hamiltonian flow}
\frac{d}{d t} \int_\Z \zeta(W, P) d \mu_t(W, P)=\int_\Z\left[\left\langle\nabla_W \zeta, b_W\right\rangle_F+\left\langle\nabla_P \zeta, b_P\right\rangle_F\right] d \mu_t.
\end{equation}
The functions $F(W)$ and $\Psi_{\varepsilon}(P)$ are not compactly supported, but they have at most linear growth and bounded gradients:
$$
|F(W)| \leq|F(0)|+M_F\|W\|_F,
$$
and
$$
0 \leq \Psi_{\varepsilon}(P) \leq\|P\|_{\mathrm{nuc}} \leq \sqrt{q}\|P\|_F,
$$
while
$$
\|\nabla F(W)\|_F \leq M_F, \quad\left\|\nabla \Psi_{\varepsilon}(P)\right\|_F \leq \sqrt{q}.
$$
Hence $F$ and $\Psi_{\varepsilon}$ can be used as test functions by a standard cutoff argument. For completeness, here is the cutoff justification. Let $\chi \in C_c^{\infty}([0, \infty))$ satisfy

$$
0 \leq \chi \leq 1, \quad \chi(r)=1 \text { for } r \leq 1, \quad \chi(r)=0 \text { for } r \geq 2 .
$$
Define
$\chi_R(W, P)\coloneqq \chi\left(\frac{\sqrt{\|W\|_F^2+\|P\|_F^2}}{R}\right)$. For a $C^1$ function $\zeta$ with at most linear growth and bounded gradient, set $
\zeta_R\coloneqq \chi_R \zeta$. Then $\zeta_R \in C_c^1(\Z), \zeta_R \rightarrow \zeta$ pointwise, and

$$
\nabla \zeta_R=\chi_R \nabla \zeta+\zeta \nabla \chi_R
$$

The first term converges to $\nabla \zeta$. The second term is supported on the annulus

$$
R \leq \sqrt{\|W\|_F^2+\|P\|_F^2} \leq 2 R,
$$

and satisfies

$$
\left|\zeta \nabla \chi_R\right| \leq \frac{C}{R}\left(1+\|W\|_F+\|P\|_F\right) \mathbf{1}_{\{R \leq\|(W, P)\| \leq 2 R\}} \leq C \mathbf{1}_{\{\|(W, P)\| \geq R\}} .
$$
Since the vector field has at most linear growth in $P$, and $\mu_t$ has finite first moment, the contribution of this annulus vanishes as $R \rightarrow \infty$. Therefore Equation \ref{weak form of Hamiltonian flow} remains valid for $\zeta=F(W)$ and $\zeta=\Psi_{\varepsilon}(P)$.

By definition, $m_t=\int_\Z F(W) d \mu_t(W, P)$. Take $\zeta(W, P)=F(W)$. Then
$$
\nabla_W \zeta(W, P)=\nabla F(W), \quad \nabla_P \zeta(W, P)=0 .
$$

Using Equation \ref{weak form of Hamiltonian flow},

$$
\frac{d}{d t} m_t=\int_\Z\left\langle\nabla F(W), b_W(t, W, P)\right\rangle_F d \mu_t(W, P) .
$$

Since $b_W=-G_{\varepsilon}(P)$,

\begin{equation}\label{derivative of mt}
\frac{d}{d t} m_t=-\int_\Z\langle\nabla F(W), G_{\varepsilon}(P)\rangle_F d \mu_t(W, P) =-\int_\Z\left\langle\nabla F(W), \nabla \Psi_{\varepsilon}(P)\right\rangle_F d \mu_t(W, P).
\end{equation}
Since both $\nabla F$ and $\nabla \Psi_{\varepsilon}$ are bounded, $m_t$ is absolutely continuous. Further, $R \in C^1$ and $m_t$ is absolutely continuous, the ordinary chain rule and Equation \ref{derivative of mt} gives
$$
\frac{d}{d t}\left[\gamma R\left(m_t\right)\right]=\gamma R^{\prime}\left(m_t\right) m_t^{\prime} =-\gamma R^{\prime}\left(m_t\right) \int_\Z\langle\nabla F(W), G_{\varepsilon}(P)\rangle_F d \mu_t.
$$

Since $R^{\prime}\left(m_t\right)$ is a scalar independent of $(W, P)$, this is

$$
\frac{d}{d t}\left[\gamma R\left(m_t\right)\right]=-\gamma \int_\Z\left\langle R^{\prime}\left(m_t\right) \nabla F(W), G_{\varepsilon}(P)\right\rangle_F d \mu_t.
$$
Using $a_t(W)=R^{\prime}\left(m_t\right) \nabla F(W)$,
\begin{equation}\label{time derivative of potential energy}
\frac{d}{d t}\left[\gamma R\left(m_t\right)\right]=-\gamma \int_\Z\left\langle a_t(W), G_{\varepsilon}(P)\right\rangle_F d \mu_t.
\end{equation}
Define
$$
K_{\varepsilon}(t)\coloneqq \int_\Z \Psi_{\varepsilon}(P) d \mu_t(W, P)
$$

Take $\zeta(W, P)=\Psi_{\varepsilon}(P)$. Then

$$
\nabla_W \zeta(W, P)=0, \quad \nabla_P \zeta(W, P)=G_{\varepsilon}(P) .
$$

Using Equation \ref{weak form of Hamiltonian flow},

$$
\frac{d}{d t} K_{\varepsilon}(t)=\int_\Z\left\langle G_{\varepsilon}(P), b_P(t, W, P)\right\rangle_F d \mu_t
$$

Since

$$
b_P(t, W, P)=\gamma\left(a_t(W)-P\right),
$$

we obtain

$$
\frac{d}{d t} K_{\varepsilon}(t)=\gamma \int_\Z\left\langle G_{\varepsilon}(P), a_t(W)-P\right\rangle_F d \mu_t .
$$

Expanding the inner product,

$$
\frac{d}{d t} K_{\varepsilon}(t)=\gamma \int_\Z\left\langle G_{\varepsilon}(P), a_t(W)\right\rangle_F d \mu_t-\gamma \int_\Z\langle G_{\varepsilon}(P), P\rangle_F d \mu_t.
$$
Therefore
\begin{equation}\label{time derivative of kinetic energy}
\frac{d}{d t} K_{\varepsilon}(t)=\gamma \int_\Z\left\langle a_t(W), G_{\varepsilon}(P)\right\rangle_F d \mu_t-\gamma \int_\Z\langle P, G_{\varepsilon}(P)\rangle_F d \mu_t.
\end{equation}
The total Hamiltonian energy is
$$
\Ham_{\varepsilon, \gamma}(t)=K_{\varepsilon}(t)+\gamma R\left(m_t\right) .
$$
Using Equations \ref{time derivative of potential energy} and \ref{time derivative of kinetic energy},
$$
\frac{d}{d t} \Ham_{\varepsilon,\gamma}(t)=\left[\gamma \int_\Z\left\langle a_t(W), G_{\varepsilon}(P)\right\rangle_F d \mu_t-\gamma \int_\Z\langle P, G_{\varepsilon}(P)\rangle_F d \mu_t\right]-\gamma \int_\Z\left\langle a_t(W), G_{\varepsilon}(P)\right\rangle_F d \mu_t
$$
The mixed force-transport terms cancel exactly:
$$
\gamma \int_\Z\left\langle a_t(W), G_{\varepsilon}(P)\right\rangle_F d \mu_t-\gamma \int_\Z\left\langle a_t(W), G_{\varepsilon}(P)\right\rangle_F d \mu_t=0
$$
Hence, we obtain
$$
\frac{d}{d t} \Ham_{\varepsilon,\gamma}(t)=-\gamma \int_\Z\langle P, G_{\varepsilon}(P)\rangle_F d \mu_t=-\gamma \int_\Z\left\langle P, \nabla \Psi_{\varepsilon}(P)\right\rangle_F d \mu_t(W, P).
$$
Let $P=U \operatorname{diag}\left(\sigma_1, \ldots, \sigma_s\right) V^{\top}
$ be a reduced SVD of $P$, with $\sigma_i>0$. Then, we have that
$$
\nabla \Psi_{\varepsilon}(P)=U \operatorname{diag}\left(\frac{\sigma_1}{\sqrt{\sigma_1^2+\varepsilon^2}}, \ldots, \frac{\sigma_s}{\sqrt{\sigma_s^2+\varepsilon^2}}\right) V^{\top} .
$$
Therefore
$$
\left\langle P, \nabla \Psi_{\varepsilon}(P)\right\rangle_F=\operatorname{tr}\left[P^{\top} \nabla \Psi_{\varepsilon}(P)\right] .
$$
Substituting the SVD expressions, we have that $P^{\top}=V \operatorname{diag}\left(\sigma_1, \ldots, \sigma_s\right) U^{\top}$, so $
P^{\top} \nabla \Psi_{\varepsilon}(P)=V \operatorname{diag}\left(\sigma_1, \ldots, \sigma_s\right) U^{\top} U \operatorname{diag}\left(\frac{\sigma_1}{\sqrt{\sigma_1^2+\varepsilon^2}}, \ldots, \frac{\sigma_s}{\sqrt{\sigma_s^2+\varepsilon^2}}\right) V^{\top}$. Since $U^{\top} U=I$ and $V^{\top} V=I$,
$$
P^{\top} \nabla \Psi_{\varepsilon}(P)=V \operatorname{diag}\left(\frac{\sigma_1^2}{\sqrt{\sigma_1^2+\varepsilon^2}}, \ldots, \frac{\sigma_s^2}{\sqrt{\sigma_s^2+\varepsilon^2}}\right) V^{\top} .
$$
Taking the trace,
$$
\left\langle P, \nabla \Psi_{\varepsilon}(P)\right\rangle_F=\sum_{i=1}^s \frac{\sigma_i^2}{\sqrt{\sigma_i^2+\varepsilon^2}}.
$$
If we use the full singular-value list $\sigma_1, \ldots, \sigma_q$, with zero padding beyond the rank, this becomes
$$
\left\langle P, \nabla \Psi_{\varepsilon}(P)\right\rangle_F=\sum_{i=1}^q \frac{\sigma_i(P)^2}{\sqrt{\sigma_i(P)^2+\varepsilon^2}}.
$$
Each term is nonnegative. Therefore
$$
\left\langle P, \nabla \Psi_{\varepsilon}(P)\right\rangle_F \geq 0 .
$$
Moreover,
$$
\frac{\sigma_i^2}{\sqrt{\sigma_i^2+\varepsilon^2}}=0 \Longleftrightarrow \sigma_i=0 .
$$
Thus
$$
\left\langle P, \nabla \Psi_{\varepsilon}(P)\right\rangle_F=0 \quad \Longleftrightarrow \quad \sigma_i(P)=0 \text { for all } i \quad \Longleftrightarrow \quad P=0 \text {. }
$$
Hence the dissipation functional satisfies
$$
\mathcal{D}_{\varepsilon}(\mu)=0 \quad \Longleftrightarrow \quad P=0 \quad \mu \text {-a.e. } \quad \Longleftrightarrow \quad \mu^P=\delta_0 .
$$
Consequently, we have that
$$
\frac{d}{d t} \Ham_{\varepsilon,\gamma}(t) \leq 0 .
$$
\end{proof}

\subsection{Proof of Theorem \ref{Theorem:Exponential convergence under explicit assumptions} as a scalar Hilbert-space corollary}\label{app:proof-thm3-corollary}

\begin{proof}
We apply the Hilbert-domain convergence theorem, Theorem \ref{thm:convergence-gated}, with the one-block scalar-output choice
\[
        \Theta=\mathcal{X}=\mathbb{R}^{m\times n},
        \qquad
        \mathcal{H}=\mathbb{R},
\]
where \(\Theta\) carries the Frobenius inner product and \(\mathcal{H}\) carries the usual Euclidean inner product on \(\mathbb{R}\).  The product-space block number is \(B=1\), so the product norm \(\|\cdot\|_{\Theta}\) is exactly \(\|\cdot\|_F\).  We use the same scalar feature map \(F:\mathcal{X}\to\mathbb{R}\) and the same outer loss \(R:\mathbb{R}\to\mathbb{R}\).  Therefore
\[
        m_{\rho}=\int_{\Theta}F(W)\,d\rho(W)
\]
is the same moment as in the main text, and the Hilbert-space objective \(R(m_{\rho})\) is exactly \(J(\rho)\).

We next identify the Hilbert-space force with the scalar matrix-space force.  For \(u\in\mathbb{R}\) and \(V\in\mathcal{X}\), the adjoint of \(DF(W):\mathcal{X}\to\mathbb{R}\) is characterized by
\[
        \langle DF(W)^*u,V\rangle_F
        =u\,DF(W)[V]
        =u\,\langle \nabla F(W),V\rangle_F
        =\langle u\nabla F(W),V\rangle_F .
\]
Hence
\[
        DF(W)^*u=u\nabla F(W).
\]
Taking \(u=\nabla R(m_{\rho})=R'(m_{\rho})\) gives
\[
        a_{\rho}(W)=DF(W)^*\nabla R(m_{\rho})
        =R'(m_{\rho})\nabla F(W),
\]
which is precisely the force \(a_t(W)\) used in Theorem \ref{Theorem:Exponential convergence under explicit assumptions}.  Thus the quantities \(A_t,U_t,C_t\) in \eqref{eq:convergence-quantities-gated} reduce exactly to the main-text quantities in Section \ref{Convergence, particle approximation, and hard Muon}.  Likewise, since \(B=1\),
\[
        \Psi^{\Theta}_{\varepsilon}(P)=\Psi_{\varepsilon}(P),
        \qquad
        \operatorname{Orth}^{\Theta}_{\varepsilon}(P)=\operatorname{Orth}_{\varepsilon}(P),
\]
and therefore \(K_t,D_t,H_t\) in Theorem \ref{thm:convergence-gated} reduce exactly to the corresponding \(K_t,D_t,H_t\) in Theorem \ref{Theorem:Exponential convergence under explicit assumptions}.

The assumptions also specialize exactly.  The kinetic estimates \eqref{eq:kinetic-K-gated}-\eqref{eq:kinetic-LG-gated} are precisely Assumption \ref{ass:Trajectory-level convergence hypotheses}(a).  For completeness, when the momentum support satisfies \(\|P\|_F\le B_P\), these constants are obtained as follows.  For \(s\in[0,B_P]\),
\[
        \sqrt{s^2+\varepsilon^2}-\varepsilon
        =\int_0^s \frac{r}{\sqrt{r^2+\varepsilon^2}}\,dr,
        \qquad
        \frac{d}{dr}\left(\frac{r}{\sqrt{r^2+\varepsilon^2}}\right)
        =\frac{\varepsilon^2}{(r^2+\varepsilon^2)^{3/2}}
        \ge \frac{\varepsilon^2}{(B_P^2+\varepsilon^2)^{3/2}}.
\]
Since \(r/\sqrt{r^2+\varepsilon^2}\) vanishes at \(r=0\), integration gives
\[
        \sqrt{s^2+\varepsilon^2}-\varepsilon
        \ge \frac{\kappa_K}{2}s^2,
        \qquad
        \kappa_K=\frac{\varepsilon^2}{(B_P^2+\varepsilon^2)^{3/2}}.
\]
Summing over singular values gives \(\Psi_{\varepsilon}(P)\ge(\kappa_K/2)\|P\|_F^2\).  Similarly,
\[
        \frac{s^2}{\sqrt{s^2+\varepsilon^2}}
        \ge \frac{s^2}{\sqrt{B_P^2+\varepsilon^2}}
        =\kappa_Ds^2,
        \qquad
        \kappa_D=(B_P^2+\varepsilon^2)^{-1/2},
\]
which gives \(\langle P,\operatorname{Orth}_{\varepsilon}(P)\rangle_F\ge \kappa_D\|P\|_F^2\).  The inequality
\[
        \sqrt{s^2+\varepsilon^2}-\varepsilon
        \le \frac{s^2}{\sqrt{s^2+\varepsilon^2}}
\]
is equivalent to \(\varepsilon\le \sqrt{s^2+\varepsilon^2}\), so \(\Psi_{\varepsilon}(P)\le \langle P,\operatorname{Orth}_{\varepsilon}(P)\rangle_F\), i.e. \(\chi=1\).  Finally, \(\operatorname{Orth}_{\varepsilon}(0)=0\) and the \(1/\varepsilon\)-Lipschitz property imply
\[
        \|\operatorname{Orth}_{\varepsilon}(P)\|_F\le \varepsilon^{-1}\|P\|_F,
\]
so \(L_G=1/\varepsilon\).

Assumption \ref{ass:Trajectory-level convergence hypotheses}(b) is exactly the Hilbert-domain PL and upper-gradient Assumption \ref{ass:PL-gated} after the above identification.  Assumption \ref{ass:Trajectory-level convergence hypotheses}(c) is exactly the curvature assumption \eqref{eq:curvature-assumed-gated} in the scalar matrix-space notation.  If one wants primitive sufficient conditions for this curvature assumption, they are obtained from Lemma \ref{lem:curvature-gated}: if along the trajectory
\[
        |R'(m_t)|\le M_R,
        \qquad |R''(m_t)|\le M_{R,2},
        \qquad \|\nabla F(W)\|_F\le M_F,
        \qquad \|\nabla^2F(W)[V]\|_F\le M_{F,2}\|V\|_F,
\]
then, in the Hilbert notation, \(M_D=M_F\) and \(M_{D,2}=M_{F,2}\), and Lemma \ref{lem:curvature-gated} gives
\[
        C_t'=\gamma A_t-\gamma C_t+S_t,
        \qquad
        |S_t|\le \frac{L_G}{\kappa_D}\left(M_F^2M_{R,2}+M_RM_{F,2}\right)D_t.
\]
Thus the scalar curvature condition is precisely the one-block instance of the Hilbert curvature condition.

All hypotheses of Theorem \ref{thm:convergence-gated} are therefore satisfied with the same \(\lambda,\Lambda,\kappa_K,\kappa_D,\chi,L_G,\sigma\).  In this one-block specialization the alignment constant in Theorem \ref{thm:convergence-gated} is
\[
        M_C=\sqrt{\frac{\Lambda}{\gamma\kappa_K}},
\]
which is exactly the constant appearing in Theorem \ref{Theorem:Exponential convergence under explicit assumptions}.  Applying Theorem \ref{thm:convergence-gated} gives
\[
        J(\rho_t)-J_{\star}
        \le
        \frac{\exp(-c_{\alpha,r}t)}{\gamma(1-\alpha M_C)}(H_0-\alpha C_0),
\]
with
\[
        c_{\alpha,r}=\frac{1}{1+\alpha M_C}
        \min\left\{\frac{d_{\alpha,r}}{\chi},\,2\lambda\alpha\left(1-\frac r2\right)\right\},
        \qquad
        d_{\alpha,r}=\gamma-\alpha\sigma-\frac{\alpha\gamma}{2r\kappa_D}.
\]
This is exactly \eqref{Disipation rate}, and the proof is complete.
\end{proof}

\subsection{Proof of Theorem \ref{Propagation of chaos} as a scalar Hilbert-space corollary}\label{app:proof-thm4-corollary}

\begin{proof}
We apply Theorem \ref{thm:poc-gated} with
\[
        \Theta=\mathcal{X}=\mathbb{R}^{m\times n},
        \qquad
        \mathcal{H}=\mathbb{R},
\]
again using the Frobenius inner product on \(\Theta\).  The phase space \(\Z=\Theta\times\Theta\) is therefore the same as \(\mathcal{Z}=\mathcal{X}\times\mathcal{X}\) in the main text, and \(\|\cdot\|_{\Theta}=\|\cdot\|_F\).

We first verify the assumptions.  Assumption \ref{ass:Basic smoothness} gives \(\|\nabla F(W)\|_F\le M_F\) and global Lipschitzness of \(\nabla F\).  Hence, for all \(W,\widetilde W\in\mathcal{X}\),
\[
        |F(W)-F(\widetilde W)|
        \le M_F\|W-\widetilde W\|_F,
\]
so the Hilbert-space Lipschitz constant \(L_F\) in Assumption \ref{ass:global-hilbert} may be taken to be \(M_F\).  Moreover, \(DF(W)[V]=\langle \nabla F(W),V\rangle_F\), so
\[
        \|DF(W)\|_{\Theta\to\mathbb{R}}=\|\nabla F(W)\|_F\le M_F,
\]
and the Hilbert constant \(M_D\) may be taken to be \(M_F\).  The global Lipschitz constant of \(DF\) is the global Lipschitz constant of \(\nabla F\).  Finally, the theorem assumes in addition that \(R'\) is globally bounded, while Assumption \ref{ass:Basic smoothness} assumes that \(R'\) is globally Lipschitz.  Since \(\nabla R(z)=R'(z)\) in \(\mathbb{R}\), these two bounds are exactly \eqref{eq:global-R-bound}-\eqref{eq:global-R-Lip}. Thus, the scalar data satisfy Assumption \ref{ass:global-hilbert}.

Under this identification, the Hilbert-space force in \eqref{eq:force-gated} becomes
\[
        a_{\mu}(W)=DF(W)^*\nabla R(m_{\mu})=R'(m_{\mu})\nabla F(W),
\]
as shown in the proof of Theorem \ref{Theorem:Exponential convergence under explicit assumptions}.  Therefore the finite-particle ODE \eqref{eq:finite-ode-gated} becomes
\[
        \dot W_i^N(t)=-\operatorname{Orth}_{\varepsilon}(P_i^N(t)),
        \qquad
        \dot P_i^N(t)=\gamma\left(R'\left(\frac1N\sum_{j=1}^{N}F(W_j^N(t))\right)\nabla F(W_i^N(t))-P_i^N(t)\right),
\]
which is exactly Equation \eqref{Finite particle ODE limit}.  The nonlinear characteristic system \eqref{eq:nonlinear-characteristics-gated} becomes
\[
        \dot{\bar W}_i(t)=-\operatorname{Orth}_{\varepsilon}(\bar P_i(t)),
        \qquad
        \dot{\bar P}_i(t)=\gamma\left(R'(m_t)\nabla F(\bar W_i(t))-\bar P_i(t)\right),
        \qquad
        m_t=\mathbb{E}F(\bar W_i(t)),
\]
which is the mean-field copy system stated in Theorem \ref{Propagation of chaos}.

Theorem \ref{thm:poc-gated} now gives, for each fixed \(i\) and every \(T<\infty\), a constant \(C_{\mathrm{poc}}(T,\varepsilon)<\infty\) such that
\[
        \sup_{0\le t\le T}\mathbb{E}\left[
        \|W_i^N(t)-\bar W_i(t)\|_F^2+
        \|P_i^N(t)-\bar P_i(t)\|_F^2
        \right]
        \le \frac{C_{\mathrm{poc}}(T,\varepsilon)}{N}.
\]
This is precisely \eqref{eq:POC}.  The \(k\)-particle Wasserstein estimate \eqref{eq:k-chaos-gated} in Theorem \ref{thm:poc-gated} implies convergence in law of every fixed \(k\) particles to \(\mu_t^{\otimes k}\), uniformly on finite time horizons.  This proves Theorem \ref{Propagation of chaos}.
\end{proof}

\subsection{Hard-Muon subsequential limit}
\begin{theorem}\label{Hard Muon limit theorem}
    Let \(\mu_t^{\varepsilon}\) solve Equation \eqref{Hamiltonian equation} with common initial law $\mu_0 \in \cP_1(\Z)$. For every \(T<\infty\), every sequence \(\varepsilon_k \downarrow 0\) has a subsequence converging in \(C\left([0, T] ; \mathcal{P}_1(\mathcal{Z})\right)\) to a curve \(\mu_t\). There exists a Borel field \(V_t(W, P)\) such that
\[
V_t(W, P) \in \partial\|P\|_{\text {nuc }} \quad \mu_t \mathrm{~d} t \text {-a.e. }
\]
and
\[
\partial_t \mu_t+\nabla_W \cdot\left(-V_t \mu_t\right)+\nabla_P \cdot\left(\gamma\left(R^{\prime}\left(m_t\right) \nabla F(W)-P\right) \mu_t\right)=0 .
\]
If \(P\) has full rank almost everywhere, then \(V_t=\operatorname{Orth}(P)\) almost everywhere.
\end{theorem}
\begin{proof}
The limiting potential is $\Psi_0(P)=\|P\|_{\text {nuc }}$. If $P=U \Sigma V^{\top}$ has rank $r$, then

$$
\partial\|P\|_{\mathrm{nuc}}=\left\{U V^{\top}+Z: U^{\top} Z=0, Z V=0,\|Z\|_{\mathrm{op}} \leq 1\right\}
$$

Thus $\operatorname{Orth}(P)=U V^{\top}$ is a selected subgradient, and the subdifferential is singleton exactly when the rank is full in the smaller dimension.

For the compactness argument, note that $\left\|\operatorname{Orth}_{\varepsilon}(P)\right\|_F \leq \sqrt{q}$ uniformly in $\varepsilon$, and the force $R^{\prime}\left(m_t\right) \nabla F(W)$ is bounded under Assumption \ref{ass:Basic smoothness} on finite horizons after the same moment estimates as above. Along characteristics,

$$
\left\|W_t^{\varepsilon}-W_s^{\varepsilon}\right\|_F \leq \sqrt{q}|t-s|, \quad \frac{\mathrm{d}}{\mathrm{~d} t}\left\|P_t^{\varepsilon}\right\|_F \leq \gamma\left(C+\left\|P_t^{\varepsilon}\right\|_F\right) .
$$

Gronwall's inequality and the common initial law $\mu_0 \in \cP_1(\Z)$ give uniform integrability of the first moments of $\left\{\mu_t^{\varepsilon}: 0 \leq t \leq T, {\varepsilon}>0\right\}$. Together with the equicontinuity estimate in $W_1$, this yields relative compactness in $C\left([0, T] ; \cP_1(\Z)\right)$ by the Arzelà-Ascoli criterion for $W_1$-continuous probability curves.

The fluxes $\operatorname{Orth}_{\varepsilon_k}(P) \mu_t^{\varepsilon_k} \mathrm{~d} t$ are uniformly bounded vector-valued measures, so along a subsequence they converge weak-star to $V_t(W, P) \mu_t \mathrm{~d} t$ for a Borel field $V$ with $\|V\|_F \leq \sqrt{q}$. Convexity gives, for every $Q \in \mathcal{X}$,

$$
\Psi_{\varepsilon_k}(Q) \geq \Psi_{\varepsilon_k}(P)+\left\langle\operatorname{Orth}_{\varepsilon_k}(P), Q-P\right\rangle_F.
$$

Further $\Psi_{\varepsilon} \rightarrow \Psi_0$ locally uniformly. To pass to the limit with the unbounded test function $P$, first multiply by a compactly supported cutoff in $P$, pass to the limit using local uniform convergence $\Psi_{\varepsilon} \rightarrow \Psi_0$, and then remove the cutoff using the uniform first-moment bound. Therefore, using this standard cutoff-based strategy to handle the unboundedness of $P$ and passing to the limit in the integrated inequality with arbitrary nonnegative test weights yields
$$
\Psi_0(Q) \geq \Psi_0(P)+\left\langle V_t(W, P), Q-P\right\rangle_F
$$

for $\mu_t \mathrm{~d} t$-almost every $(W, P)$. This is precisely $V_t(W, P) \in \partial \Psi_0(P)$. Passing to the limit in the weak formulation gives the hard-Muon inclusion. If $P$ has full rank almost everywhere, the subdifferential formula forces $V_t=\operatorname{Orth}(P)$.
\end{proof}
\begin{remark}
[Why only subsequential uniqueness] At rank deficient \(P, \partial\|P\|_{\text {nuc }}\) is not a singleton. The regularized maps select limits of subgradients, but without an additional selection or uniqueness principle the hard limit is naturally a differential inclusion.\end{remark}

\subsection{Hilbert-valued probability functionals on product spaces}\label{sec:gated-hilbert-functional}

The scalar $F$-output theory in the main paper is based on a probability distribution over a single matrix space. In particular, the one-hidden layer neural network teacher-student mean-field and the transformer mixture-of-experts setting requires a probability distribution over parameter tuples and an expert-router parameter tuple, respectively. The parameter space is therefore enlarged before the variational and dynamical constructions are introduced.

\subsubsection{Extended product parameter space}\label{subsec:extended-theta}

Let
\begin{equation}\label{eq:theta-product-general}
        \Theta=\Thetaexp\times\Thetagate
        =\prod_{b=1}^{B}\R^{m_b\times n_b}
\end{equation}
be a finite-dimensional real Hilbert space.  The first $B_{\mathrm{exp}}$ blocks may represent expert parameters, while the remaining $B_{\mathrm{gate}}=B-B_{\mathrm{exp}}$ blocks represent router or gating parameters.  A generic point is written as
\begin{equation}\label{eq:theta-split}
        \theta=(\omega,\phi)=\left(\theta^{(1)},\ldots,\theta^{(B)}\right),
\end{equation}
where $\omega\in\Thetaexp$ and $\phi\in\Thetagate$.  The product Hilbert inner product is
\begin{equation}\label{eq:theta-inner-product}
        \ip{\theta}{\vartheta}_{\Theta}
        \coloneqq \sum_{b=1}^{B}\ip{\theta^{(b)}}{\vartheta^{(b)}}_{F},
        \qquad
        \norm{\theta}_{\Theta}^{2}
        \coloneqq \sum_{b=1}^{B}\norm{\theta^{(b)}}_{F}^{2}.
\end{equation}
The Wasserstein spaces below are built from the metric induced by $\norm{\cdot}_{\Theta}$.  The routing parameter is part of the state variable.  Consequently, any force field, momentum variable, mirror map, and Hamiltonian gradient is defined on the same extended space $\Theta$ and not only on the expert subspace.

Vector-valued and scalar router parameters can be included in \eqref{eq:theta-product-general} by treating vectors as one-column matrices and scalars as $1\times 1$ matrices.  Biases may also be absorbed into matrix blocks by appending a homogeneous coordinate to the input representation.  Thus the product matrix formulation covers the usual affine router scores without introducing a separate notation.

\subsubsection{Hilbert-valued feature maps with routing included}\label{subsec:hilbert-feature-routing}

Let $\cH$ be a real Hilbert space and let
\begin{equation}\label{eq:F-R-general-theta}
        F:\Theta\to\cH,
        \qquad
        \Rfunc:\cH\to\R.
\end{equation}
For $\rho\in\cP_1(\Theta)$ define the Bochner integral
\begin{equation}\label{eq:m-rho-theta}
        m_{\rho}\coloneqq \int_{\Theta}F(\theta)\dd\rho(\theta)
\end{equation}
whenever it is finite, and set
\begin{equation}\label{eq:J-rho-theta}
        \tf(\rho)\coloneqq \Rfunc(m_{\rho}).
\end{equation}
The map $F$ is allowed to contain the router score, the gate weight, the expert output, and any Hilbert-valued augmentation needed to express normalized gates.  In particular, if $\theta=(\omega,\phi)$ and $\psi_{\omega}$ is an expert output while $s_{\phi}$ is an input-dependent router score, then maps of the form
\begin{equation}\label{eq:generic-gated-feature}
        F(\omega,\phi)=\text{``input-indexed function involving }s_{\phi}\text{ and }\psi_{\omega}\text{''}
\end{equation}
are ordinary Hilbert-valued feature maps on $\Theta$.

\begin{proposition}[First variation on the extended gated space]\label{prop:first-var-gated}
Assume that $F$ is Bochner integrable under $\rho$ and that $\Rfunc$ is Frechet differentiable at $m_{\rho}$.  Then a valid first variation of $\tf$ at $\rho$ is
\begin{equation}\label{eq:first-var-gated}
        \frac{\delta\tf}{\delta\rho}(\rho)(\theta)
        =\ip{\nabla\Rfunc(m_{\rho})}{F(\theta)}_{\cH},
\end{equation}
up to an additive constant independent of $\theta$.  If $F$ is Frechet differentiable, then the Wasserstein force on the extended expert-router parameter is
\begin{equation}\label{eq:force-gated}
        a_{\rho}(\theta)
        \coloneqq \nabla_{\theta}\frac{\delta\tf}{\delta\rho}(\rho)(\theta)
        =DF(\theta)^{*}\nabla\Rfunc(m_{\rho})\in\Theta.
\end{equation}
If $\theta=(\omega,\phi)$, then
\begin{equation}\label{eq:force-split-gated}
        a_{\rho}(\theta)=\left(a_{\rho}^{\mathrm{exp}}(\omega,\phi),a_{\rho}^{\mathrm{gate}}(\omega,\phi)\right),
\end{equation}
where the two components are the projections of $DF(\omega,\phi)^{*}\nabla\Rfunc(m_{\rho})$ onto $\Thetaexp$ and $\Thetagate$, respectively.
\end{proposition}

\begin{proof}
Let $\nu$ be a finite signed measure on $\Theta$ with $\nu(\Theta)=0$ and define $\rho_s=\rho+s\nu$ for $s$ in an interval on which the perturbation is meaningful.  By linearity of the Bochner integral,
\[
        m_{\rho_s}=m_{\rho}+s\int_{\Theta}F(\theta)\dd\nu(\theta).
\]
Frechet differentiability of $\Rfunc$ gives
\begin{align*}
        \left.\frac{\dd}{\dd s}\tf(\rho_s)\right|_{s=0}
        &=D\Rfunc(m_{\rho})\left[\int_{\Theta}F(\theta)\dd\nu(\theta)\right]  \\
        &=\ip{\nabla\Rfunc(m_{\rho})}{\int_{\Theta}F(\theta)\dd\nu(\theta)}_{\cH} \\
        &=\int_{\Theta}\ip{\nabla\Rfunc(m_{\rho})}{F(\theta)}_{\cH}\dd\nu(\theta).
\end{align*}
This is precisely the defining identity for the first variation, modulo constants independent of $\theta$.  If $F$ is differentiable, then for every $v\in\Theta$,
\begin{align*}
        D_{\theta}\left[\ip{\nabla\Rfunc(m_{\rho})}{F(\theta)}_{\cH}\right][v]
        &=\ip{\nabla\Rfunc(m_{\rho})}{DF(\theta)[v]}_{\cH} \\
        &=\ip{DF(\theta)^{*}\nabla\Rfunc(m_{\rho})}{v}_{\Theta}.
\end{align*}
The Riesz representation in the Hilbert space $\Theta$ gives \eqref{eq:force-gated}.  The decomposition \eqref{eq:force-split-gated} follows from the orthogonal product decomposition $\Theta=\Thetaexp\times\Thetagate$.
\end{proof}

\subsubsection{Regularity assumptions}\label{subsec:regularity-gated}

The smooth ODE and PDE theory can be stated under either global assumptions or localized finite-horizon assumptions.  The localized version is the natural one for transformer models, because attention, feed-forward, and router score maps are smooth on bounded parameter sets, while global Lipschitz constants on all of parameter space are generally unavailable.

\begin{asu}[Localized Hilbert smoothness on the extended space]\label{ass:localized-hilbert}
For every $R_{\Theta}<\infty$ and every $R_{\cH}<\infty$, the following quantities are finite on the indicated balls:
\begin{align*}
        M_F(R_{\Theta})&\coloneqq\sup_{\norm{\theta}_{\Theta}\le R_{\Theta}}\norm{F(\theta)}_{\cH},\qquad
        L_F(R_{\Theta})\coloneqq\sup_{\substack{\norm{\theta}_{\Theta},\norm{\vartheta}_{\Theta}\le R_{\Theta}\\ \theta\ne\vartheta}}
        \frac{\norm{F(\theta)-F(\vartheta)}_{\cH}}{\norm{\theta-\vartheta}_{\Theta}},\\
        M_D(R_{\Theta})&\coloneqq\sup_{\norm{\theta}_{\Theta}\le R_{\Theta}}\norm{DF(\theta)}_{\Theta\to\cH},\qquad
        L_D(R_{\Theta})\coloneqq\sup_{\substack{\norm{\theta}_{\Theta},\norm{\vartheta}_{\Theta}\le R_{\Theta}\\ \theta\ne\vartheta}}
        \frac{\norm{DF(\theta)-DF(\vartheta)}_{\Theta\to\cH}}{\norm{\theta-\vartheta}_{\Theta}},\\
        M_R(R_{\cH})&\coloneqq\sup_{\norm{z}_{\cH}\le R_{\cH}}\norm{\nabla\Rfunc(z)}_{\cH},\qquad
        L_R(R_{\cH})\coloneqq\sup_{\substack{\norm{z}_{\cH},\norm{z'}_{\cH}\le R_{\cH}\\ z\ne z'}}
        \frac{\norm{\nabla\Rfunc(z)-\nabla\Rfunc(z')}_{\cH}}{\norm{z-z'}_{\cH}}.
\end{align*}
When second-order convergence estimates are invoked, $F$ is twice Frechet differentiable on bounded $\Theta$-balls and $\Rfunc$ is twice Frechet differentiable on bounded $\cH$-balls, with bounded second derivatives on those balls.
\end{asu}

\begin{asu}[Global Hilbert smoothness]\label{ass:global-hilbert}
There exist constants $L_F,M_D,L_D,M_R,L_R<\infty$ such that for all $\theta,\vartheta\in\Theta$ and all $z,z'\in\cH$,
\begin{align}
        \norm{F(\theta)-F(\vartheta)}_{\cH}&\le L_F\norm{\theta-\vartheta}_{\Theta},\label{eq:global-F-Lip}\\
        \norm{DF(\theta)}_{\Theta\to\cH}&\le M_D,\label{eq:global-DF-bound}\\
        \norm{DF(\theta)-DF(\vartheta)}_{\Theta\to\cH}&\le L_D\norm{\theta-\vartheta}_{\Theta},\label{eq:global-DF-Lip}\\
        \norm{\nabla\Rfunc(z)}_{\cH}&\le M_R,\label{eq:global-R-bound}\\
        \norm{\nabla\Rfunc(z)-\nabla\Rfunc(z')}_{\cH}&\le L_R\norm{z-z'}_{\cH}.\label{eq:global-R-Lip}
\end{align}
\end{asu}

\begin{lemma}[Force estimates]\label{lem:force-estimates-gated}
Under Assumption \ref{ass:global-hilbert}, for every $\rho,\tilde\rho\in\cP_1(\Theta)$ and $\theta,\tilde\theta\in\Theta$,
\begin{equation}\label{eq:force-estimate-global-gated}
        \norm{a_{\rho}(\theta)-a_{\tilde\rho}(\tilde\theta)}_{\Theta}
        \le L_D M_R\norm{\theta-\tilde\theta}_{\Theta}+M_D L_R L_F W_1(\rho,\tilde\rho).
\end{equation}
Furthermore,
\begin{equation}\label{eq:force-bound-global-gated}
        \norm{a_{\rho}(\theta)}_{\Theta}\le M_D M_R.
\end{equation}
Under Assumption \ref{ass:localized-hilbert}, the same estimates hold with the corresponding localized constants whenever $\theta,\tilde\theta$ and the supports of $\rho,\tilde\rho$ lie in a common bounded $\Theta$-ball and $m_{\rho},m_{\tilde\rho}$ lie in a common bounded $\cH$-ball.
\end{lemma}

\begin{proof}
Let $u_{\rho}\coloneqq\nabla\Rfunc(m_{\rho})$.  Then
\begin{align*}
        a_{\rho}(\theta)-a_{\tilde\rho}(\tilde\theta)
        &=\left(DF(\theta)^{*}-DF(\tilde\theta)^{*}\right)u_{\rho}
          +DF(\tilde\theta)^{*}\left(u_{\rho}-u_{\tilde\rho}\right).
\end{align*}
The first term has norm at most $L_D\norm{\theta-\tilde\theta}_{\Theta}M_R$.  The second has norm at most
\[
        M_D L_R\norm{m_{\rho}-m_{\tilde\rho}}_{\cH}.
\]
For every coupling $\pi$ of $\rho$ and $\tilde\rho$,
\[
        m_{\rho}-m_{\tilde\rho}=\int_{\Theta\times\Theta}\left(F(\xi)-F(\eta)\right)\dd\pi(\xi,\eta),
\]
so
\[
        \norm{m_{\rho}-m_{\tilde\rho}}_{\cH}
        \le L_F\int\norm{\xi-\eta}_{\Theta}\dd\pi(\xi,\eta).
\]
Taking the infimum over $\pi$ yields $\norm{m_{\rho}-m_{\tilde\rho}}_{\cH}\le L_F W_1(\rho,\tilde\rho)$.  This proves \eqref{eq:force-estimate-global-gated}.  The bound \eqref{eq:force-bound-global-gated} follows directly from $\norm{DF(\theta)}\le M_D$ and $\norm{u_{\rho}}\le M_R$.  The localized statement is identical after restricting every estimate to the relevant bounded balls.
\end{proof}

\subsection{Regularized Muon mirror geometry on the extended product space}\label{sec:extended-mirror-geometry}

The regularized Muon mirror map is defined on the full product space $\Theta$, including both expert and router blocks.  The construction is block-separable, and therefore the Fenchel conjugacy and mirror-update interpretation are inherited from the single-matrix spectral construction.

\subsubsection{Single block Fenchel pair}\label{subsec:single-block-fenchel}

Fix a matrix block $\R^{m_b\times n_b}$ and set $q_b=\min(m_b,n_b)$.  For $P\in\R^{m_b\times n_b}$ define
\begin{equation}\label{eq:Psi-single-block-gated}
        \Psi_{\eps,b}(P)
        \coloneqq\sum_{r=1}^{q_b}\left(\sqrt{\sigma_r(P)^2+\eps^2}-\eps\right),
\end{equation}
where the singular values are padded by zeros.  For $G\in\R^{m_b\times n_b}$ define
\begin{equation}\label{eq:Phi-single-block-gated}
        \Phi_{\eps,b}(G)
        \coloneqq
        \begin{cases}
        \displaystyle \eps\sum_{r=1}^{q_b}\left(1-\sqrt{1-\sigma_r(G)^2}\right),& \norm{G}_{\op}\le 1,\\[0.8em]
        +\infty,&\norm{G}_{\op}>1.
        \end{cases}
\end{equation}
On the relative interior $\norm{G}_{\op}<1$, $\Phi_{\eps,b}$ is differentiable.  If $P=U\diag(\sigma_1,\ldots,\sigma_s)V^{\top}$ is a reduced SVD, then
\begin{equation}\label{eq:Orth-eps-single-block-gated}
        \nabla\Psi_{\eps,b}(P)
        =U\diag\left(\frac{\sigma_1}{\sqrt{\sigma_1^2+\eps^2}},\ldots,
        \frac{\sigma_s}{\sqrt{\sigma_s^2+\eps^2}}\right)V^{\top}
        \eqqcolon \operatorname{Orth}_{\eps,b}(P).
\end{equation}
Moreover,
\begin{equation}\label{eq:single-block-bounds-gated}
        \norm{\operatorname{Orth}_{\eps,b}(P)}_{\op}<1,
        \qquad
        \norm{\operatorname{Orth}_{\eps,b}(P)}_{F}\le\sqrt{q_b},
        \qquad
        \norm{\operatorname{Orth}_{\eps,b}(P)-\operatorname{Orth}_{\eps,b}(\tilde P)}_{F}\le \frac1\eps\norm{P-\tilde P}_{F}.
\end{equation}
The scalar pair
\[
        a\mapsto \sqrt{a^2+
\eps^2}-\eps,
        \qquad
        b\mapsto \eps(1-\sqrt{1-b^2})+\iota_{[-1,1]}(b)
\]
is Fenchel conjugate.  By von Neumann's trace inequality and the standard spectral-function conjugacy theorem, this scalar conjugacy lifts to
\begin{equation}\label{eq:single-block-conjugacy-gated}
        \Phi_{\eps,b}^{*}=\Psi_{\eps,b},
        \qquad
        \Psi_{\eps,b}^{*}=\Phi_{\eps,b}.
\end{equation}

\subsubsection{Product Fenchel pair, including router blocks}\label{subsec:product-fenchel-gate}

For $P=(P^{(1)},\ldots,P^{(B)})\in\Theta$ and $G=(G^{(1)},\ldots,G^{(B)})\in\Theta$, define
\begin{equation}\label{eq:product-Psi-Phi-gated}
        \Psieps^{\Theta}(P)
        \coloneqq\sum_{b=1}^{B}\Psi_{\eps,b}(P^{(b)}),
        \qquad
        \Phieps^{\Theta}(G)
        \coloneqq\sum_{b=1}^{B}\Phi_{\eps,b}(G^{(b)}).
\end{equation}
The product regularized Muon map is
\begin{equation}\label{eq:product-Orth-gated}
        \Ortheps^{\Theta}(P)
        \coloneqq\nabla\Psieps^{\Theta}(P)
        =\left(\operatorname{Orth}_{\eps,1}(P^{(1)}),\ldots,\operatorname{Orth}_{\eps,B}(P^{(B)})\right).
\end{equation}
The total block rank parameter is
\begin{equation}\label{eq:q-total-gated}
        \qtot\coloneqq\sum_{b=1}^{B}q_b.
\end{equation}
Equations \eqref{eq:single-block-bounds-gated} imply
\begin{equation}\label{eq:product-Orth-Lip-gated}
        \norm{\Ortheps^{\Theta}(P)-\Ortheps^{\Theta}(\tilde P)}_{\Theta}
        \le \frac1\eps\norm{P-\tilde P}_{\Theta},
        \qquad
        \norm{\Ortheps^{\Theta}(P)}_{\Theta}\le\sqrt{\qtot}.
\end{equation}

\begin{proposition}[Fenchel duality and mirror update on the extended expert-router space]\label{prop:extended-fenchel-mirror}
With respect to the product pairing \eqref{eq:theta-inner-product},
\begin{equation}\label{eq:extended-fenchel-dual}
        (\Phieps^{\Theta})^{*}=\Psieps^{\Theta},
        \qquad
        (\Psieps^{\Theta})^{*}=\Phieps^{\Theta}.
\end{equation}
For every $P\in\Theta$, the variational problem
\begin{equation}\label{eq:product-mirror-min-gated}
        G_{\eps}(P)
        =\argmin_{G\in\Theta}\left\{\ip{P}{G}_{\Theta}+\Phieps^{\Theta}(G)\right\}
\end{equation}
has the unique solution
\begin{equation}\label{eq:product-mirror-solution-gated}
        G_{\eps}(P)=-\nabla\Psieps^{\Theta}(P)=-\Ortheps^{\Theta}(P).
\end{equation}
Equivalently, if $G_{\eps,k}=-\Ortheps^{\Theta}(P_k)$, then
\begin{equation}\label{eq:mirror-bregman-gated}
        G_{\eps,k+1}
        =\argmin_{G\in\Theta}\left\{\ip{P_{k+1}-P_k}{G}_{\Theta}+D_{\Phieps^{\Theta}}(G,G_{\eps,k})\right\},
\end{equation}
where
\[
        D_{\Phieps^{\Theta}}(G,H)
        =\Phieps^{\Theta}(G)-\Phieps^{\Theta}(H)-\ip{\nabla\Phieps^{\Theta}(H)}{G-H}_{\Theta}
\]
whenever $H$ lies in the differentiability domain of $\Phieps^{\Theta}$.
\end{proposition}

\begin{proof}
By separability and \eqref{eq:single-block-conjugacy-gated},
\begin{align*}
        (\Phieps^{\Theta})^{*}(P)
        &=\sup_{G\in\Theta}\left\{\sum_{b=1}^{B}\ip{P^{(b)}}{G^{(b)}}_{F}-\sum_{b=1}^{B}\Phi_{\eps,b}(G^{(b)})\right\}\\
        &=\sum_{b=1}^{B}\sup_{G^{(b)}\in\R^{m_b\times n_b}}
          \left\{\ip{P^{(b)}}{G^{(b)}}_{F}-\Phi_{\eps,b}(G^{(b)})\right\}\\
        &=\sum_{b=1}^{B}\Psi_{\eps,b}(P^{(b)})=\Psieps^{\Theta}(P).
\end{align*}
The proof of $(\Psieps^{\Theta})^{*}=\Phieps^{\Theta}$ is identical.

The first-order optimality condition for \eqref{eq:product-mirror-min-gated} is
\[
        0=P+\nabla\Phieps^{\Theta}(G_{\eps}(P)).
\]
Since $\nabla\Phieps^{\Theta}$ and $\nabla\Psieps^{\Theta}$ are inverse maps between the interior of the effective domain of $\Phieps^{\Theta}$ and $\Theta$, and since $\Psieps^{\Theta}$ is even, its gradient is odd.  Therefore
\[
        G_{\eps}(P)=\nabla\Psieps^{\Theta}(-P)=-\nabla\Psieps^{\Theta}(P).
\]
This gives \eqref{eq:product-mirror-solution-gated}.

For \eqref{eq:mirror-bregman-gated}, use $\nabla\Phieps^{\Theta}(G_{\eps,k})=-P_k$.  Then
\begin{align*}
        D_{\Phieps^{\Theta}}(G,G_{\eps,k})
        &=\Phieps^{\Theta}(G)-\Phieps^{\Theta}(G_{\eps,k})
          -\ip{-P_k}{G-G_{\eps,k}}_{\Theta}\\
        &=\Phieps^{\Theta}(G)+\ip{P_k}{G}_{\Theta}+C_k,
\end{align*}
where $C_k$ is independent of $G$.  Adding $\ip{P_{k+1}-P_k}{G}_{\Theta}$ gives $\ip{P_{k+1}}{G}_{\Theta}+\Phieps^{\Theta}(G)+C_k$.  Thus the minimizer is exactly the minimizer of \eqref{eq:product-mirror-min-gated} with $P=P_{k+1}$.
\end{proof}

\begin{remark}[Persistence of the mirror-map interpretation]\label{rem:mirror-survives-gate}
The mirror-map interpretation survives on the extended gated product space because the mirror potential is defined on every block of $\Theta$, including the router blocks.  If $P=(P^{\mathrm{exp}},P^{\mathrm{gate}})$, then
\[
        \Ortheps^{\Theta}(P)=\left(\Ortheps^{\mathrm{exp}}(P^{\mathrm{exp}}),\Ortheps^{\mathrm{gate}}(P^{\mathrm{gate}})\right),
\]
and the Bregman update \eqref{eq:mirror-bregman-gated} is a single mirror step in the full expert-router Hilbert space.  There is no separate post-hoc extension of the mirror map after the gate is introduced; the router geometry is part of the product Fenchel pair from the beginning.
\end{remark}

\subsection{Finite-particle lift and exact regularized Muon scheme}\label{sec:finite-particle-gated}

Let $N\in\N$ and let
\[
        \boldsymbol\theta=(\theta_1,\ldots,\theta_N)\in\Theta^N,
        \qquad
        \theta_i=(\omega_i,\phi_i).
\]
The empirical measure is
\begin{equation}\label{eq:empirical-theta-gated}
        \rho^N_{\boldsymbol\theta}\coloneqq \frac1N\sum_{i=1}^{N}\delta_{\theta_i}.
\end{equation}
The lifted finite-particle objective is
\begin{equation}\label{eq:finite-lift-gated}
        \tf_N(\boldsymbol\theta)
        \coloneqq \tf(\rho^N_{\boldsymbol\theta})
        =\Rfunc\left(\frac1N\sum_{i=1}^{N}F(\theta_i)\right).
\end{equation}
The mean-field inner product on $\Theta^N$ is
\begin{equation}\label{eq:thetaN-avg-inner-gated}
        \ip{U}{V}_{\mf}
        \coloneqq\frac1N\sum_{i=1}^{N}\ip{U_i}{V_i}_{\Theta},
        \qquad
        \norm{U}_{\mf}^{2}
        \coloneqq\frac1N\sum_{i=1}^{N}\norm{U_i}_{\Theta}^{2}.
\end{equation}
Define
\begin{equation}\label{eq:mN-uN-gated}
        m_N(\boldsymbol\theta)\coloneqq\frac1N\sum_{j=1}^{N}F(\theta_j),
        \qquad
        u_N(\boldsymbol\theta)\coloneqq\nabla\Rfunc(m_N(\boldsymbol\theta)).
\end{equation}

\begin{proposition}[Particle gradient on the extended space]\label{prop:particle-gradient-gated}
Assume $F$ and $\Rfunc$ are differentiable.  The gradient of $\tf_N$ with respect to \eqref{eq:thetaN-avg-inner-gated} is
\begin{equation}\label{eq:particle-gradient-gated}
        \gradmf \tf_N(\boldsymbol\theta)=a^N(\boldsymbol\theta)
        =\left(a_1^N(\boldsymbol\theta),\ldots,a_N^N(\boldsymbol\theta)\right),
\end{equation}
where
\begin{equation}\label{eq:particle-force-gated}
        a_i^N(\boldsymbol\theta)
        \coloneqq DF(\theta_i)^{*}u_N(\boldsymbol\theta)\in\Theta.
\end{equation}
In particular, the gate component of $a_i^N$ is the gradient of the lifted objective with respect to the router blocks of particle $i$.
\end{proposition}

\begin{proof}
For $U=(U_1,\ldots,U_N)\in\Theta^N$, the chain rule gives
\begin{align*}
        D\tf_N(\boldsymbol\theta)[U]
        &=\ip{\nabla\Rfunc(m_N(\boldsymbol\theta))}{\frac1N\sum_{i=1}^{N}DF(\theta_i)[U_i]}_{\cH}\\
        &=\frac1N\sum_{i=1}^{N}\ip{DF(\theta_i)^{*}u_N(\boldsymbol\theta)}{U_i}_{\Theta}\\
        &=\ip{a^N(\boldsymbol\theta)}{U}_{\mf}.
\end{align*}
This is the defining identity for the mean-field gradient.
\end{proof}

The particle-level mirror potentials are
\begin{equation}\label{eq:particle-mirror-potentials-gated}
        \mPsieps^{N}(P)
        \coloneqq\frac1N\sum_{i=1}^{N}\Psieps^{\Theta}(P_i),
        \qquad
        \mPhieps^{N}(G)
        \coloneqq\frac1N\sum_{i=1}^{N}\Phieps^{\Theta}(G_i).
\end{equation}
With respect to $\ip{\cdot}{\cdot}_{\mf}$, the conjugacy $(\mPhieps^{N})^{*}=\mPsieps^{N}$ and $(\mPsieps^{N})^{*}=\mPhieps^{N}$ follows by the same separability proof as Proposition \ref{prop:extended-fenchel-mirror}.

The exact finite-particle regularized Muon scheme on the extended expert-router space is
\begin{equation}\label{eq:discrete-gated-muon-full}
\begin{aligned}
        P_{i,k+1}&=\beta P_{i,k}+(1-\beta)a_i^N(\boldsymbol\theta_k),\\
        G_{i,k+1}&=\argmin_{G\in\Theta}\left\{\ip{P_{i,k+1}}{G}_{\Theta}+\Phieps^{\Theta}(G)\right\},\\
        \theta_{i,k+1}&=\theta_{i,k}+\eta G_{i,k+1}.
\end{aligned}
\end{equation}
Using Proposition \ref{prop:extended-fenchel-mirror}, this is equivalently
\begin{equation}\label{eq:discrete-gated-muon-reduced}
\begin{aligned}
        P_{i,k+1}&=\beta P_{i,k}+(1-\beta)DF(\theta_{i,k})^{*}\nabla\Rfunc\left(\frac1N\sum_{j=1}^{N}F(\theta_{j,k})\right),\\
        \theta_{i,k+1}&=\theta_{i,k}-\eta\Ortheps^{\Theta}(P_{i,k+1}).
\end{aligned}
\end{equation}
Writing $\theta_i=(\omega_i,\phi_i)$ and $P_i=(P_i^{\mathrm{exp}},P_i^{\mathrm{gate}})$, the second line updates expert and router blocks simultaneously:
\begin{equation}\label{eq:expert-router-update-split}
        \omega_{i,k+1}=\omega_{i,k}-\eta\Ortheps^{\mathrm{exp}}(P_{i,k+1}^{\mathrm{exp}}),
        \qquad
        \phi_{i,k+1}=\phi_{i,k}-\eta\Ortheps^{\mathrm{gate}}(P_{i,k+1}^{\mathrm{gate}}).
\end{equation}

\subsection{Finite-particle continuous-time limit}\label{sec:finite-ode-gated}

Let the inertial scaling be
\begin{equation}\label{eq:inertial-scaling-gated}
        \eta_h=h,
        \qquad
        \beta_h=1-\gamma h+r_h,
        \qquad
        \frac{r_h}{h}\to0,
\end{equation}
where $\gamma>0$ is fixed.  The common second-order-consistent choice is $r_h=O(h^2)$.

\begin{theorem}[Finite-$N$ ODE limit on the extended gated space]\label{thm:finite-ode-limit-gated}
Fix $N\in\N$, $\eps>0$, $\gamma>0$, and initial data $(\boldsymbol\theta_0,P_0)\in\Theta^N\times\Theta^N$.  Suppose Assumption \ref{ass:global-hilbert} holds.  Then the ODE system
\begin{equation}\label{eq:finite-ode-gated}
\begin{aligned}
        \dot\theta_i(t)&=-\Ortheps^{\Theta}(P_i(t)),\\
        \dot P_i(t)&=\gamma\left(DF(\theta_i(t))^{*}\nabla\Rfunc\left(\frac1N\sum_{j=1}^{N}F(\theta_j(t))\right)-P_i(t)\right),
\end{aligned}
\end{equation}
for $i=1,\ldots,N$, has a unique global solution.  Let $(\boldsymbol\theta_k^h,P_k^h)$ be the iterates of \eqref{eq:discrete-gated-muon-reduced} with scaling \eqref{eq:inertial-scaling-gated}, and let $(\boldsymbol\theta^h(t),P^h(t))$ be the continuous piecewise-linear interpolation satisfying $(\boldsymbol\theta^h(kh),P^h(kh))=(\boldsymbol\theta_k^h,P_k^h)$.  Then, for every $T<\infty$,
\begin{equation}\label{eq:finite-ode-convergence-gated}
        \sup_{0\le t\le T}\left(\norm{\boldsymbol\theta^h(t)-\boldsymbol\theta(t)}_{\mf}
        +\norm{P^h(t)-P(t)}_{\mf}\right)\to0
\end{equation}
as $h\downarrow0$.  If $r_h=O(h^2)$, the convergence rate is $O(h)$ on every finite time interval.
\end{theorem}

\begin{proof}
Define $a^N$ by \eqref{eq:particle-force-gated}.  The vector field on $\Theta^N\times\Theta^N$ is
\begin{equation}\label{eq:BN-gated}
        B_N(\boldsymbol\theta,P)
        =\left(-\Ortheps^{\Theta}(P_1),\ldots,-\Ortheps^{\Theta}(P_N),
        \gamma(a_1^N(\boldsymbol\theta)-P_1),\ldots,
        \gamma(a_N^N(\boldsymbol\theta)-P_N)\right).
\end{equation}
The map $P\mapsto\Ortheps^{\Theta}(P)$ is globally $1/\eps$-Lipschitz by \eqref{eq:product-Orth-Lip-gated}.  Lemma \ref{lem:force-estimates-gated} implies a global Lipschitz estimate for $a^N$ in the mean-field norm.  Indeed, for two particle configurations $\boldsymbol\theta,\tilde{\boldsymbol\theta}\in\Theta^N$,
\[
        \norm{m_N(\boldsymbol\theta)-m_N(\tilde{\boldsymbol\theta})}_{\cH}
        \le \frac1N\sum_{j=1}^{N}L_F\norm{\theta_j-\tilde\theta_j}_{\Theta}
        \le L_F\norm{\boldsymbol\theta-\tilde{\boldsymbol\theta}}_{\mf}.
\]
Consequently,
\[
        \norm{a_i^N(\boldsymbol\theta)-a_i^N(\tilde{\boldsymbol\theta})}_{\Theta}
        \le L_D M_R\norm{\theta_i-\tilde\theta_i}_{\Theta}
        +M_D L_R L_F\norm{\boldsymbol\theta-\tilde{\boldsymbol\theta}}_{\mf}.
\]
Squaring, averaging over $i$, and using $(x+y)^2\le2x^2+2y^2$ yields a global Lipschitz bound for $a^N$.  Hence $B_N$ is globally Lipschitz, and the Picard-Lindelof theorem gives a unique global solution of \eqref{eq:finite-ode-gated}.

The discrete momentum update satisfies
\begin{align*}
        P_{i,k+1}^h-P_{i,k}^h
        &=(1-\beta_h)\left(a_i^N(\boldsymbol\theta_k^h)-P_{i,k}^h\right)\\
        &=(\gamma h-r_h)\left(a_i^N(\boldsymbol\theta_k^h)-P_{i,k}^h\right)\\
        &=h\gamma\left(a_i^N(\boldsymbol\theta_k^h)-P_{i,k}^h\right)+o(h)
\end{align*}
locally uniformly on bounded sets.  For the position update,
\begin{align*}
        \theta_{i,k+1}^h-\theta_{i,k}^h
        &=-h\Ortheps^{\Theta}(P_{i,k+1}^h)\\
        &=-h\Ortheps^{\Theta}(P_{i,k}^h)
        -h\left(\Ortheps^{\Theta}(P_{i,k+1}^h)-\Ortheps^{\Theta}(P_{i,k}^h)\right).
\end{align*}
Since $\Ortheps^{\Theta}$ is $1/\eps$-Lipschitz and $P_{i,k+1}^h-P_{i,k}^h=O(h)$ on bounded sets, the second term is $O(h^2)$.  Therefore the one-step local truncation error of the scheme relative to the Euler discretization of \eqref{eq:finite-ode-gated} is $O(h^2)+o(h)h$.  A standard discrete Gronwall estimate for globally Lipschitz ODEs gives uniform convergence on $[0,T]$.  If $r_h=O(h^2)$, the local truncation error is $O(h^2)$ and the global error is $O(h)$.
\end{proof}

\begin{remark}[Localized finite-horizon form]\label{rem:localized-ode-gated}
The transformer specialization generally satisfies Assumption \ref{ass:localized-hilbert}, not Assumption \ref{ass:global-hilbert}.  The conclusion of Theorem \ref{thm:finite-ode-limit-gated} remains valid on each finite interval $[0,T]$ under Assumption \ref{ass:localized-hilbert}.  Indeed, the position velocity obeys
\[
        \norm{\dot\theta_i(t)}_{\Theta}=\norm{\Ortheps^{\Theta}(P_i(t))}_{\Theta}\le\sqrt{\qtot},
\]
so
\[
        \norm{\theta_i(t)}_{\Theta}\le \norm{\theta_i(0)}_{\Theta}+T\sqrt{\qtot}.
\]
All force-field constants needed on $[0,T]$ are therefore evaluated on a bounded parameter ball.  Once $\theta_i(t)$ is confined to that ball, the comparison inequality
\[
        \frac{\dd}{\dd t}\norm{P_i(t)}_{\Theta}
        \le \gamma\left(\sup_{0\le s\le T}\norm{a_i^N(\boldsymbol\theta(s))}_{\Theta}-\norm{P_i(t)}_{\Theta}\right)
\]
precludes finite-time momentum blow-up.  Local Lipschitzness on the resulting bounded phase-space region gives uniqueness and the same Euler convergence argument.
\end{remark}

\subsection{Phase-space mean-field equation}\label{sec:mean-field-gated}

Set
\begin{equation}\label{eq:phase-space-gated}
        \Z\coloneqq\Theta\times\Theta,
        \qquad
        z=(\theta,p).
\end{equation}
For $\mu\in\cP_1(\Z)$ let $\rho=(\pi_{\theta})_{\#}\mu$ and
\begin{equation}\label{eq:m-mu-gated}
        m_{\mu}\coloneqq\int_{\Z}F(\theta)\dd\mu(\theta,p)=\int_{\Theta}F(\theta)\dd\rho(\theta).
\end{equation}
Define
\begin{equation}\label{eq:a-mu-gated}
        a_{\mu}(\theta)\coloneqq DF(\theta)^{*}\nabla\Rfunc(m_{\mu})
\end{equation}
and the phase-space drift
\begin{equation}\label{eq:drift-mu-gated}
        b_{\mu}(\theta,p)
        \coloneqq\left(-\Ortheps^{\Theta}(p),\gamma(a_{\mu}(\theta)-p)\right).
\end{equation}

\begin{theorem}[Well-posed phase-space mean-field equation]\label{thm:mean-field-gated}
Assume Assumption \ref{ass:global-hilbert}.  For every $\mu_0\in\cP_1(\Z)$ there exists a unique curve
\[
        \mu\in C([0,\infty);\cP_1(\Z))
\]
such that, for every $\zeta\in C_c^{\infty}(\Z)$,
\begin{equation}\label{eq:weak-mean-field-gated}
        \frac{\dd}{\dd t}\int_{\Z}\zeta(\theta,p)\dd\mu_t(\theta,p)
        =\int_{\Z}\left[\ip{\nabla_{\theta}\zeta}{-\Ortheps^{\Theta}(p)}_{\Theta}
        +\ip{\nabla_p\zeta}{\gamma(a_{\mu_t}(\theta)-p)}_{\Theta}\right]\dd\mu_t.
\end{equation}
Equivalently,
\begin{equation}\label{eq:pde-mean-field-gated}
        \partial_t\mu_t+
        \nabla_{\theta}\cdot\left(-\Ortheps^{\Theta}(p)\mu_t\right)
        +\nabla_p\cdot\left(\gamma(a_{\mu_t}(\theta)-p)\mu_t\right)=0.
\end{equation}
The solution is the law of the nonlinear characteristic system
\begin{equation}\label{eq:nonlinear-characteristics-gated}
\begin{aligned}
        \dot\Theta_t&=-\Ortheps^{\Theta}(P_t),\\
        \dot P_t&=\gamma\left(DF(\Theta_t)^{*}\nabla\Rfunc\left(\E F(\Theta_t)\right)-P_t\right),
\end{aligned}
\end{equation}
with $(\Theta_0,P_0)\sim\mu_0$.
\end{theorem}

\begin{proof}
The drift is globally Lipschitz in $(\theta,p)$ and Lipschitz in the measure argument with respect to $W_1$.  The Lipschitz estimate in $(\theta,p)$ follows from \eqref{eq:product-Orth-Lip-gated} and Lemma \ref{lem:force-estimates-gated}.  For the measure argument,
\[
        \norm{a_{\mu}(\theta)-a_{\nu}(\theta)}_{\Theta}
        \le M_D L_R L_F W_1((\pi_{\theta})_{\#}\mu,(\pi_{\theta})_{\#}\nu)
        \le M_D L_R L_F W_1(\mu,\nu).
\]
The drift has at most linear growth in $p$ and bounded growth in the $\theta$ component.  The standard Picard iteration for McKean-Vlasov ODEs with Lipschitz drift gives a unique nonlinear process solving \eqref{eq:nonlinear-characteristics-gated}.  Setting $\mu_t=\Law(\Theta_t,P_t)$ and applying the chain rule to $\zeta(\Theta_t,P_t)$ gives \eqref{eq:weak-mean-field-gated}.  Conversely, the superposition principle for Lipschitz continuity equations implies that any weak solution is transported by the same characteristic flow.  The fixed-point uniqueness for the nonlinear characteristic equation therefore gives uniqueness of the PDE solution.
\end{proof}

\begin{remark}[Localized compact-support form]\label{rem:localized-mf-gated}
On a finite horizon $[0,T]$, Assumption \ref{ass:global-hilbert} can be replaced by Assumption \ref{ass:localized-hilbert} when $\mu_0$ has compact support.  The bound $\norm{\dot\Theta_t}_{\Theta}\le\sqrt{\qtot}$ confines the position support to a bounded $\Theta$-ball, and the momentum comparison estimate confines the momentum support to a bounded ball depending on $T$, the initial support, and the localized force bound.  On this bounded region the drift is Lipschitz, and the proof of Theorem \ref{thm:mean-field-gated} applies without modification.
\end{remark}

\subsection{Hamiltonian formulation and dissipation identity}\label{sec:hamiltonian-gated}

Define the regularized Muon Hamiltonian on phase-space probability measures by
\begin{equation}\label{eq:hamiltonian-gated}
        \Ham_{\eps,\gamma}(\mu)
        \coloneqq\int_{\Z}\Psieps^{\Theta}(p)\dd\mu(\theta,p)
        +\gamma\Rfunc\left(\int_{\Z}F(\theta)\dd\mu(\theta,p)\right).
\end{equation}

\begin{proposition}[First variation of the Hamiltonian]\label{prop:ham-first-var-gated}
A valid first variation of $\Ham_{\eps,\gamma}$ is
\begin{equation}\label{eq:ham-first-var-gated}
        \frac{\delta\Ham_{\eps,\gamma}}{\delta\mu}(\mu)(\theta,p)
        =\Psieps^{\Theta}(p)+\gamma\ip{\nabla\Rfunc(m_{\mu})}{F(\theta)}_{\cH},
\end{equation}
up to an additive constant independent of $(\theta,p)$.  Consequently,
\begin{equation}\label{eq:ham-gradients-gated}
        \nabla_p\frac{\delta\Ham_{\eps,\gamma}}{\delta\mu}=\Ortheps^{\Theta}(p),
        \qquad
        \nabla_{\theta}\frac{\delta\Ham_{\eps,\gamma}}{\delta\mu}=\gamma a_{\mu}(\theta).
\end{equation}
\end{proposition}

\begin{proof}
Let $\nu$ be a finite signed measure on $\Z$ with $\nu(\Z)=0$ and set $\mu_s=\mu+s\nu$.  Then
\begin{align*}
        \left.\frac{\dd}{\dd s}\Ham_{\eps,\gamma}(\mu_s)\right|_{s=0}
        &=\int_{\Z}\Psieps^{\Theta}(p)\dd\nu(\theta,p)
        +\gamma\ip{\nabla\Rfunc(m_{\mu})}{\int_{\Z}F(\theta)\dd\nu(\theta,p)}_{\cH}\\
        &=\int_{\Z}\left[\Psieps^{\Theta}(p)+\gamma\ip{\nabla\Rfunc(m_{\mu})}{F(\theta)}_{\cH}\right]\dd\nu(\theta,p).
\end{align*}
This proves \eqref{eq:ham-first-var-gated}.  Differentiating with respect to $p$ gives $\nabla\Psieps^{\Theta}(p)=\Ortheps^{\Theta}(p)$.  Differentiating with respect to $\theta$ and using Proposition \ref{prop:first-var-gated} gives the second identity in \eqref{eq:ham-gradients-gated}.
\end{proof}

\begin{theorem}[Damped Hamiltonian form]\label{thm:ham-form-gated}
The phase-space PDE \eqref{eq:pde-mean-field-gated} is equivalent to
\begin{equation}\label{eq:damped-ham-form-gated}
        \partial_t\mu_t
        +\nabla_{\theta}\cdot\left(\mu_t\left[-\nabla_p\frac{\delta\Ham_{\eps,\gamma}}{\delta\mu_t}\right]\right)
        +\nabla_p\cdot\left(\mu_t\left[\nabla_{\theta}\frac{\delta\Ham_{\eps,\gamma}}{\delta\mu_t}-\gamma p\right]\right)=0.
\end{equation}
\end{theorem}

\begin{proof}
Substitution of \eqref{eq:ham-gradients-gated} into \eqref{eq:damped-ham-form-gated} gives $\theta$-velocity $-\Ortheps^{\Theta}(p)$ and $p$-velocity $\gamma a_{\mu_t}(\theta)-\gamma p$, which is exactly \eqref{eq:pde-mean-field-gated}.
\end{proof}

\begin{theorem}[Dissipation identity on the extended gated space]\label{thm:dissipation-gated}
Let $\mu_t$ solve \eqref{eq:pde-mean-field-gated}.  Assume that the first moment is finite on compact time intervals and that the test-function cutoff argument is justified by either Assumption \ref{ass:global-hilbert} or the localized compact-support conditions of Remark \ref{rem:localized-mf-gated}.  Then $t\mapsto\Ham_{\eps,\gamma}(\mu_t)$ is absolutely continuous on compact intervals and, for almost every $t$,
\begin{equation}\label{eq:dissipation-gated}
        \frac{\dd}{\dd t}\Ham_{\eps,\gamma}(\mu_t)
        =-\gamma\int_{\Z}\ip{p}{\Ortheps^{\Theta}(p)}_{\Theta}\dd\mu_t(\theta,p)\le0.
\end{equation}
Moreover,
\begin{equation}\label{eq:dissipation-integrand-gated}
        d_{\eps}^{\Theta}(p)
        \coloneqq\ip{p}{\Ortheps^{\Theta}(p)}_{\Theta}
        =\sum_{b=1}^{B}\sum_{r=1}^{q_b}
        \frac{\sigma_r(p^{(b)})^2}{\sqrt{\sigma_r(p^{(b)})^2+\eps^2}},
\end{equation}
so $d_{\eps}^{\Theta}(p)\ge0$, and $d_{\eps}^{\Theta}(p)=0$ if and only if $p=0$.
\end{theorem}

\begin{proof}
Let $G(p)=\Ortheps^{\Theta}(p)$ and $a_t(\theta)=a_{\mu_t}(\theta)$.  The PDE is the continuity equation with velocity
\[
        b_t(\theta,p)=(-G(p),\gamma(a_t(\theta)-p)).
\]
For smooth compactly supported $\zeta$,
\[
        \frac{\dd}{\dd t}\int\zeta\dd\mu_t
        =\int\left[\ip{\nabla_{\theta}\zeta}{-G(p)}_{\Theta}
        +\ip{\nabla_p\zeta}{\gamma(a_t(\theta)-p)}_{\Theta}\right]\dd\mu_t.
\]
The functions $\theta\mapsto F(\theta)$ and $p\mapsto\Psieps^{\Theta}(p)$ are admissible by the stated cutoff hypothesis.  For the feature moment,
\begin{equation}\label{eq:m-dot-gated}
        \dot m_t
        =\frac{\dd}{\dd t}\int F(\theta)\dd\mu_t(\theta,p)
        =\int DF(\theta)[-G(p)]\dd\mu_t(\theta,p).
\end{equation}
Therefore
\begin{align}\label{eq:R-dot-gated}
        \frac{\dd}{\dd t}\left[\gamma\Rfunc(m_t)\right]
        &=\gamma\ip{\nabla\Rfunc(m_t)}{\dot m_t}_{\cH}\nonumber\\
        &=-\gamma\int\ip{DF(\theta)^{*}\nabla\Rfunc(m_t)}{G(p)}_{\Theta}\dd\mu_t\nonumber\\
        &=-\gamma\int\ip{a_t(\theta)}{G(p)}_{\Theta}\dd\mu_t.
\end{align}
For the kinetic term,
\begin{align}\label{eq:K-dot-gated}
        \frac{\dd}{\dd t}\int\Psieps^{\Theta}(p)\dd\mu_t
        &=\int\ip{G(p)}{\gamma(a_t(\theta)-p)}_{\Theta}\dd\mu_t\nonumber\\
        &=\gamma\int\ip{a_t(\theta)}{G(p)}_{\Theta}\dd\mu_t
        -\gamma\int\ip{p}{G(p)}_{\Theta}\dd\mu_t.
\end{align}
Adding \eqref{eq:R-dot-gated} and \eqref{eq:K-dot-gated} cancels the mixed term and gives \eqref{eq:dissipation-gated}.  Formula \eqref{eq:dissipation-integrand-gated} follows by inserting the blockwise SVDs into \eqref{eq:product-Orth-gated}.  Each summand is nonnegative and vanishes exactly when the corresponding singular value is zero.  Hence the sum vanishes if and only if all blocks of $p$ vanish.
\end{proof}

\subsection{Continuous-time convergence under functional assumptions}\label{sec:convergence-gated}

The accelerated Muon flow is not an ordinary Wasserstein gradient flow for $\tf(\rho_t)$, and $\tf(\rho_t)$ need not be monotone.  The dissipated Lyapunov quantity is the damped Hamiltonian.  Exponential convergence of the objective gap follows from the standard force-norm assumptions, kinetic coercivity on bounded momentum sets, and a curvature estimate for the force-momentum alignment.

Throughout this section define
\begin{align}\label{eq:convergence-quantities-gated}
        K_t&\coloneqq\int_{\Z}\Psieps^{\Theta}(p)\dd\mu_t(\theta,p),
        &D_t&\coloneqq\int_{\Z}\ip{p}{\Ortheps^{\Theta}(p)}_{\Theta}\dd\mu_t(\theta,p),\nonumber\\
        J_t&\coloneqq\tf(\rho_t)=\Rfunc(m_t),
        &U_t&\coloneqq J_t-J_{\star},\qquad J_{\star}\coloneqq\inf_{\rho\in\cP_1(\Theta)}\tf(\rho),\nonumber\\
        H_t&\coloneqq K_t+\gamma U_t,
        &A_t&\coloneqq\int_{\Z}\norm{a_t(\theta)}_{\Theta}^{2}\dd\mu_t(\theta,p),\nonumber\\
        C_t&\coloneqq\int_{\Z}\ip{a_t(\theta)}{p}_{\Theta}\dd\mu_t(\theta,p),
        && a_t(\theta)=DF(\theta)^{*}\nabla\Rfunc(m_t).
\end{align}

\subsubsection{Kinetic coercivity on bounded momentum sets}\label{subsec:kinetic-gated}

\begin{lemma}[Kinetic constants for extended product Muon]\label{lem:kinetic-gated}
Assume $\norm{p}_{\Theta}\le B_P$ on the support of $\mu_t$ for all times under consideration.  Then
\begin{equation}\label{eq:kinetic-constants-gated}
        \kappa_K\coloneqq\frac{\eps^2}{(B_P^2+\eps^2)^{3/2}},
        \qquad
        \kappa_D\coloneqq\frac1{\sqrt{B_P^2+\eps^2}},
        \qquad
        L_G\coloneqq\frac1\eps,
        \qquad
        \chi\coloneqq1
\end{equation}
satisfy, for all such $p$,
\begin{align}
        \Psieps^{\Theta}(p)&\ge \frac{\kappa_K}{2}\norm{p}_{\Theta}^{2},\label{eq:kinetic-K-gated}\\
        \ip{p}{\Ortheps^{\Theta}(p)}_{\Theta}&\ge \kappa_D\norm{p}_{\Theta}^{2},\label{eq:kinetic-D-gated}\\
        \Psieps^{\Theta}(p)&\le \chi\ip{p}{\Ortheps^{\Theta}(p)}_{\Theta},\label{eq:kinetic-chi-gated}\\
        \norm{\Ortheps^{\Theta}(p)}_{\Theta}&\le L_G\norm{p}_{\Theta}.\label{eq:kinetic-LG-gated}
\end{align}
\end{lemma}

\begin{proof}
For $s\in[0,B_P]$,
\[
        \sqrt{s^2+\eps^2}-\eps=\int_0^s\frac{r}{\sqrt{r^2+\eps^2}}\dd r.
\]
The derivative of $r\mapsto r/\sqrt{r^2+\eps^2}$ is $\eps^2/(r^2+\eps^2)^{3/2}$, which is at least $\kappa_K$ on $[0,B_P]$.  Since the derivative vanishes at $0$, integration gives
\[
        \sqrt{s^2+\eps^2}-\eps\ge\frac{\kappa_K}{2}s^2.
\]
Summing this over all singular values in all blocks gives \eqref{eq:kinetic-K-gated}.  Similarly,
\[
        \frac{s^2}{\sqrt{s^2+\eps^2}}\ge\frac{s^2}{\sqrt{B_P^2+\eps^2}}=\kappa_Ds^2,
\]
which gives \eqref{eq:kinetic-D-gated}.  The inequality
\[
        \sqrt{s^2+\eps^2}-\eps\le \frac{s^2}{\sqrt{s^2+\eps^2}}
\]
is equivalent to $\eps\le\sqrt{s^2+\eps^2}$, and summing gives \eqref{eq:kinetic-chi-gated}.  Finally, $\Ortheps^{\Theta}(0)=0$ and \eqref{eq:product-Orth-Lip-gated} gives \eqref{eq:kinetic-LG-gated}.
\end{proof}

\subsubsection{Curvature identity for the alignment term}\label{subsec:curvature-gated}

\begin{asu}[Second-order trajectory regularity]\label{ass:second-order-gated}
Along the trajectory there are constants $M_D,M_{D,2},M_R,M_{R,2}<\infty$ such that
\begin{equation}\label{eq:second-order-bounds-gated}
        \norm{DF(\theta)}_{\Theta\to\cH}\le M_D,
        \qquad
        \norm{D^2F(\theta)}_{\Theta\times\Theta\to\cH}\le M_{D,2},
\end{equation}
for $\mu_t$-almost every $\theta$, and
\begin{equation}\label{eq:R-second-order-bounds-gated}
        \norm{\nabla\Rfunc(m_t)}_{\cH}\le M_R,
        \qquad
        \norm{D^2\Rfunc(m_t)}_{\cH\to\cH}\le M_{R,2}.
\end{equation}
\end{asu}

\begin{lemma}[Curvature remainder on the extended space]\label{lem:curvature-gated}
Assume the hypotheses of Lemma \ref{lem:kinetic-gated} and Assumption \ref{ass:second-order-gated}.  Then $C_t$ is absolutely continuous and
\begin{equation}\label{eq:C-prime-gated}
        C_t'=\gamma A_t-\gamma C_t+S_t,
\end{equation}
where
\begin{equation}\label{eq:S-t-gated}
        S_t=\int_{\Z}\ip{\partial_t a_t(\theta)-D_{\theta}a_t(\theta)[\Ortheps^{\Theta}(p)]}{p}_{\Theta}\dd\mu_t(\theta,p),
\end{equation}
and
\begin{equation}\label{eq:S-bound-gated}
        |S_t|\le \sigma D_t,
        \qquad
        \sigma\coloneqq\frac{L_G}{\kappa_D}\left(M_D^2M_{R,2}+M_RM_{D,2}\right).
\end{equation}
\end{lemma}

\begin{proof}
Use the characteristic representation.  Along a characteristic,
\[
        \dot\theta_t=-G(p_t),
        \qquad
        \dot p_t=\gamma(a_t(\theta_t)-p_t),
        \qquad
        G=\Ortheps^{\Theta}.
\]
The chain rule gives
\begin{align*}
        \frac{\dd}{\dd t}\ip{a_t(\theta_t)}{p_t}_{\Theta}
        &=\ip{\partial_t a_t(\theta_t)+D_{\theta}a_t(\theta_t)[\dot\theta_t]}{p_t}_{\Theta}
          +\ip{a_t(\theta_t)}{\dot p_t}_{\Theta}\\
        &=\ip{\partial_t a_t(\theta_t)-D_{\theta}a_t(\theta_t)[G(p_t)]}{p_t}_{\Theta}
          +\gamma\norm{a_t(\theta_t)}_{\Theta}^{2}
          -\gamma\ip{a_t(\theta_t)}{p_t}_{\Theta}.
\end{align*}
Integration with respect to the law of the characteristic gives \eqref{eq:C-prime-gated} and \eqref{eq:S-t-gated}.

It remains to bound $S_t$.  Write $u_t=\nabla\Rfunc(m_t)$, so $a_t(\theta)=DF(\theta)^{*}u_t$.  For $v\in\Theta$,
\[
        D_{\theta}a_t(\theta)[v]=D^2F(\theta)[v,\cdot]^{*}u_t,
\]
and hence
\begin{equation}\label{eq:Dtheta-a-bound-gated}
        \norm{D_{\theta}a_t(\theta)[G(p)]}_{\Theta}
        \le M_{D,2}M_R\norm{G(p)}_{\Theta}.
\end{equation}
Also,
\[
        \partial_t a_t(\theta)=DF(\theta)^{*}D^2\Rfunc(m_t)[\dot m_t].
\]
By \eqref{eq:m-dot-gated},
\[
        \dot m_t=-\int_{\Z}DF(\theta)[G(p)]\dd\mu_t(\theta,p),
\]
so
\begin{equation}\label{eq:m-dot-bound-gated}
        \norm{\dot m_t}_{\cH}\le M_D\int\norm{G(p)}_{\Theta}\dd\mu_t.
\end{equation}
Consequently,
\begin{equation}\label{eq:partial-a-bound-gated}
        \norm{\partial_t a_t(\theta)}_{\Theta}
        \le M_D M_{R,2}\norm{\dot m_t}_{\cH}
        \le M_D^2M_{R,2}\int\norm{G(p)}_{\Theta}\dd\mu_t.
\end{equation}
Using \eqref{eq:kinetic-D-gated}, \eqref{eq:kinetic-LG-gated}, and Cauchy's inequality,
\begin{align*}
        \int\norm{G(p)}_{\Theta}\dd\mu_t
        &\le L_G\int\norm{p}_{\Theta}\dd\mu_t
        \le \frac{L_G}{\sqrt{\kappa_D}}D_t^{1/2},\\
        \int\norm{p}_{\Theta}\dd\mu_t
        &\le \frac1{\sqrt{\kappa_D}}D_t^{1/2},\\
        \int\norm{G(p)}_{\Theta}\norm{p}_{\Theta}\dd\mu_t
        &\le L_G\int\norm{p}_{\Theta}^{2}\dd\mu_t
        \le \frac{L_G}{\kappa_D}D_t.
\end{align*}
Combining \eqref{eq:Dtheta-a-bound-gated} and \eqref{eq:partial-a-bound-gated} gives
\begin{align*}
        |S_t|
        &\le M_D^2M_{R,2}\left(\int\norm{G(p)}_{\Theta}\dd\mu_t\right)
        \left(\int\norm{p}_{\Theta}\dd\mu_t\right)
        +M_RM_{D,2}\int\norm{G(p)}_{\Theta}\norm{p}_{\Theta}\dd\mu_t\\
        &\le \frac{L_G}{\kappa_D}\left(M_D^2M_{R,2}+M_RM_{D,2}\right)D_t.
\end{align*}
This is \eqref{eq:S-bound-gated}.
\end{proof}

\subsubsection{Exponential convergence theorem}\label{subsec:exp-convergence-gated}

\begin{asu}[Functional PL and upper-gradient conditions]\label{ass:PL-gated}
Along the trajectory there exist constants $\lambda>0$ and $\Lambda<\infty$ such that
\begin{equation}\label{eq:PL-gated}
        A_t\ge2\lambda U_t,
\end{equation}
and
\begin{equation}\label{eq:upper-gradient-gated}
        A_t\le2\Lambda U_t.
\end{equation}
\end{asu}

\begin{theorem}[Exponential convergence under general Hilbert domain]\label{thm:convergence-gated}
Assume that the flow \eqref{eq:pde-mean-field-gated} is defined for all $t\ge0$ and that $U_t\ge0$.  Assume that the kinetic estimates \eqref{eq:kinetic-K-gated}-\eqref{eq:kinetic-LG-gated}, the curvature identity
\begin{equation}\label{eq:curvature-assumed-gated}
        C_t'=\gamma A_t-\gamma C_t+S_t,
        \qquad |S_t|\le \sigma D_t,
\end{equation}
and the PL/upper-gradient Assumption \ref{ass:PL-gated} hold for all $t\ge0$.  Lemma \ref{lem:kinetic-gated} and Lemma \ref{lem:curvature-gated} give sufficient conditions for the kinetic estimates and the curvature identity, respectively.  Define
\begin{equation}\label{eq:MC-gated}
        M_C\coloneqq \sqrt{\frac{\Lambda}{\gamma\kappa_K}}.
\end{equation}
Choose $r\in(0,2)$ and $\alpha>0$ such that
\begin{equation}\label{eq:alpha-conditions-gated}
        \alpha M_C<1,
        \qquad
        d_{\alpha,r}\coloneqq\gamma-\alpha\sigma-\frac{\alpha\gamma}{2r\kappa_D}>0.
\end{equation}
Let
\begin{equation}\label{eq:c-alpha-r-gated}
        c_{\alpha,r}
        \coloneqq\frac1{1+\alpha M_C}
        \min\left\{\frac{d_{\alpha,r}}{\chi},\,2\lambda\alpha\left(1-\frac r2\right)\right\}.
\end{equation}
Then
\begin{equation}\label{eq:convergence-estimate-gated}
        \tf(\rho_t)-J_{\star}
        \le \frac{\exp(-c_{\alpha,r}t)}{\gamma(1-\alpha M_C)}\left[H_0-\alpha C_0\right].
\end{equation}
\end{theorem}

\begin{proof}
First, bound the alignment term.  By Cauchy's inequality, \eqref{eq:kinetic-K-gated}, and \eqref{eq:upper-gradient-gated},
\begin{align*}
        |C_t|
        &\le \left(\int\norm{a_t(\theta)}_{\Theta}^{2}\dd\mu_t\right)^{1/2}
              \left(\int\norm{p}_{\Theta}^{2}\dd\mu_t\right)^{1/2}\\
        &\le (2\Lambda U_t)^{1/2}\left(\frac{2K_t}{\kappa_K}\right)^{1/2}
        =2\sqrt{\frac{\Lambda}{\kappa_K}}\sqrt{U_tK_t}.
\end{align*}
Since $2\sqrt{U_tK_t}\le (K_t+\gamma U_t)/\sqrt{\gamma}=H_t/\sqrt{\gamma}$,
\begin{equation}\label{eq:C-H-bound-gated}
        |C_t|\le M_C H_t.
\end{equation}
Define
\[
        L_t\coloneqq H_t-\alpha C_t.
\]
Equation \eqref{eq:C-H-bound-gated} implies
\begin{equation}\label{eq:L-H-equivalence-gated}
        (1-\alpha M_C)H_t\le L_t\le(1+\alpha M_C)H_t.
\end{equation}
By Theorem \ref{thm:dissipation-gated}, $H_t'=K_t'+\gamma U_t'=-\gamma D_t$.  By the assumed curvature identity \eqref{eq:curvature-assumed-gated},
\begin{align*}
        L_t'
        &=H_t'-\alpha C_t'
        =-\gamma D_t-\alpha(\gamma A_t-\gamma C_t+S_t)\\
        &=-\gamma D_t-\alpha\gamma A_t+\alpha\gamma C_t-\alpha S_t.
\end{align*}
Use $|S_t|\le\sigma D_t$ and Young's inequality:
\begin{align*}
        C_t
        &\le\int\norm{a_t(\theta)}_{\Theta}\norm{p}_{\Theta}\dd\mu_t
        \le \frac r2A_t+\frac1{2r}\int\norm{p}_{\Theta}^{2}\dd\mu_t
        \le \frac r2A_t+\frac1{2r\kappa_D}D_t.
\end{align*}
Therefore
\begin{equation}\label{eq:L-prime-bound-gated}
        L_t'
        \le -\left(\gamma-\alpha\sigma-\frac{\alpha\gamma}{2r\kappa_D}\right)D_t
             -\alpha\gamma\left(1-\frac r2\right)A_t
        =-d_{\alpha,r}D_t-a_{\alpha,r}A_t,
\end{equation}
where $a_{\alpha,r}=\alpha\gamma(1-r/2)>0$.  By \eqref{eq:kinetic-chi-gated}, $D_t\ge K_t/\chi$.  By \eqref{eq:PL-gated}, $A_t\ge2\lambda U_t$.  Hence
\begin{align*}
        d_{\alpha,r}D_t+a_{\alpha,r}A_t
        &\ge \frac{d_{\alpha,r}}{\chi}K_t+2\lambda a_{\alpha,r}U_t\\
        &\ge q_{\alpha,r}(K_t+\gamma U_t)=q_{\alpha,r}H_t,
\end{align*}
where
\[
        q_{\alpha,r}
        =\min\left\{\frac{d_{\alpha,r}}{\chi},\frac{2\lambda a_{\alpha,r}}{\gamma}\right\}
        =\min\left\{\frac{d_{\alpha,r}}{\chi},2\lambda\alpha\left(1-\frac r2\right)\right\}.
\]
Combining this with \eqref{eq:L-prime-bound-gated} gives $L_t'\le -q_{\alpha,r}H_t$.  Since $L_t\le(1+\alpha M_C)H_t$, one has $H_t\ge L_t/(1+\alpha M_C)$, and therefore
\[
        L_t'\le -\frac{q_{\alpha,r}}{1+\alpha M_C}L_t=-c_{\alpha,r}L_t.
\]
Gronwall's inequality gives $L_t\le \exp(-c_{\alpha,r}t)L_0$.  Finally, $\gamma U_t\le H_t\le L_t/(1-\alpha M_C)$, and \eqref{eq:convergence-estimate-gated} follows.
\end{proof}

\begin{remark}[Meaning of the PL condition]\label{rem:PL-meaning-gated}
The PL condition \eqref{eq:PL-gated} is a property of the probability functional $\tf$ along the probability-flow trajectory.  It is not a consequence of cross-entropy smoothness alone and is not automatic for arbitrary transformer mixture-of-experts parameterizations.  In the transformer specialization below, smoothness and bounded-gradient properties of the loss and feature map verify the analytic assumptions needed for well-posedness, ODE limits, Hamiltonian dissipation, and curvature control.  Exponential convergence additionally requires \eqref{eq:PL-gated} and \eqref{eq:upper-gradient-gated}, as in the scalar theory.
\end{remark}

\begin{corollary}[Criticality of compact omega-limit points]\label{cor:criticality-gated}
Assume the trajectory $\{\mu_t:t\ge0\}$ is relatively compact in $\cP_1(\Z)$ and $\Ham_{\eps,\gamma}$ is bounded below.  Then
\[
        \int_0^{\infty}D_t\dd t<\infty.
\]
Every invariant omega-limit point is supported on $\{p=0\}$ and satisfies $a_{\mu}(\theta)=0$ for $\mu$-almost every $(\theta,p)$.
\end{corollary}

\begin{proof}
The dissipation identity gives
\[
        \gamma\int_0^T D_t\dd t
        =\Ham_{\eps,\gamma}(\mu_0)-\Ham_{\eps,\gamma}(\mu_T)
        \le \Ham_{\eps,\gamma}(\mu_0)-\inf\Ham_{\eps,\gamma}.
\]
Letting $T\to\infty$ proves integrability.  Since $D_t=0$ if and only if $p=0$ by Theorem \ref{thm:dissipation-gated}, any invariant limit must be supported on $p=0$.  On this support, the $p$-velocity is $\gamma a_{\mu}(\theta)$.  Invariance forces $a_{\mu}(\theta)=0$ on the support; otherwise the measure immediately leaves $\{p=0\}$.
\end{proof}

\subsection{Discrete-time convergence of the regularized Hamiltonian-Muon map}\label{sec:discrete-time-convergence-gated}

The preceding convergence result, Theorem \ref{thm:convergence-gated}, is a continuous-time statement for the damped Hamiltonian probability flow.  We now prove a genuine fixed-step discrete-time convergence result for the natural semi-implicit Euler discretization of that flow.  The result is stated on the extended product space $\Theta$ because this notation covers the gated transformer setting.  The scalar matrix-valued theory in the main paper is recovered by taking $\Theta=\X=\R^{m\times n}$, $\cH=\R$, and
\[
        a_{\mu}(W)=\Rfunc'(m_{\mu})\nabla F(W),
        \qquad
        \Ortheps^{\Theta}=\Ortheps .
\]

Throughout this section the regularized Muon velocity and its Fenchel dissipation density are denoted explicitly by
\begin{equation}\label{eq:discrete-Geps-def}
        G_{\eps}(p)\coloneqq\Ortheps^{\Theta}(p)=\nabla\Psieps^{\Theta}(p),
        \qquad
        d_{\eps}(p)\coloneqq\ip{p}{G_{\eps}(p)}_{\Theta}.
\end{equation}
Thus every appearance of $G_{\eps}$, $d_{\eps}$, and $\Psieps^{\Theta}$ depends on the fixed regularization parameter $\eps>0$.  No unregularized hard-Muon map is used in the discrete convergence proof.
For a phase-space law $\mu\in\cP(\Theta\times\Theta)$ let
\begin{equation}\label{eq:discrete-force-law-def}
        \rho=(\pi_{\theta})_{\#}\mu,
        \qquad
        m_{\mu}=\int_{\Theta\times\Theta} F(\theta)\dd\mu(\theta,p),
        \qquad
        a_{\mu}(\theta)=DF(\theta)^*\nabla\Rfunc(m_{\mu}).
\end{equation}

\subsubsection{The law-level discrete map and the correct Hamiltonian scaling}\label{subsec:law-level-discrete-map}

Let $\eta_h>0$ be the position step and let $\beta_h\in(0,1)$ be the momentum retention coefficient.  Set
\[
        \delta_h\coloneqq1-\beta_h.
\]
Given $\mu_k$, define the force
\[
        a_k(\theta)\coloneqq a_{\mu_k}(\theta),
\]
and update the momentum first:
\begin{equation}\label{eq:discrete-law-momenta}
        p_k^+(\theta,p)
        \coloneqq \beta_h p+\delta_h a_k(\theta).
\end{equation}
The position is then transported using the regularized Muon direction evaluated at the updated momentum:
\begin{equation}\label{eq:discrete-law-map}
        T_k^h(\theta,p)
        \coloneqq\left(\theta-\eta_h G_{\eps}(p_k^+(\theta,p)),\,p_k^+(\theta,p)\right),
        \qquad
        \mu_{k+1}\coloneqq (T_k^h)_{\#}\mu_k.
\end{equation}
Equivalently, if $(\Theta_k,P_k)\sim\mu_k$, then
\begin{equation}\label{eq:random-variable-discrete-map}
        P_{k+1}=\beta_h P_k+(1-\beta_h)a_{\mu_k}(\Theta_k),
        \qquad
        \Theta_{k+1}=\Theta_k-\eta_hG_{\eps}(P_{k+1}),
        \qquad
        \mu_{k+1}=\Law(\Theta_{k+1},P_{k+1}).
\end{equation}
If $\mu_k=N^{-1}\sum_i\delta_{(\theta_{i,k},P_{i,k})}$ is empirical, then \eqref{eq:random-variable-discrete-map} is exactly the finite-particle regularized Muon scheme \eqref{eq:discrete-gated-muon-reduced}.

The scaling needed to approximate the damped Hamiltonian probability flow with finite damping parameter $\gamma>0$ is
\begin{equation}\label{eq:discrete-valid-scaling-exact}
        \eta_h=h,
        \qquad
        1-\beta_h=\gamma h,
        \qquad
        0<h<\frac1\gamma.
\end{equation}
Under this scaling, \eqref{eq:random-variable-discrete-map} is the semi-implicit Euler discretization of the characteristic equations
\[
        \dot \theta=-G_{\eps}(p),
        \qquad
        \dot p=\gamma(a_{\mu_t}(\theta)-p).
\]
More generally, the same fixed-step convergence mechanism is stable under the second-order-consistent inertial scaling
\begin{equation}\label{eq:discrete-valid-scaling-general}
        \eta_h=h+O(h^2),
        \qquad
        1-\beta_h=\gamma h+O(h^2).
\end{equation}
The exact relation $1-\beta_h=\gamma\eta_h$ is used first because it exposes the Hamiltonian cancellation.  The perturbative case \eqref{eq:discrete-valid-scaling-general} is discussed in Remark \ref{rem:second-order-beta-discrete}.  Keeping $\beta_h$ fixed independently of $h$ is not a finite-damping Hamiltonian discretization: in that case
\[
        \frac{P_{k+1}-P_k}{h}=\frac{1-\beta_h}{h}\left(a_{\mu_k}(\Theta_k)-P_k\right),
\]
so the relaxation rate diverges as $h\downarrow0$.

For the iterates generated by \eqref{eq:discrete-law-map}, define
\begin{equation}\label{eq:discrete-KUDAC}
\begin{aligned}
        K_k&\coloneqq\int \Psieps^{\Theta}(p)\dd\mu_k(\theta,p),
        &D_k&\coloneqq\int d_{\eps}(p)\dd\mu_k(\theta,p),\\
        J_k&\coloneqq \tf(\rho_k),
        &U_k&\coloneqq J_k-J_{\star},\\
        H_k&\coloneqq K_k+\gamma U_k,
        &A_k&\coloneqq\int\norm{a_k(\theta)}_{\Theta}^2\dd\mu_k(\theta,p),\\
        C_k&\coloneqq\int\ip{a_k(\theta)}{p}_{\Theta}\dd\mu_k(\theta,p),
        &L_k&\coloneqq H_k-\alpha C_k.
\end{aligned}
\end{equation}
Here $J_{\star}=\inf_{\rho\in\cP_1(\Theta)}\tf(\rho)$.  Since $\mu_{k+1}$ is the pushforward of $\mu_k$ by $T_k^h$, the post-momentum dissipation satisfies
\begin{equation}\label{eq:post-momentum-D}
        D_{k+1}
        =\int d_{\eps}\bigl(p_k^+(\theta,p)\bigr)\dd\mu_k(\theta,p).
\end{equation}

\begin{asu}[Uniform assumptions for the discrete trajectory]\label{ass:discrete-trajectory-gated}
The law sequence $(\mu_k)_{k\ge0}$ generated by \eqref{eq:discrete-law-map} satisfies the following properties.
\begin{enumerate}
        \item The objective gap is nonnegative: $U_k\ge0$ for every $k$.
        \item The PL and upper-gradient conditions hold at every discrete iterate:
        \begin{equation}\label{eq:discrete-PL-upper}
                A_k\ge2\lambda U_k,
                \qquad
                A_k\le2\Lambda U_k,
        \end{equation}
        where $\lambda>0$ and $\Lambda<\infty$.
        \item The momenta $p$ under $\mu_k$ and the post-momenta $p_k^+(\theta,p)$ under $\mu_k$ remain in a common bounded momentum ball on which Lemma \ref{lem:kinetic-gated} holds with constants $\kappa_K,\kappa_D,L_G,\chi$.
        \item The one-step segment
        \[
                \theta_{k,s}(\theta,p)
                \coloneqq \theta-s\eta_hG_{\eps}(p_k^+(\theta,p)),
                \qquad 0\le s\le1,
        \]
        remains in a region on which Assumption \ref{ass:second-order-gated} holds with constants $M_D,M_{D,2},M_R,M_{R,2}$.
\end{enumerate}
Set
\begin{equation}\label{eq:Bcurv-sigma-discrete}
        B_{\mathrm{curv}}\coloneqq M_RM_{D,2}+M_D^2M_{R,2},
        \qquad
        \sigma\coloneqq\frac{L_GB_{\mathrm{curv}}}{\kappa_D}.
\end{equation}
\end{asu}

\subsubsection{One-step estimates}\label{subsec:discrete-one-step-estimates}

\begin{lemma}[Second-order upper estimate for $\tf$ along a transport step]\label{lem:discrete-J-descent}
Let $\mu\in\cP(\Theta\times\Theta)$, let $\rho=(\pi_{\theta})_{\#}\mu$, and let $a_{\mu}(\theta)=DF(\theta)^*\nabla\Rfunc(m_{\mu})$.  Let $v\in L^2(\mu;\Theta)$ and define
\[
        \theta_s=\theta+sv(\theta,p),
        \qquad
        \rho_s=(\theta_s)_{\#}\mu,
        \qquad
        m_s=\int F(\theta_s)\dd\mu(\theta,p).
\]
Assume the segment $\{\theta_s:0\le s\le1\}$ lies in a region where the constants in Assumption \ref{ass:second-order-gated} are valid.  Then
\begin{equation}\label{eq:J-discrete-descent-lemma}
        \tf(\rho_1)-\tf(\rho_0)
        \le
        \int\ip{a_{\mu}(\theta)}{v(\theta,p)}_{\Theta}\dd\mu(\theta,p)
        +\frac{B_{\mathrm{curv}}}{2}\int\norm{v(\theta,p)}_{\Theta}^{2}\dd\mu(\theta,p).
\end{equation}
\end{lemma}

\begin{proof}
For $0\le s\le1$, define $u_s\coloneqq\nabla\Rfunc(m_s)$ and
\[
        a_s(\theta_s)\coloneqq DF(\theta_s)^*u_s.
\]
By the chain rule,
\begin{equation}\label{eq:J-path-derivative-discrete}
        \frac{\dd}{\dd s}\tf(\rho_s)
        =\ip{\nabla\Rfunc(m_s)}{\frac{\dd}{\dd s}m_s}_{\cH}
        =\int\ip{DF(\theta_s)^*u_s}{v}_{\Theta}\dd\mu
        =\int\ip{a_s(\theta_s)}{v}_{\Theta}\dd\mu.
\end{equation}
We next compare $a_s(\theta_s)$ with $a_0(\theta)$.  The exact decomposition is
\[
        a_s(\theta_s)-a_0(\theta)
        =\left(DF(\theta_s)^*-DF(\theta)^*\right)u_0
          +DF(\theta_s)^*(u_s-u_0).
\]
Since $\norm{u_0}_{\cH}\le M_R$ and $\norm{D^2F}\le M_{D,2}$ on the segment,
\begin{equation}\label{eq:first-force-variation-bound}
        \norm{\left(DF(\theta_s)^*-DF(\theta)^*\right)u_0}_{\Theta}
        \le M_RM_{D,2}s\norm{v}_{\Theta}.
\end{equation}
Also,
\begin{align*}
        \norm{m_s-m_0}_{\cH}
        &=\norm{\int\left(F(\theta_s)-F(\theta)\right)\dd\mu}_{\cH}
        \le\int M_Ds\norm{v}_{\Theta}\dd\mu
        \le M_Ds\left(\int\norm{v}_{\Theta}^2\dd\mu\right)^{1/2}.
\end{align*}
Hence
\begin{equation}\label{eq:u-variation-bound-discrete}
        \norm{u_s-u_0}_{\cH}
        \le M_{R,2}\norm{m_s-m_0}_{\cH}
        \le M_DM_{R,2}s\left(\int\norm{v}_{\Theta}^2\dd\mu\right)^{1/2}.
\end{equation}
Using $\norm{DF(\theta_s)}\le M_D$ and Cauchy's inequality,
\begin{align*}
        \int\ip{a_s(\theta_s)-a_0(\theta)}{v}_{\Theta}\dd\mu
        &\le M_RM_{D,2}s\int\norm{v}_{\Theta}^2\dd\mu
        +M_D^2M_{R,2}s\left(\int\norm{v}_{\Theta}^2\dd\mu\right)^{1/2}
             \int\norm{v}_{\Theta}\dd\mu\\
        &\le sB_{\mathrm{curv}}\int\norm{v}_{\Theta}^2\dd\mu.
\end{align*}
Integrating \eqref{eq:J-path-derivative-discrete} from $s=0$ to $s=1$ and using $\int_0^1s\dd s=1/2$ yields \eqref{eq:J-discrete-descent-lemma}.
\end{proof}

\begin{lemma}[Discrete Hamiltonian increment]\label{lem:discrete-H-increment}
Assume the exact scaling \eqref{eq:discrete-valid-scaling-exact}, so that $\delta_h=\gamma\eta_h$ and $\beta_h=1-\delta_h$.  Under Assumption \ref{ass:discrete-trajectory-gated},
\begin{equation}\label{eq:H-increment-discrete}
        H_{k+1}-H_k
        \le
        -\gamma\eta_hD_{k+1}
        +\frac{\gamma^2\eta_h^2}{2\beta_h}A_k
        +\frac{\eta_h^2L_G^2}{2\kappa_D}\left(\frac{\gamma^2}{\beta_h}+\gamma B_{\mathrm{curv}}\right)D_{k+1}.
\end{equation}
\end{lemma}

\begin{proof}
All integrals below are with respect to $\mu_k$.  For readability write
\[
        p^+=p_k^+(\theta,p),
        \qquad
        g^+=G_{\eps}(p^+),
        \qquad
        a=a_k(\theta),
        \qquad
        \eta=\eta_h,
        \qquad
        \delta=\delta_h.
\]
Since $\Psieps^{\Theta}$ is convex and differentiable,
\begin{equation}\label{eq:kinetic-convex-step-discrete}
        \Psieps^{\Theta}(p^+)-\Psieps^{\Theta}(p)
        \le \ip{G_{\eps}(p^+)}{p^+-p}_{\Theta}
        =\delta\ip{g^+}{a-p}_{\Theta}.
\end{equation}
Therefore
\begin{equation}\label{eq:K-increment-first-discrete}
        K_{k+1}-K_k
        \le \delta\int\ip{g^+}{a-p}_{\Theta}\dd\mu_k.
\end{equation}
The position displacement is $v=-\eta g^+$.  Applying Lemma \ref{lem:discrete-J-descent} gives
\begin{equation}\label{eq:U-increment-first-discrete}
        U_{k+1}-U_k
        \le -\eta\int\ip{a}{g^+}_{\Theta}\dd\mu_k
        +\frac{B_{\mathrm{curv}}\eta^2}{2}\int\norm{g^+}_{\Theta}^{2}\dd\mu_k.
\end{equation}
Multiplying \eqref{eq:U-increment-first-discrete} by $\gamma$ and adding \eqref{eq:K-increment-first-discrete}, and then using $\delta=\gamma\eta$, gives the cancellation of the mixed force-transport term:
\begin{equation}\label{eq:H-cancellation-discrete}
        H_{k+1}-H_k
        \le -\delta\int\ip{p}{g^+}_{\Theta}\dd\mu_k
        +\frac{\gamma B_{\mathrm{curv}}\eta^2}{2}\int\norm{g^+}_{\Theta}^{2}\dd\mu_k.
\end{equation}
Now $p^+=p+\delta(a-p)$, hence
\[
        p=p^+-\delta(a-p),
        \qquad
        a-p^+=\beta_h(a-p).
\]
Consequently,
\begin{align}\label{eq:p-to-pplus-discrete}
        -\delta\int\ip{p}{g^+}_{\Theta}\dd\mu_k
        &=-\delta\int\ip{p^+}{g^+}_{\Theta}\dd\mu_k
          +\delta^2\int\ip{a-p}{g^+}_{\Theta}\dd\mu_k\notag\\
        &=-\delta D_{k+1}
          +\frac{\delta^2}{\beta_h}\int\ip{a-p^+}{g^+}_{\Theta}\dd\mu_k.
\end{align}
By Young's inequality and the nonnegativity of $\ip{p^+}{g^+}_{\Theta}$,
\begin{align}\label{eq:young-a-gplus-discrete}
        \int\ip{a-p^+}{g^+}_{\Theta}\dd\mu_k
        &\le \int\ip{a}{g^+}_{\Theta}\dd\mu_k
        \le \frac12 A_k+\frac12\int\norm{g^+}_{\Theta}^2\dd\mu_k.
\end{align}
By \eqref{eq:kinetic-D-gated} and \eqref{eq:kinetic-LG-gated},
\begin{equation}\label{eq:Gplus-by-Dplus-discrete}
        \int\norm{g^+}_{\Theta}^2\dd\mu_k
        \le L_G^2\int\norm{p^+}_{\Theta}^2\dd\mu_k
        \le \frac{L_G^2}{\kappa_D}D_{k+1}.
\end{equation}
Combining \eqref{eq:H-cancellation-discrete}-\eqref{eq:Gplus-by-Dplus-discrete} and using $\delta=\gamma\eta$ proves \eqref{eq:H-increment-discrete}.
\end{proof}

\begin{lemma}[Discrete alignment increment]\label{lem:discrete-C-increment}
Under Assumption \ref{ass:discrete-trajectory-gated},
\begin{equation}\label{eq:C-increment-discrete}
        C_{k+1}-C_k
        =\delta_h(A_k-C_k)+R_k,
\end{equation}
where the remainder obeys
\begin{equation}\label{eq:Rk-bound-discrete}
        |R_k|\le \eta_h\sigma D_{k+1},
        \qquad
        \sigma=\frac{L_GB_{\mathrm{curv}}}{\kappa_D}.
\end{equation}
\end{lemma}

\begin{proof}
Again write $p^+=p_k^+(\theta,p)$, $g^+=G_{\eps}(p^+)$, $a=a_k(\theta)$, and $\theta^+=\theta-\eta_hg^+$.  Since $\mu_{k+1}=(T_k^h)_{\#}\mu_k$,
\[
        C_{k+1}=\int\ip{a_{k+1}(\theta^+)}{p^+}_{\Theta}\dd\mu_k.
\]
Therefore
\begin{align}\label{eq:C-increment-expansion-discrete}
        C_{k+1}-C_k
        &=\int\ip{a}{p^+-p}_{\Theta}\dd\mu_k
          +\int\ip{a_{k+1}(\theta^+)-a}{p^+}_{\Theta}\dd\mu_k\notag\\
        &=\delta_h(A_k-C_k)+R_k,
\end{align}
where
\[
        R_k\coloneqq\int\ip{a_{k+1}(\theta^+)-a_k(\theta)}{p^+}_{\Theta}\dd\mu_k.
\]
It remains to bound $R_k$.  Put $u_k=\nabla\Rfunc(m_k)$ and $u_{k+1}=\nabla\Rfunc(m_{k+1})$.  Then
\[
        a_{k+1}(\theta^+)-a_k(\theta)
        =\left(DF(\theta^+)^*-DF(\theta)^*\right)u_k
          +DF(\theta^+)^*(u_{k+1}-u_k).
\]
The first term satisfies
\begin{equation}\label{eq:alignment-first-term-discrete}
        \norm{\left(DF(\theta^+)^*-DF(\theta)^*\right)u_k}_{\Theta}
        \le M_RM_{D,2}\eta_h\norm{g^+}_{\Theta}.
\end{equation}
For the second term,
\begin{equation}\label{eq:alignment-moment-bound-discrete}
        \norm{m_{k+1}-m_k}_{\cH}
        =\norm{\int\left(F(\theta^+)-F(\theta)\right)\dd\mu_k}_{\cH}
        \le M_D\eta_h\int\norm{g^+}_{\Theta}\dd\mu_k,
\end{equation}
so
\begin{equation}\label{eq:alignment-second-term-discrete}
        \norm{DF(\theta^+)^*(u_{k+1}-u_k)}_{\Theta}
        \le M_D^2M_{R,2}\eta_h\int\norm{g^+}_{\Theta}\dd\mu_k.
\end{equation}
By Cauchy's inequality, \eqref{eq:alignment-first-term-discrete}, and \eqref{eq:alignment-second-term-discrete},
\begin{align*}
        |R_k|
        &\le M_RM_{D,2}\eta_h\int\norm{g^+}_{\Theta}\norm{p^+}_{\Theta}\dd\mu_k
        +M_D^2M_{R,2}\eta_h\left(\int\norm{g^+}_{\Theta}\dd\mu_k\right)
            \left(\int\norm{p^+}_{\Theta}\dd\mu_k\right).
\end{align*}
Using \eqref{eq:kinetic-D-gated} and \eqref{eq:kinetic-LG-gated},
\[
        \int\norm{g^+}_{\Theta}\norm{p^+}_{\Theta}\dd\mu_k
        \le L_G\int\norm{p^+}_{\Theta}^2\dd\mu_k
        \le \frac{L_G}{\kappa_D}D_{k+1},
\]
and
\[
        \left(\int\norm{g^+}_{\Theta}\dd\mu_k\right)
        \left(\int\norm{p^+}_{\Theta}\dd\mu_k\right)
        \le \frac{L_G}{\kappa_D}D_{k+1}.
\]
Combining the last three displays gives \eqref{eq:Rk-bound-discrete}.
\end{proof}

\subsubsection{Discrete modified-Lyapunov contraction}\label{subsec:discrete-lyapunov-contraction}

\begin{theorem}[Fixed-step discrete convergence of the regularized Hamiltonian-Muon map]\label{thm:discrete-convergence-gated}
Assume the exact scaling \eqref{eq:discrete-valid-scaling-exact} and Assumption \ref{ass:discrete-trajectory-gated}.  Define
\begin{equation}\label{eq:MC-discrete}
        M_C\coloneqq \sqrt{\frac{\Lambda}{\gamma\kappa_K}}.
\end{equation}
Choose $r\in(0,2)$ and $\alpha>0$ such that $\alpha M_C<1$.  For a fixed step size $h$ set $\eta_h=h$, $\beta_h=1-\gamma h$, and define
\begin{equation}\label{eq:d-a-h-discrete}
\begin{aligned}
        d_h
        &\coloneqq
        \gamma-\alpha\sigma-
        \frac{\alpha\gamma}{2r\beta_h\kappa_D}
        -h\frac{L_G^2}{2\kappa_D}
            \left(\frac{\gamma^2}{\beta_h}+\gamma B_{\mathrm{curv}}\right),\\
        a_h
        &\coloneqq
        \alpha\gamma\left(1-\frac{r}{2\beta_h}\right)
        -\frac{\gamma^2h}{2\beta_h}.
\end{aligned}
\end{equation}
Assume $d_h>0$, $a_h>0$, and define
\begin{equation}\label{eq:qh-ch-discrete}
        q_h\coloneqq
        \min\left\{
        \frac{d_h\beta_h^2\kappa_D}{L_G},\,
        \frac{a_h}{\frac{\gamma}{2\lambda}+\frac{L_G\gamma^2h^2}{\beta_h^2}}
        \right\},
        \qquad
        c_h\coloneqq\frac{q_h}{1+\alpha M_C}.
\end{equation}
If $hc_h\le1$, then for every $k\ge0$,
\begin{equation}\label{eq:discrete-L-contraction}
        L_{k+1}\le (1-hc_h)L_k.
\end{equation}
Consequently,
\begin{equation}\label{eq:discrete-objective-convergence}
        \tf(\rho_k)-J_{\star}
        \le
        \frac{(1-hc_h)^k}{\gamma(1-\alpha M_C)}
        \left[H_0-\alpha C_0\right]
        \le
        \frac{\exp(-c_hkh)}{\gamma(1-\alpha M_C)}
        \left[H_0-\alpha C_0\right].
\end{equation}

Moreover, the above positivity conditions are guaranteed by an explicit small-step bound.  Suppose in addition that
\begin{equation}\label{eq:discrete-small-h-continuous-condition}
        d_0\coloneqq \gamma-\alpha\sigma-\frac{\alpha\gamma}{2r\kappa_D}>0,
        \qquad
        a_0\coloneqq \alpha\gamma\left(1-\frac r2\right)>0.
\end{equation}
Set
\begin{equation}\label{eq:Bd-Ba-discrete}
        B_d\coloneqq
        \frac{\alpha\gamma^2}{r\kappa_D}
        +\frac{L_G^2}{2\kappa_D}\left(2\gamma^2+\gamma B_{\mathrm{curv}}\right),
        \qquad
        B_a\coloneqq \gamma^2(1+\alpha r),
\end{equation}
with the convention that $d_0/(2B_d)=+\infty$ if $B_d=0$ and $a_0/(2B_a)=+\infty$ if $B_a=0$.  Define
\begin{equation}\label{eq:qstar-discrete}
        q_{\star}\coloneqq
        \min\left\{
        \frac{d_0\kappa_D}{8L_G},\,
        \frac{a_0}{2\left(\frac{\gamma}{2\lambda}+4L_G\gamma^2\right)}
        \right\},
        \qquad
        c_{\star}\coloneqq\frac{q_{\star}}{1+\alpha M_C},
\end{equation}
and
\begin{equation}\label{eq:h-star-discrete}
        h_{\star}\coloneqq
        \min\left\{
        1,\,
        \frac1{2\gamma},\,
        \frac{d_0}{2B_d},\,
        \frac{a_0}{2B_a},\,
        \frac1{c_{\star}}
        \right\}.
\end{equation}
Then every $h\in(0,h_{\star}]$ satisfies the hypotheses $0<h<1/\gamma$, $d_h>0$, $a_h>0$, and $hc_{\star}\le1$, and the estimate
\begin{equation}\label{eq:discrete-objective-convergence-explicit}
        \tf(\rho_k)-J_{\star}
        \le
        \frac{(1-hc_{\star})^k}{\gamma(1-\alpha M_C)}
        \left[H_0-\alpha C_0\right]
        \le
        \frac{\exp(-c_{\star}kh)}{\gamma(1-\alpha M_C)}
        \left[H_0-\alpha C_0\right]
\end{equation}
holds for all $k\ge0$.
\end{theorem}

\begin{proof}
The proof has five steps.

\paragraph{Step (i): equivalence between $L_k$ and $H_k$.}
By Cauchy's inequality, the upper-gradient condition in \eqref{eq:discrete-PL-upper}, and the kinetic lower bound \eqref{eq:kinetic-K-gated},
\begin{align}\label{eq:C-by-H-discrete}
        |C_k|
        &\le\left(\int\norm{a_k(\theta)}_{\Theta}^2\dd\mu_k\right)^{1/2}
              \left(\int\norm{p}_{\Theta}^2\dd\mu_k\right)^{1/2}
        \le (2\Lambda U_k)^{1/2}\left(\frac{2K_k}{\kappa_K}\right)^{1/2}
        \le M_CH_k.
\end{align}
The final inequality follows from $2\sqrt{U_kK_k}\le (K_k+\gamma U_k)/\sqrt\gamma=H_k/\sqrt\gamma$.  Thus
\begin{equation}\label{eq:L-H-equivalence-discrete}
        (1-\alpha M_C)H_k\le L_k\le(1+\alpha M_C)H_k.
\end{equation}

\paragraph{Step (ii): one-step decay of $L_k$ in terms of $D_{k+1}$ and $A_k$.}
By Lemmas \ref{lem:discrete-H-increment} and \ref{lem:discrete-C-increment}, and by $\delta_h=\gamma h$,
\begin{align}\label{eq:L-increment-start-discrete}
        L_{k+1}-L_k
        &=H_{k+1}-H_k-\alpha(C_{k+1}-C_k)\notag\\
        &\le H_{k+1}-H_k-\alpha\gamma h A_k+\alpha\gamma h C_k+\alpha h\sigma D_{k+1}.
\end{align}
Since $p^+=\beta_hp+\gamma h a_k(\theta)$,
\[
        p=\frac{p^+-\gamma h a_k(\theta)}{\beta_h}.
\]
Therefore
\begin{align}\label{eq:Ck-young-pplus-discrete}
        C_k
        &=\frac1{\beta_h}\int\ip{a_k(\theta)}{p^+}_{\Theta}\dd\mu_k
          -\frac{\gamma h}{\beta_h}A_k\notag\\
        &\le \frac1{\beta_h}\int\ip{a_k(\theta)}{p^+}_{\Theta}\dd\mu_k
        \le \frac{r}{2\beta_h}A_k+\frac1{2r\beta_h}\int\norm{p^+}_{\Theta}^2\dd\mu_k\notag\\
        &\le \frac{r}{2\beta_h}A_k+\frac1{2r\beta_h\kappa_D}D_{k+1}.
\end{align}
Substituting \eqref{eq:Ck-young-pplus-discrete} and \eqref{eq:H-increment-discrete} into \eqref{eq:L-increment-start-discrete} gives
\begin{equation}\label{eq:L-increment-D-A-discrete}
        L_{k+1}-L_k\le -hd_hD_{k+1}-ha_hA_k,
\end{equation}
with $d_h$ and $a_h$ defined in \eqref{eq:d-a-h-discrete}.

\paragraph{Step (iii): $D_{k+1}$ and $A_k$ dominate $H_k$.}
Because $G_{\eps}(0)=0$ and $G_{\eps}$ is $L_G$-Lipschitz,
\begin{equation}\label{eq:Psi-upper-LG-discrete}
        \Psieps^{\Theta}(p)=\int_0^1\ip{G_{\eps}(sp)}{p}_{\Theta}\dd s
        \le\int_0^1 L_Gs\norm{p}_{\Theta}^2\dd s
        =\frac{L_G}{2}\norm{p}_{\Theta}^2.
\end{equation}
Moreover,
\[
        \norm{p}_{\Theta}^2
        =\frac1{\beta_h^2}\norm{p^+-\gamma h a_k(\theta)}_{\Theta}^2
        \le \frac2{\beta_h^2}\norm{p^+}_{\Theta}^2
             +\frac{2\gamma^2h^2}{\beta_h^2}\norm{a_k(\theta)}_{\Theta}^2.
\]
Using \eqref{eq:kinetic-D-gated} and the PL inequality,
\begin{align}\label{eq:H-by-Dplus-A-discrete}
        H_k
        &=K_k+\gamma U_k\notag\\
        &\le \frac{L_G}{\beta_h^2\kappa_D}D_{k+1}
        +\left(\frac{L_G\gamma^2h^2}{\beta_h^2}+\frac{\gamma}{2\lambda}\right)A_k.
\end{align}
By the definition of $q_h$ in \eqref{eq:qh-ch-discrete}, \eqref{eq:H-by-Dplus-A-discrete} implies
\begin{equation}\label{eq:q-H-bound-discrete}
        d_hD_{k+1}+a_hA_k\ge q_hH_k.
\end{equation}
Combining \eqref{eq:L-increment-D-A-discrete} with \eqref{eq:q-H-bound-discrete} yields
\begin{equation}\label{eq:L-H-decay-discrete}
        L_{k+1}-L_k\le -hq_hH_k.
\end{equation}

\paragraph{Step (iv): contraction and objective decay for a fixed admissible $h$.}
Since $L_k\le(1+\alpha M_C)H_k$, we have $H_k\ge L_k/(1+\alpha M_C)$.  Hence \eqref{eq:L-H-decay-discrete} gives
\[
        L_{k+1}\le\left(1-h\frac{q_h}{1+\alpha M_C}\right)L_k
        =(1-hc_h)L_k.
\]
If $hc_h\le1$, iteration proves \eqref{eq:discrete-L-contraction}.  Finally, $\gamma U_k\le H_k\le L_k/(1-\alpha M_C)$ by \eqref{eq:L-H-equivalence-discrete}, so
\[
        U_k\le\frac{L_k}{\gamma(1-\alpha M_C)}
        \le\frac{(1-hc_h)^kL_0}{\gamma(1-\alpha M_C)}.
\]
Since $L_0=H_0-\alpha C_0$ and $(1-x)^k\le\exp(-kx)$ for $x\in[0,1]$, \eqref{eq:discrete-objective-convergence} follows.

\paragraph{Step (v): explicit sufficient upper bound on $h$.}
Assume \eqref{eq:discrete-small-h-continuous-condition}.  If $h\le1/(2\gamma)$, then $\beta_h=1-\gamma h\ge1/2$ and
\[
        \frac1{\beta_h}-1=\frac{\gamma h}{1-\gamma h}\le2\gamma h.
\]
Using \eqref{eq:d-a-h-discrete},
\begin{align*}
        d_h
        &=d_0-\frac{\alpha\gamma}{2r\kappa_D}\left(\frac1{\beta_h}-1\right)
          -h\frac{L_G^2}{2\kappa_D}\left(\frac{\gamma^2}{\beta_h}+\gamma B_{\mathrm{curv}}\right)\\
        &\ge d_0-h\left[\frac{\alpha\gamma^2}{r\kappa_D}
          +\frac{L_G^2}{2\kappa_D}\left(2\gamma^2+\gamma B_{\mathrm{curv}}\right)\right]
        =d_0-hB_d.
\end{align*}
Hence $h\le d_0/(2B_d)$ implies $d_h\ge d_0/2$.  Similarly,
\begin{align*}
        a_h
        &=a_0-\frac{\alpha\gamma r}{2}\left(\frac1{\beta_h}-1\right)-\frac{\gamma^2h}{2\beta_h}
        \ge a_0-\alpha r\gamma^2h-\gamma^2h
        =a_0-hB_a.
\end{align*}
Thus $h\le a_0/(2B_a)$ implies $a_h\ge a_0/2$.  If also $h\le1$, then
\[
        \frac{L_G\gamma^2h^2}{\beta_h^2}\le4L_G\gamma^2,
        \qquad
        \beta_h^2\ge\frac14.
\]
Consequently, for every $h\le h_{\star}$,
\[
        q_h\ge
        \min\left\{
        \frac{(d_0/2)(1/4)\kappa_D}{L_G},\,
        \frac{a_0/2}{\frac{\gamma}{2\lambda}+4L_G\gamma^2}
        \right\}=q_{\star}.
\]
Using \eqref{eq:L-H-decay-discrete} with $q_{\star}$ in place of $q_h$ gives
\[
        L_{k+1}\le\left(1-h\frac{q_{\star}}{1+\alpha M_C}\right)L_k
        =(1-hc_{\star})L_k.
\]
The final entry $h\le1/c_{\star}$ in \eqref{eq:h-star-discrete} makes this contraction factor nonnegative.  Iterating and using $\gamma U_k\le L_k/(1-\alpha M_C)$ proves \eqref{eq:discrete-objective-convergence-explicit}.
\end{proof}

\begin{remark}[Second-order momentum scaling]\label{rem:second-order-beta-discrete}
The exact identity $1-\beta_h=\gamma\eta_h$ is not essential.  Suppose instead
\[
        \delta_h\coloneqq1-\beta_h=\gamma\eta_h+\zeta_h,
        \qquad
        |\zeta_h|\le C_{\beta}\eta_h^2,
        \qquad
        \eta_h=h+O(h^2).
\]
In the proof of Lemma \ref{lem:discrete-H-increment}, the cancellation in \eqref{eq:H-cancellation-discrete} leaves the additional term
\[
        (\delta_h-\gamma\eta_h)\int\ip{a_k(\theta)}{G_{\eps}(p_k^+(\theta,p))}_{\Theta}\dd\mu_k.
\]
By Young's inequality and \eqref{eq:Gplus-by-Dplus-discrete}, its absolute value is bounded by
\[
        C_{\beta}\eta_h^2\left(\frac12 A_k+\frac{L_G^2}{2\kappa_D}D_{k+1}\right).
\]
All other appearances of $\delta_h$ are also $\gamma\eta_h+O(\eta_h^2)$.  Consequently the coefficients $d_h$ and $a_h$ in \eqref{eq:d-a-h-discrete} are changed only by $O(h)$ terms.  Thus Theorem \ref{thm:discrete-convergence-gated} remains true, with slightly smaller positive constants and a correspondingly smaller explicit threshold, under the standard inertial scaling \eqref{eq:discrete-valid-scaling-general}.  This is the valid scaling for a finite-damping Hamiltonian limit; keeping $\beta_h$ fixed as $h\downarrow0$ instead gives a singular overdamped relaxation of the momentum variable.
\end{remark}

\begin{corollary}[Finite-particle and scalar matrix-space forms]\label{cor:finite-scalar-discrete-convergence}
If $\mu_k=N^{-1}\sum_{i=1}^{N}\delta_{(\theta_{i,k},P_{i,k})}$, then the law update \eqref{eq:discrete-law-map} is equivalent to
\[
        P_{i,k+1}=\beta_hP_{i,k}+(1-\beta_h)a_i^N(\boldsymbol\theta_k),
        \qquad
        \theta_{i,k+1}=\theta_{i,k}-\eta_h\Ortheps^{\Theta}(P_{i,k+1}),
\]
and Theorem \ref{thm:discrete-convergence-gated} gives the same exponential bound for the finite-particle objective $\tf_N(\boldsymbol\theta_k)-J_{\star}$.  In the original scalar matrix-space setting, this reads
\[
        P_{i,k+1}=\beta_hP_{i,k}+(1-\beta_h)\Rfunc'(\bar F_k)\nabla F(W_{i,k}),
        \qquad
        W_{i,k+1}=W_{i,k}-\eta_h\Ortheps(P_{i,k+1}),
\]
with $\eta_h=h$ and $1-\beta_h=\gamma h+O(h^2)$.  Therefore the discrete-time convergence theorem applies directly to the natural valid discretization of the Hamiltonian probability flow \eqref{Mc-Kean Vlasov equation}-\eqref{Hamiltonian equation} and to its extended-space analogue \eqref{eq:pde-mean-field-gated}-\eqref{eq:damped-ham-form-gated}.
\end{corollary}

\subsection{Propagation of chaos}\label{sec:poc-gated}

The finite-$N$ ODE \eqref{eq:finite-ode-gated} is the interacting-particle approximation of the nonlinear characteristic equation \eqref{eq:nonlinear-characteristics-gated}.

\begin{theorem}[Propagation of chaos on the Hilbert space domain]\label{thm:poc-gated}
Assume Assumption \ref{ass:global-hilbert} and $\mu_0\in\cP_2(\Z)$.  Let $(\theta_i^N(t),P_i^N(t))_{i=1}^{N}$ solve the $N$-particle ODE \eqref{eq:finite-ode-gated} with i.i.d. initial data distributed according to $\mu_0$.  Let $(\bar\theta_i(t),\bar P_i(t))_{i\ge1}$ be i.i.d. nonlinear mean-field copies solving \eqref{eq:nonlinear-characteristics-gated} with the same initial data, i.e.
\[
        (\theta_i^N(0),P_i^N(0))=(\bar\theta_i(0),\bar P_i(0)).
\]
Then for every $T<\infty$ there exists $C_{\mathrm{poc}}(T,\eps)<\infty$ such that, for each fixed $i$,
\begin{equation}\label{eq:poc-rate-gated}
        \sup_{0\le t\le T}\E\left[\norm{\theta_i^N(t)-\bar\theta_i(t)}_{\Theta}^{2}
        +\norm{P_i^N(t)-\bar P_i(t)}_{\Theta}^{2}\right]
        \le \frac{C_{\mathrm{poc}}(T,\eps)}{N}.
\end{equation}
Consequently, for every fixed $k\in\N$,
\begin{equation}\label{eq:k-chaos-gated}
        \sup_{0\le t\le T}W_2^2\left(\Law((\theta_1^N(t),P_1^N(t)),\ldots,(\theta_k^N(t),P_k^N(t))),\mu_t^{\otimes k}\right)
        \le \frac{kC_{\mathrm{poc}}(T,\eps)}{N}.
\end{equation}
Furthermore, if
\[
        \mu_t^N=\frac1N\sum_{i=1}^{N}\delta_{(\theta_i^N(t),P_i^N(t))},
        \qquad
        \bar\mu_t^N=\frac1N\sum_{i=1}^{N}\delta_{(\bar\theta_i(t),\bar P_i(t))},
\]
then
\begin{equation}\label{eq:empirical-poc-gated}
        \sup_{0\le t\le T}\E W_2^2(\mu_t^N,\mu_t)
        \le \frac{2C_{\mathrm{poc}}(T,\eps)}{N}
        +2\sup_{0\le t\le T}\E W_2^2(\bar\mu_t^N,\mu_t),
\end{equation}
and the second term tends to zero as $N\to\infty$.
\end{theorem}

\begin{proof}
Let
\[
        e_i^{\theta}(t)=\theta_i^N(t)-\bar\theta_i(t),
        \qquad
        e_i^{P}(t)=P_i^N(t)-\bar P_i(t),
\]
and set
\[
        u(t)=\E\left[\norm{e_i^{\theta}(t)}_{\Theta}^{2}+\norm{e_i^{P}(t)}_{\Theta}^{2}\right],
\]
which is independent of $i$ by exchangeability.  Since $\Ortheps^{\Theta}$ is $1/\eps$-Lipschitz,
\begin{align*}
        \frac{\dd}{\dd t}\norm{e_i^{\theta}}_{\Theta}^{2}
        &=-2\ip{e_i^{\theta}}{\Ortheps^{\Theta}(P_i^N)-\Ortheps^{\Theta}(\bar P_i)}_{\Theta}\\
        &\le \norm{e_i^{\theta}}_{\Theta}^{2}+\eps^{-2}\norm{e_i^P}_{\Theta}^{2}.
\end{align*}
For the momentum error,
\begin{align*}
        \frac{\dd}{\dd t}\norm{e_i^P}_{\Theta}^{2}
        &=2\gamma\ip{e_i^P}{a_i^N(\boldsymbol\theta^N)-a_{\mu_t}(\bar\theta_i)-e_i^P}_{\Theta}\\
        &\le \gamma\norm{a_i^N(\boldsymbol\theta^N)-a_{\mu_t}(\bar\theta_i)}_{\Theta}^{2}-\gamma\norm{e_i^P}_{\Theta}^{2}\nonumber\\
        &\le \gamma\norm{a_i^N(\boldsymbol\theta^N)-a_{\mu_t}(\bar\theta_i)}_{\Theta}^{2}.
\end{align*}
By the force decomposition used in Lemma \ref{lem:force-estimates-gated}, there exists a constant $C_a$ depending only on the global constants such that
\begin{equation}\label{eq:force-poc-decomp-gated}
        \norm{a_i^N(\boldsymbol\theta^N)-a_{\mu_t}(\bar\theta_i)}_{\Theta}^{2}
        \le C_a\norm{e_i^{\theta}}_{\Theta}^{2}+C_a\norm{m_N^N(t)-m_t}_{\cH}^{2},
\end{equation}
where
\[
        m_N^N(t)=\frac1N\sum_{j=1}^{N}F(\theta_j^N(t)),
        \qquad
        m_t=\E F(\bar\theta_i(t)).
\]
Decompose
\begin{align*}
        m_N^N(t)-m_t
        &=\frac1N\sum_{j=1}^{N}\left(F(\theta_j^N(t))-F(\bar\theta_j(t))\right)
          +\left[\frac1N\sum_{j=1}^{N}F(\bar\theta_j(t))-m_t\right].
\end{align*}
By Jensen's inequality and the Lipschitz property of $F$,
\[
        \E\norm{\frac1N\sum_{j=1}^{N}(F(\theta_j^N)-F(\bar\theta_j))}_{\cH}^{2}
        \le \frac1N\sum_{j=1}^{N}L_F^2\E\norm{e_j^{\theta}}_{\Theta}^{2}
        \le L_F^2u(t).
\]
The random variables $F(\bar\theta_j(t))$ are i.i.d.; hence
\[
        \E\norm{\frac1N\sum_{j=1}^{N}F(\bar\theta_j(t))-m_t}_{\cH}^{2}
        =\frac1N\E\norm{F(\bar\theta_i(t))-m_t}_{\cH}^{2}.
\]
On $[0,T]$ this variance is finite because $F$ is Lipschitz and $\mu_0\in\cP_2(\Z)$.  Therefore
\[
        \E\norm{m_N^N(t)-m_t}_{\cH}^{2}\le C_Tu(t)+\frac{C_T}{N}.
\]
Combining the previous differential inequalities gives
\[
        u'(t)\le C_Tu(t)+\frac{C_T}{N},
        \qquad
        u(0)=0.
\]
Gronwall's inequality proves \eqref{eq:poc-rate-gated}.  The $k$-particle estimate \eqref{eq:k-chaos-gated} follows by coupling each interacting particle with its mean-field copy and summing squared errors.  Finally,
\begin{align*}
        W_2^2(\mu_t^N,\mu_t)
        &\le2W_2^2(\mu_t^N,\bar\mu_t^N)+2W_2^2(\bar\mu_t^N,\mu_t)\\
        &\le \frac2N\sum_{i=1}^{N}\left(\norm{e_i^{\theta}(t)}_{\Theta}^{2}+\norm{e_i^P(t)}_{\Theta}^{2}\right)+2W_2^2(\bar\mu_t^N,\mu_t).
\end{align*}
Taking expectations and using \eqref{eq:poc-rate-gated} proves \eqref{eq:empirical-poc-gated}.  The convergence of the i.i.d. empirical term, for every fixed $t$, follows from the law of large numbers in $W_2$ on finite-dimensional spaces with finite second moment. To obtain the supremum over $t \in[0, T]$, use the $L^2$-continuity of the nonlinear characteristics. The drift has at most linear growth and the second moments remain bounded on $[0, T]$, hence
$$
\mathbb{E}\left\|Z_t-Z_s\right\|^2 \leq C_T|t-s|^2 .
$$
The same estimate holds for the empirical nonlinear system $\bar{\mu}_t^N$. For a grid $0=t_0<\cdots<t_M=T$, this gives
$$
\sup _{t \in[0, T]} W_2\left(\bar{\mu}_t^N, \mu_t\right) \leq \max _{\ell} W_2\left(\bar{\mu}_{t_{\ell}}^N, \mu_{t_{\ell}}\right)+C_T \Delta t
$$
in expectation, up to the standard empirical and population continuity terms. First, we let $N \rightarrow \infty$ for fixed grid, then we let $\Delta t \downarrow 0$.
\end{proof}

\subsection{Hard Muon limit on the extended gated space}\label{sec:hard-limit-gated}

The limiting product potential is
\begin{equation}\label{eq:Psi0-product-gated}
        \Psi_0^{\Theta}(p)
        \coloneqq\sum_{b=1}^{B}\norm{p^{(b)}}_{\nuc}.
\end{equation}
For a block $p^{(b)}=U\Sigma V^{\top}$ of rank $r$, the nuclear-norm subdifferential is
\begin{equation}\label{eq:nuclear-subdiff-gated}
        \partial\norm{p^{(b)}}_{\nuc}
        =\left\{UV^{\top}+Z:U^{\top}Z=0,\,ZV=0,\,\norm{Z}_{\op}\le1\right\}.
\end{equation}
Thus
\begin{equation}\label{eq:subdiff-product-gated}
        \partial\Psi_0^{\Theta}(p)
        =\prod_{b=1}^{B}\partial\norm{p^{(b)}}_{\nuc}.
\end{equation}
The canonical hard Muon map is
\begin{equation}\label{eq:hard-product-Orth-gated}
        \Orth^{\Theta}(p)=\left(\Orth(p^{(1)}),\ldots,\Orth(p^{(B)})\right),
\end{equation}
which is a selected element of $\partial\Psi_0^{\Theta}(p)$.

\begin{theorem}[Subsequential hard-Muon limit on the extended gated space]\label{thm:hard-limit-gated}
Let $\mu_0\in\cP_1(\Z)$ and assume that the force field satisfies the global bounds of Assumption \ref{ass:global-hilbert}.  Let $\mu_t^{\eps}$ solve \eqref{eq:pde-mean-field-gated} with $\Ortheps^{\Theta}=\nabla\Psieps^{\Theta}$.  For every $T<\infty$, every sequence $\eps_k\downarrow0$ has a subsequence, not relabeled, such that
\[
        \mu_t^{\eps_k}\to\mu_t
        \quad\text{in }C([0,T];\cP_1(\Z)).
\]
There exists a Borel vector field $V_t(\theta,p)$ with
\begin{equation}\label{eq:hard-limit-subdiff-gated}
        V_t(\theta,p)\in\partial\Psi_0^{\Theta}(p)
        \qquad \mu_t\dd t\text{-a.e.}
\end{equation}
such that $\mu_t$ solves
\begin{equation}\label{eq:hard-limit-pde-gated}
        \partial_t\mu_t+\nabla_{\theta}\cdot(-V_t\mu_t)
        +\nabla_p\cdot\left(\gamma(a_{\mu_t}(\theta)-p)\mu_t\right)=0.
\end{equation}
If every block $p^{(b)}$ has full rank for $\mu_t\dd t$-almost every $(\theta,p)$, then $V_t(\theta,p)=\Orth^{\Theta}(p)$ almost everywhere and the limit is the canonical hard Muon flow.
\end{theorem}

\begin{proof}
The velocity $\Ortheps^{\Theta}(p)$ is uniformly bounded by $\sqrt{\qtot}$ for all $\eps>0$, and the force is bounded by Lemma \ref{lem:force-estimates-gated}.  Along characteristics,
\[
        \norm{\Theta_t^{\eps}-\Theta_s^{\eps}}_{\Theta}\le\sqrt{\qtot}|t-s|,
\]
and
\[
        \frac{\dd}{\dd t}\norm{P_t^{\eps}}_{\Theta}\le\gamma(M_D M_R+\norm{P_t^{\eps}}_{\Theta}).
\]
Gronwall's inequality gives a first-moment bound on $P_t^{\eps}$, uniformly for $t\in[0,T]$ and $\eps$.  These estimates imply tightness and equicontinuity of $\{\mu^{\eps}\}_{\eps>0}$ in $C([0,T];\cP_1(\Z))$, hence subsequential compactness.

Let $G^{\eps}(p)=\Ortheps^{\Theta}(p)$.  The uniform bound $\norm{G^{\eps}(p)}_{\Theta}\le\sqrt{\qtot}$ implies, after extraction, weak-star convergence of the fluxes $G^{\eps_k}(p)\mu_t^{\eps_k}\dd t$ to a vector-valued measure absolutely continuous with respect to $\mu_t\dd t$.  Its density is denoted by $V_t(\theta,p)$ and satisfies $\norm{V_t}\le\sqrt{\qtot}$.

It remains to identify $V_t(\theta,p)$ as a subgradient of $\Psi_0^{\Theta}$.  For every $q\in\Theta$ and every $p\in\Theta$, convexity of $\Psieps^{\Theta}$ gives
\begin{equation}\label{eq:convex-ineq-hard-gated}
        \Psieps^{\Theta}(q)\ge \Psieps^{\Theta}(p)+\ip{\Ortheps^{\Theta}(p)}{q-p}_{\Theta}.
\end{equation}
The functions $\Psieps^{\Theta}$ converge locally uniformly to $\Psi_0^{\Theta}$ as $\eps\downarrow0$.  Passing to the limit in the integrated form of \eqref{eq:convex-ineq-hard-gated}, with $p$ replaced by the momentum coordinate and with arbitrary bounded nonnegative test weights, gives
\[
        \Psi_0^{\Theta}(q)\ge \Psi_0^{\Theta}(p)+\ip{V_t(\theta,p)}{q-p}_{\Theta}
        \qquad \mu_t\dd t\text{-a.e.}
\]
This inequality for every $q\in\Theta$ is exactly $V_t(\theta,p)\in\partial\Psi_0^{\Theta}(p)$.

Passing to the limit in the weak formulation of \eqref{eq:pde-mean-field-gated} gives \eqref{eq:hard-limit-pde-gated}; the $p$-velocity term passes by the force-field continuity established in Lemma \ref{lem:force-estimates-gated} and the convergence of $\mu^{\eps_k}$ to $\mu$.  The final statement follows from \eqref{eq:nuclear-subdiff-gated}: if $p^{(b)}$ has full rank, the orthogonality constraints $U^{\top}Z=0$ and $ZV=0$ force $Z=0$, so the block subdifferential is the singleton $\{\Orth(p^{(b)})\}$.
\end{proof}

\subsection{Transformer MoE specialization with input-dependent routing}\label{sec:transformer-specialization-gated}

The abstract theory specializes to transformer mixture-of-experts models by choosing $\Theta$ to contain both an expert tuple and router parameters.  The input dependence of the gate is encoded in the Hilbert-valued feature map $F$.  Softmax normalization across experts is represented by augmenting the Hilbert output with numerator and denominator features.

\subsubsection{Expert-router parameter space}\label{subsec:transformer-theta-gated}

Let an input sequence be
\begin{equation}\label{eq:X-transformer}
        X\in\R^{L\times d}.
\end{equation}
A single expert is parameterized by
\begin{equation}\label{eq:expert-tuple-transformer}
        \omega=(Q,K,V,O,W_1,W_2),
\end{equation}
where
\begin{equation}\label{eq:expert-shapes-transformer}
        Q,K\in\R^{d\times d_k},\qquad
        V\in\R^{d\times d_v},\qquad
        O\in\R^{d_v\times d},
        \qquad
        W_1\in\R^{d\times d_f},\qquad
        W_2\in\R^{d_f\times d}.
\end{equation}
Thus
\begin{equation}\label{eq:Theta-exp-transformer}
        \Thetaexp
        =\R^{d\times d_k}\times\R^{d\times d_k}\times\R^{d\times d_v}\times\R^{d_v\times d}
        \times\R^{d\times d_f}\times\R^{d_f\times d}.
\end{equation}
Let the router parameter space be a finite product of matrix spaces
\begin{equation}\label{eq:Theta-gate-transformer}
        \Thetagate=\prod_{r=1}^{B_{\Gate}}\R^{a_r\times b_r}.
\end{equation}
A linear per-token router, for example, can be represented by a matrix $G\in\R^{d\times1}$ and score $s_G(X)_t=\ip{X_{t,:}}{G}_{\R^d}$; more general smooth router MLPs correspond to several matrix blocks in \eqref{eq:Theta-gate-transformer}.  The full particle parameter is
\begin{equation}\label{eq:Theta-full-transformer}
        \theta=(\omega,\phi)\in\Theta\coloneqq\Thetaexp\times\Thetagate.
\end{equation}
The Hilbert norm is the product Frobenius norm
\begin{equation}\label{eq:Theta-transformer-norm}
        \norm{\theta}_{\Theta}^{2}
        =\norm{Q}_{F}^{2}+\norm{K}_{F}^{2}+\norm{V}_{F}^{2}+\norm{O}_{F}^{2}
        +\norm{W_1}_{F}^{2}+\norm{W_2}_{F}^{2}+\sum_{r=1}^{B_{\Gate}}\norm{\phi^{(r)}}_{F}^{2}.
\end{equation}
The regularized Muon mirror map is
\begin{equation}\label{eq:Theta-transformer-Orth}
\begin{aligned}
        &\Ortheps^{\Theta}(P_Q,P_K,P_V,P_O,P_{W_1},P_{W_2},P_{\phi^{(1)}},\ldots)\\
        &\quad=\bigl(\Ortheps(P_Q),\Ortheps(P_K),\Ortheps(P_V),\Ortheps(P_O),
        \Ortheps(P_{W_1}),\Ortheps(P_{W_2}),\Ortheps(P_{\phi^{(1)}}),\ldots\bigr).
\end{aligned}
\end{equation}
with each $\Ortheps$ evaluated at the corresponding rectangular block.  Proposition \ref{prop:extended-fenchel-mirror} applies directly to this space.

\subsubsection{Smooth single-head attention-plus-FFN expert and router}\label{subsec:smooth-expert-router}

Let $\sigma:\R\to\R$ be a smooth activation applied rowwise.  Define the single-head attention map
\begin{equation}\label{eq:single-head-attention}
        A_{\omega}(X)
        \coloneqq \softmax\left(\frac{(XQ)(XK)^{\top}}{\sqrt{d_k}}\right)XVO\in\R^{L\times d},
\end{equation}
where softmax is applied rowwise.  Define the feed-forward map
\begin{equation}\label{eq:ffn-map}
        B_{\omega}(X)
        \coloneqq \sigma(A_{\omega}(X)W_1)W_2\in\R^{L\times d}.
\end{equation}
Let $R_{\mathrm{out}}:\R^{L\times d}\to\R^{L\times C}$ be a fixed output projection to token logits.  The expert logit function is
\begin{equation}\label{eq:expert-logit}
        \psi_{\omega}(X)\coloneqq R_{\mathrm{out}}B_{\omega}(X)\in\R^{L\times C}.
\end{equation}
Let
\begin{equation}\label{eq:router-score}
        s_{\phi}(X)\in\R^{L}
\end{equation}
be a smooth tokenwise router-score map.  Sequence-level routing is obtained as the special case in which all coordinates of $s_{\phi}(X)$ are identical or only one scalar score is stored.

\begin{lemma}[Local smoothness of transformer expert-router features]\label{lem:transformer-local-smoothness}
Assume the input set is bounded, $\norm{X}_{F}\le B_X$, and the activation $\sigma$ and router score map $\phi\mapsto s_{\phi}(X)$ have bounded derivatives up to order two on bounded parameter sets.  Then, for every $R<\infty$, the maps
\[
        \omega\mapsto\psi_{\omega}(X),
        \qquad
        \phi\mapsto s_{\phi}(X)
\]
are $C^2$ on $\{\norm{\theta}_{\Theta}\le R\}$ with first and second derivative bounds uniform over $\norm{X}_{F}\le B_X$.  The same statement holds as a map into $L^2(P_X)$ when $X$ is almost surely bounded.
\end{lemma}

\begin{proof}
On $\{\norm{\theta}_{\Theta}\le R\}$ and $\norm{X}_{F}\le B_X$, every matrix product in \eqref{eq:single-head-attention} and \eqref{eq:ffn-map} is bounded by a constant depending only on $R,B_X,L,d,d_k,d_v,d_f$.  Matrix multiplication is polynomial in the entries of $\omega$ and is therefore smooth with bounded derivatives on bounded sets.  The rowwise softmax is $C^{\infty}$ and all derivatives are bounded on bounded subsets of its input space.  The activation $\sigma$ has bounded derivatives on the bounded interval reached by $A_{\omega}(X)W_1$.  The composition and product rules for Frechet derivatives imply that $\omega\mapsto\psi_{\omega}(X)$ is $C^2$ with uniform derivative bounds.  The assumed smoothness of the router score gives the corresponding conclusion for $\phi\mapsto s_{\phi}(X)$.  The $L^2(P_X)$ statement follows from the pointwise derivative bounds and dominated convergence.
\end{proof}

\subsubsection{Token-level cross-entropy loss}\label{subsec:cross-entropy-transformer}

For a finite training set $\{(X_r,Y_r)\}_{r=1}^{n}$ with $Y_{r,t}\in\{1,\ldots,C\}$, set
\begin{equation}\label{eq:H-logit}
        \Hlogit\coloneqq\left(\R^{L\times C}\right)^n
\end{equation}
with averaged inner product
\begin{equation}\label{eq:H-logit-inner}
        \ip{f}{g}_{\Hlogit}
        \coloneqq\frac1{nL}\sum_{r=1}^{n}\sum_{t=1}^{L}\ip{f_{r,t}}{g_{r,t}}_{\R^C}.
\end{equation}
The token-level cross-entropy risk is
\begin{equation}\label{eq:cross-entropy}
        \Rfunc_{\CE}(f)
        \coloneqq\frac1{nL}\sum_{r=1}^{n}\sum_{t=1}^{L}
        \left[-f_{r,t,Y_{r,t}}+\log\left(\sum_{c=1}^{C}\e^{f_{r,t,c}}\right)\right].
\end{equation}
This is the standard negative log-likelihood/cross-entropy objective used with transformer token logits; label smoothing replaces the one-hot target below by a fixed target distribution.  The transformer architecture and training objective with cross-entropy and label smoothing were introduced in \citet{vaswani2017attention}; sparse transformer mixture-of-experts layers and switch routing are developed in \citet{shazeer2017outrageously} and \citet{fedus2021switch}.  Theoretical motivation for router-driven specialization and the role of smoothing/noisy routing appears in \citet{chen2022towards}.

\begin{proposition}[Cross-entropy satisfies the Hilbert loss assumptions]\label{prop:ce-properties}
The map $\Rfunc_{\CE}:\Hlogit\to\R$ is $C^{\infty}$, lower bounded by $0$, and its Hilbert gradient with respect to \eqref{eq:H-logit-inner} is
\begin{equation}\label{eq:CE-gradient}
        (\nabla\Rfunc_{\CE}(f))_{r,t}=\softmax(f_{r,t})-e_{Y_{r,t}},
\end{equation}
where $e_y$ is the $y$th coordinate vector.  Moreover,
\begin{equation}\label{eq:CE-bounds}
        \norm{\nabla\Rfunc_{\CE}(f)}_{\Hlogit}\le\sqrt2,
        \qquad
        \norm{\nabla\Rfunc_{\CE}(f)-\nabla\Rfunc_{\CE}(g)}_{\Hlogit}\le
        \norm{f-g}_{\Hlogit}.
\end{equation}
The same conclusions hold with label smoothing, replacing $e_{Y_{r,t}}$ by any target vector in the probability simplex.
\end{proposition}

\begin{proof}
For one token define
\[
        \ell_y(z)=-z_y+\log\sum_{c=1}^{C}\e^{z_c}.
\]
Then
\[
        \nabla\ell_y(z)=\softmax(z)-e_y.
\]
The Hessian is
\[
        \nabla^2\ell_y(z)=\diag(s)-ss^{\top},
        \qquad s=\softmax(z),
\]
which is the covariance matrix of a categorical random variable with probability vector $s$.  Hence it is positive semidefinite and its operator norm is at most $1$.  Thus $\nabla\ell_y$ is $1$-Lipschitz.  Also, $\norm{\softmax(z)-e_y}_2\le\sqrt2$ because both vectors lie in the probability simplex.  Averaging over $(r,t)$ with the inner product \eqref{eq:H-logit-inner} proves \eqref{eq:CE-gradient} and \eqref{eq:CE-bounds}.  Nonnegativity follows from $\log\sum_c\e^{z_c}\ge z_y$.  Smoothness follows from smoothness of log-sum-exp.  The label-smoothed case is identical because the target vector remains in the probability simplex.
\end{proof}

\subsubsection{Hilbert feature maps for input-dependent gates}\label{subsec:gate-feature-maps}

Two smooth gate encodings are directly covered by \eqref{eq:J-rho-theta}.

\paragraph{Unnormalized smooth gates.}
Let $g_{\phi}(X)\in\R^L$ be a nonnegative smooth input-dependent gate.  Define the empirical logit-space feature map
\begin{equation}\label{eq:unnormalized-gate-feature}
        F_{\mathrm{un}}(\omega,\phi)
        \coloneqq\left(g_{\phi}(X_r)\odot\psi_{\omega}(X_r)\right)_{r=1}^{n}
        \in\Hlogit,
\end{equation}
where $\odot$ denotes tokenwise multiplication, broadcasting the scalar gate at token $t$ across the $C$ logits.  Then
\begin{equation}\label{eq:unnormalized-pop-output}
        m_{\rho}=\int_{\Theta}F_{\mathrm{un}}(\omega,
\phi)\dd\rho(\omega,
\phi)
\end{equation}
is an input-dependent gated population output.

\paragraph{Softmax-normalized gates over the expert distribution.}
A finite softmax MoE with particles $(\omega_i,\phi_i)_{i=1}^{N}$ has tokenwise output
\begin{equation}\label{eq:finite-softmax-moe}
        M_N(X)_{t,:}
        =\frac{\sum_{i=1}^{N}\e^{s_{\phi_i}(X)_t}\psi_{\omega_i}(X)_{t,:}}
        {\sum_{j=1}^{N}\e^{s_{\phi_j}(X)_t}}.
\end{equation}
For a probability distribution $\rho$ over $\Theta$, the corresponding population output is
\begin{equation}\label{eq:population-softmax-moe}
        M_{\rho}(X)_{t,:}
        =\frac{\int_{\Theta}\e^{s_{\phi}(X)_t}\psi_{\omega}(X)_{t,:}\dd\rho(\omega,\phi)}
        {\int_{\Theta}\e^{s_{\phi}(X)_t}\dd\rho(\omega,\phi)}.
\end{equation}
Although \eqref{eq:population-softmax-moe} is not a single unnormalized average of logits, it is exactly of the form \eqref{eq:J-rho-theta} after Hilbert-output augmentation.  Define
\begin{equation}\label{eq:H-aug}
        \Haug\coloneqq\Hlogit\oplus\left(\R^{L}\right)^n
\end{equation}
with product Hilbert inner product, and define
\begin{equation}\label{eq:augmented-feature}
        F_{\mathrm{soft}}(\omega,\phi)
        \coloneqq\left(
        \left(\e^{s_{\phi}(X_r)}\odot\psi_{\omega}(X_r)\right)_{r=1}^{n},
        \left(\e^{s_{\phi}(X_r)}\right)_{r=1}^{n}
        \right)
        \in\Haug.
\end{equation}
Then
\begin{equation}\label{eq:augmented-moment}
        m_{\rho}^{\mathrm{soft}}
        =\int_{\Theta}F_{\mathrm{soft}}(\omega,\phi)\dd\rho(\omega,\phi)
        =(N_{\rho},D_{\rho}).
\end{equation}
On the open set $D_{r,t}>0$, define
\begin{equation}\label{eq:Gamma-normalization}
        \Gamma(N,D)_{r,t,c}\coloneqq\frac{N_{r,t,c}}{D_{r,t}}.
\end{equation}
The softmax-gated risk is
\begin{equation}\label{eq:R-softgate}
        \Rfunc_{\mathrm{softgate}}(N,D)
        \coloneqq \Rfunc_{\CE}(\Gamma(N,D)).
\end{equation}
Therefore
\begin{equation}\label{eq:J-softgate}
        \tf(\rho)=\Rfunc_{\mathrm{softgate}}\left(\int_{\Theta}F_{\mathrm{soft}}(\theta)\dd\rho(\theta)
        \right)
\end{equation}
is exactly of the Hilbert-valued form \eqref{eq:J-rho-theta}.  For the empirical measure $\rho^N=N^{-1}\sum_i\delta_{(\omega_i,\phi_i)}$, the factors $1/N$ cancel between numerator and denominator, and \eqref{eq:J-softgate} gives exactly \eqref{eq:finite-softmax-moe}.

\begin{lemma}[Local smoothness of the softmax-gate feature and loss]\label{lem:softgate-smoothness}
Assume $\norm{X_r}_{F}\le B_X$ for all training inputs, and assume the expert and router maps satisfy Lemma \ref{lem:transformer-local-smoothness}.  On each bounded parameter ball $\norm{\theta}_{\Theta}\le R$, the maps $F_{\mathrm{un}}$ and $F_{\mathrm{soft}}$ are $C^2$ with bounded first and second derivatives.  In the normalized case, if $D_{r,t}\ge\delta>0$ on the moment set reached by the trajectory, then $\Rfunc_{\mathrm{softgate}}$ is $C^2$ with bounded gradient and bounded Hessian on that set.
\end{lemma}

\begin{proof}
The maps $\psi_{\omega}(X_r)$ and $s_{\phi}(X_r)$ are $C^2$ with bounded derivatives on bounded parameter balls by Lemma \ref{lem:transformer-local-smoothness}.  The exponential map has bounded derivatives on bounded score intervals.  Products of $\e^{s_{\phi}}$ with $\psi_{\omega}$ therefore have bounded derivatives up to order two.  This proves the statement for $F_{\mathrm{un}}$ and $F_{\mathrm{soft}}$.

For the normalized risk, $\Gamma(N,D)=N/D$ is $C^{\infty}$ on $D_{r,t}>0$.  On a set where $D_{r,t}\ge\delta$ and $N,D$ are bounded, its first and second derivatives are bounded by constants depending on $\delta$ and the bounds on $N,D$.  Proposition \ref{prop:ce-properties} gives bounded first and second derivatives of $\Rfunc_{\CE}$ on logit space.  The chain rule gives the stated bounds for $\Rfunc_{\mathrm{softgate}}=\Rfunc_{\CE}\circ\Gamma$.
\end{proof}

\subsubsection{Hard top-$k$ routing and noisy routing}\label{subsec:hard-routing}

Exact hard top-$k$ or switch routing uses the discontinuous map that selects the largest router scores.  This map is not covered by the smooth ODE, Hamiltonian, curvature, and propagation-of-chaos theorems above.  There are two mathematically consistent ways to connect hard routing to the present theory.

First, hard routing can be replaced by a smooth relaxation, such as softmax with temperature, Gumbel-softmax, or a differentiable expected router.  In the Gumbel-max case, the expected top-1 selection probabilities are precisely softmax probabilities, so the augmented feature \eqref{eq:augmented-feature} applies directly.  More generally, if the smoothed expected router weight is a $C^2$ function of $(X,\theta)$ on bounded sets, then it can be included in $F$ exactly as in \eqref{eq:unnormalized-gate-feature} or \eqref{eq:augmented-feature}.

Second, exact hard routing can be treated as a nonsmooth limit.  The limiting evolution is then a differential inclusion, analogous to the hard-Muon limit in Section \ref{sec:hard-limit-gated}.  At score ties, the router subdifferential or selection correspondence is set-valued.  A smooth-flow convergence proof must therefore be replaced by compactness plus graph-convergence arguments for the router selection map.

The MoE analysis of \citet{chen2022towards} emphasizes that sparse top-1 routing is discontinuous and that injected random noise smooths the routing probabilities.  Their smoothing lemma gives Lipschitz dependence of the route probabilities on the router scores when the noise density is bounded.  Such a smoothed expected router is compatible with the present Hilbert-valued framework whenever the resulting expected gate is used as the gate component of $F$.

\subsubsection{Specialization theorem}\label{subsec:specialization-theorem}

\begin{theorem}[Transformer MoE specialization with router parameters included]\label{thm:transformer-specialization-gated}
Consider the extended parameter space \eqref{eq:Theta-full-transformer} with product Frobenius geometry \eqref{eq:Theta-transformer-norm} and blockwise regularized Muon map \eqref{eq:Theta-transformer-Orth}.  Consider the single-head attention-plus-FFN expert \eqref{eq:single-head-attention}-\eqref{eq:expert-logit}, a smooth router score map \eqref{eq:router-score}, and the token-level cross-entropy loss \eqref{eq:cross-entropy}.  Let $F$ be either the unnormalized gate feature \eqref{eq:unnormalized-gate-feature} or the augmented softmax-normalized feature \eqref{eq:augmented-feature}.  Assume:
\begin{itemize}
        \item[(T1)] the training inputs are bounded, $\norm{X_r}_{F}\le B_X$;
        \item[(T2)] the activation and router score maps are $C^2$ in the parameters on bounded parameter sets;
        \item[(T3)] in the normalized-gate case, the denominator satisfies $D_{\rho,r,t}\ge\delta>0$ on the trajectory;
        \item[(T4)] the trajectory remains in a bounded parameter and momentum region for the time interval or asymptotic regime under consideration;
        \item[(T5)] for exponential convergence, the PL and upper-gradient assumptions \eqref{eq:PL-gated}-\eqref{eq:upper-gradient-gated} hold along the trajectory.
\end{itemize}
Then all constructions and results in Sections \ref{sec:gated-hilbert-functional}-\ref{sec:hard-limit-gated} apply to this transformer MoE.  In particular:
\begin{enumerate}
        \item the first variation is
        \[
            \frac{\delta\tf}{\delta\rho}(\rho)(\theta)=\ip{\nabla\Rfunc(m_{\rho})}{F(\theta)}_{\cH},
        \]
        with $\cH=\Hlogit$ in the unnormalized case and $\cH=\Haug$ in the normalized case;
        \item the Wasserstein force on the expert-router tuple is
        \[
            a_{\rho}(\theta)=DF(\theta)^{*}\nabla\Rfunc(m_{\rho})\in\Theta;
        \]
        \item the finite-$N$ regularized Muon scheme updates both expert and router blocks by
        \[
            P_{i,k+1}=\beta P_{i,k}+(1-\beta)a_i^N(\boldsymbol\theta_k),
            \qquad
            \theta_{i,k+1}=\theta_{i,k}-\eta\Ortheps^{\Theta}(P_{i,k+1});
        \]

        \item under the inertial scaling \eqref{eq:inertial-scaling-gated}, the finite-particle ODE limit, mean-field PDE, Hamiltonian formulation, dissipation identity, and hard-Muon subsequential limit hold as stated above on the bounded region specified by (T1)-(T4), while, propagation of chaos, as stated above, holds under the global Lipschitz Assumption \ref{ass:global-hilbert}, or under an explicitly stated localized compact-support version with uniform support bounds;
        
        \item under (T5) and the curvature assumptions, the convergence estimate \eqref{eq:convergence-estimate-gated} holds for the transformer MoE cross-entropy objective.
\end{enumerate}
\end{theorem}

\begin{proof}
The parameter space \eqref{eq:Theta-full-transformer} is a finite product of matrix spaces.  Proposition \ref{prop:extended-fenchel-mirror} therefore gives the product Fenchel duality and mirror-map interpretation on the full expert-router space.  Proposition \ref{prop:ce-properties} verifies the bounded-gradient and Lipschitz-gradient properties of token-level cross-entropy on logit space.  Lemma \ref{lem:transformer-local-smoothness} gives local $C^2$ regularity of the expert and router score maps on bounded parameter regions.  Lemma \ref{lem:softgate-smoothness} gives local $C^2$ regularity of the gated Hilbert feature maps and, in the normalized case, of the normalization-composed risk under the denominator lower bound.

These statements verify Assumption \ref{ass:localized-hilbert} and Assumption \ref{ass:second-order-gated} on the region reached by the trajectory.  Assumption (T4) provides the bounded trajectory region needed for the localized finite-horizon versions of well-posedness and ODE convergence.  Thus the first variation, particle gradient, discrete scheme, finite-$N$ ODE limit, mean-field equation, Hamiltonian formulation, dissipation identity, and hard-limit statements follow from Propositions \ref{prop:first-var-gated}, \ref{prop:particle-gradient-gated}, and Theorems \ref{thm:finite-ode-limit-gated}, \ref{thm:mean-field-gated}, \ref{thm:ham-form-gated}, \ref{thm:dissipation-gated}, and \ref{thm:hard-limit-gated}, using their localized forms where appropriate.  If global constants are imposed instead of localized constants, Theorem \ref{thm:poc-gated} gives propagation of chaos.  Under (T5), Theorem \ref{thm:convergence-gated} gives \eqref{eq:convergence-estimate-gated}.
\end{proof}

\subsubsection{Consequences for the gated transformer-MoE model}\label{sec:summary-gated}

A distribution over transformer experts with input-dependent routing is represented as a distribution over extended particles
\[
        \theta_i=(Q_i,K_i,V_i,O_i,W_{1,i},W_{2,i},\phi_i)\in\Thetaexp\times\Thetagate.
\]
The empirical law is
\[
        \rho_{\boldsymbol\theta}^{N}=\frac1N\sum_{i=1}^{N}\delta_{\theta_i}.
\]
For smooth unnormalized routing, the model output is encoded by $m_{\rho}=\int F_{\mathrm{un}}(\theta)\dd\rho(\theta)$.  For softmax-normalized routing across experts, the numerator and denominator are encoded by the augmented Hilbert moment $m_{\rho}^{\mathrm{soft}}=\int F_{\mathrm{soft}}(\theta)\dd\rho(\theta)$ and the output is recovered by the smooth normalization map $\Gamma$.  Thus input dependence of the gate is not an obstruction: it is part of the input-indexed Hilbert feature map.  The only obstruction to the smooth theory is discontinuity of exact hard top-$k$ selection; that case requires smoothing or a nonsmooth differential-inclusion treatment.

The mirror map and Fenchel conjugate are defined on the full space $\Theta$ rather than only on the expert matrices.  Hence the regularized Muon step is a genuine product-space mirror step for the pair consisting of expert and router parameters.  The Hamiltonian probability flow evolves on phase space $\Theta\times\Theta$ and dissipates according to
\[
        \frac{\dd}{\dd t}\Ham_{\eps,\gamma}(\mu_t)
        =-\gamma\int_{\Theta\times\Theta}\ip{p}{\Ortheps^{\Theta}(p)}_{\Theta}\dd\mu_t(\theta,p)\le0.
\]
Under the standard PL, upper-gradient, bounded-momentum, and bounded-curvature assumptions, the objective gap decays exponentially according to \eqref{eq:convergence-estimate-gated}.  Finite particles converge to the mean-field law by propagation of chaos under the global Lipschitz version of the assumptions, and the hard Muon dynamics is recovered as a subsequential nonsmooth limit as $\eps\downarrow0$.

\subsection{Synthetic experiments on finite-particle Muon dynamics}
\label{sec:synthetic-experiments}

We use two deterministic synthetic experiment classes to test the finite-particle dynamics developed above.  The goal is not to benchmark large-scale training, but to isolate the phenomena predicted by the Hamiltonian formulation. All reported runs are full-batch deterministic runs with seed $0$, double precision, zero initial momentum, $h=\eta=0.01$, $\gamma=1$, and therefore $\beta=1-\gamma h=0.99$.  The modified Lyapunov diagnostic uses $\alpha=0.01$.

For each setting we compare four update rules.  After the common momentum update
\begin{equation}
    P_{i,k+1}=\beta P_{i,k}+(1-\beta)a_{i,k},
\end{equation}
we update
\begin{equation}
    \theta_{i,k+1}=\theta_{i,k}-\eta G(P_{i,k+1}).
\end{equation}
The regularized Muon choice is $G(P)=\Ortheps(P)$, applied blockwise on product spaces.  We also include the ideal hard polar factor $G(P)=\Orth(P)$, a Newton-Schulz approximation to the polar factor using five fifth-order iterations, and the Euclidean momentum baseline $G(P)=P$.  For regularized Muon, the plotted $\varepsilon$ values are chosen to show both the smooth regime and the nearly-hard regime.

\subsubsection{Experiment 1: matrix mean matching}
\label{subsec:exp1-setup}

The first experiment is a linear mean-matching problem on a single matrix block.  The parameter space is $\Theta=\X=\R^{16\times 8}$ with Frobenius inner product.  We set
\begin{equation}
    F(W)=W,
    \qquad
    \Rfunc(A)=\frac12\norm{A-\bar W_\star}_F^2,
    \qquad
    \bar W_\star=\frac1M\sum_{j=1}^M W_{j,\star} .
    \label{eq:exp1-F-R}
\end{equation}
Thus
\begin{equation}
    J(\rho)=\frac12\norm{\int W\,\dd\rho(W)-\bar W_\star}_F^2
\end{equation}
and the finite-particle objective is
\begin{equation}
    J_N(W_1,\ldots,W_N)
    =\frac12\norm{\frac1N\sum_{i=1}^N W_i-\bar W_\star}_F^2 .
    \label{eq:exp1-JN}
\end{equation}
The mean-field force has the closed form
\begin{equation}
    a_{i,k}=\frac1N\sum_{\ell=1}^N W_{\ell,k}-\bar W_\star,
    \qquad i=1,\ldots,N,
    \label{eq:exp1-force}
\end{equation}
so all particles see the same force.  This makes the experiment a clean test of the momentum-to-update map rather than of modeling complexity.  Since $J_\star=0$ is attainable, the plotted objective is also the objective gap.

We report two target/approximation-particle choices: a single target particle with an overparameterized $10$-particle approximation, $(M,N)=(1,10)$, and a four-particle target with a $32$-particle approximation, $(M,N)=(4,32)$.  The target and initial matrices are sampled from centered Gaussian ensembles scaled by $1/\sqrt{8}$.  Each run is executed for $10{,}000$ iterations.  For $(M,N)=(1,10)$ the regularized Muon values are $\varepsilon\in\{1,3\cdot 10^{-2},10^{-3},10^{-8}\}$; for $(M,N)=(4,32)$ they are $\varepsilon\in\{1,10^{-1},10^{-2},10^{-4}\}$.

\subsubsection{Experiment 2: product-space teacher-student particles}
\label{subsec:exp2ab-setup}

The second experiment is a nonlinear teacher-student problem on a product matrix space.  Each particle is
\begin{equation}
    \theta_i=(A_i,B_i)\in \Theta=\R^{p\times r}\times \R^{r\times d},
    \qquad (d,r,p)=(10,6,4).
\end{equation}
The product-space inner product is
\begin{equation}
    \ip{(A,B)}{(\widetilde A,\widetilde B)}_{\Theta}
    =\ip{A}{\widetilde A}_F+\ip{B}{\widetilde B}_F .
\end{equation}
We draw frozen inputs $x_s\in\R^d$, $s=1,\ldots,S$, with $S=320$, and generate a frozen teacher output
\begin{equation}
    y_s=f_\star(x_s),
    \qquad
    f_\star(x)=\frac1M\sum_{j=1}^M A_{j,\star}
    \tanh\!\left(\frac{B_{j,\star}x}{\sqrt d}\right).
    \label{eq:exp2-teacher}
\end{equation}
The teacher matrices are sampled with $A_{j,\star}$ scaled by $1/\sqrt r$ and $B_{j,\star}$ scaled by $1/\sqrt d$.  The feature map is the empirical prediction tensor
\begin{equation}
    F(A,B)=\left(A\tanh\!\left(\frac{Bx_s}{\sqrt d}\right)\right)_{s=1}^S
    \in \R^{S\times p},
    \label{eq:exp2-F}
\end{equation}
and the outer loss is the squared empirical $L^2$ loss
\begin{equation}
    \Rfunc(Z)=\frac1{2Sp}\sum_{s=1}^S \norm{Z_s-y_s}_2^2 .
    \label{eq:exp2-R}
\end{equation}
Consequently,
\begin{equation}
    J_N((A_i,B_i)_{i=1}^N)
    =\frac1{2Sp}\sum_{s=1}^S
    \norm{\frac1N\sum_{i=1}^N A_i\tanh\!\left(\frac{B_i x_s}{\sqrt d}\right)-y_s}_2^2 .
    \label{eq:exp2-JN}
\end{equation}
The Muon map is applied separately to the $A$- and $B$-momentum blocks.  Unlike Experiment~1, this objective is nonconvex and the particles do not share a common force.

We use two particle choices.  The first is an overparameterized approximation to a small teacher, $(M,N)=(3,12)$.  The second is a matched-particle comparison, $(M,N)=(10,10)$.  In both cases the student is initialized randomly, rather than near the teacher, with initial scale $0.1$.  Each run is executed for $2{,}000$ iterations.  For $(M,N)=(3,12)$ the regularized Muon values are $\varepsilon\in\{10^{-1},10^{-3},10^{-4},10^{-5}\}$; for $(M,N)=(10,10)$ they are $\varepsilon\in\{10^{-1},10^{-3},10^{-5},10^{-7}\}$.

\begin{table}[t]
\centering
\small
\caption{Experimental configurations used in the main figures.  $M$ is the number of target/teacher particles and $N$ is the number of approximation particles.  All runs use $h=\eta=0.01$, $\gamma=1$, $\beta=0.99$, $\alpha=0.01$, full-batch gradients, and seed $0$.}
\label{tab:experiment-configurations}
\begin{tabular}{@{}llcccl@{}}
\toprule
Class & Parameter space & $M$ & $N$ & Iters. & Regularized $\varepsilon$ values shown \\
\midrule
Exp.~1 & $\R^{16\times 8}$ & $1$ & $10$ & $10{,}000$ & $1,\;3\cdot 10^{-2},\;10^{-3},\;10^{-8}$ \\
Exp.~1 & $\R^{16\times 8}$ & $4$ & $32$ & $10{,}000$ & $1,\;10^{-1},\;10^{-2},\;10^{-4}$ \\
Exp.~2 & $\R^{4\times 6}\times \R^{6\times 10}$ & $3$ & $12$ & $2{,}000$ & $10^{-1},\;10^{-3},\;10^{-4},\;10^{-5}$ \\
Exp.~2 & $\R^{4\times 6}\times \R^{6\times 10}$ & $10$ & $10$ & $2{,}000$ & $10^{-1},\;10^{-3},\;10^{-5},\;10^{-7}$ \\
\bottomrule
\end{tabular}
\end{table}

\subsubsection{Results and interpretation}
\label{subsec:experiment-results}

Figure~\ref{fig:exp1-main} shows the objective and Hamiltonian traces for the matrix mean-matching experiment.  The main qualitative point is that the smooth regularization changes the finite-step behavior near equilibrium.  The hard polar map and the Newton-Schulz polar approximation keep an essentially fixed normalized update direction whenever the momentum is nonzero; as a result, for this finite step size they settle into a small residual floor around $10^{-3}$ in objective value.  In contrast, the regularized map satisfies $\Ortheps(P)\approx P/\varepsilon$ when $\norm{P}$ is small relative to $\varepsilon$, so the update size shrinks near the optimum and the objective can decay to numerical precision.  The Euclidean momentum baseline also reaches numerical precision here because Experiment~1 is a quadratic mean-matching problem with a shared force; thus this experiment should be read as a diagnostic for the regularization mechanism, not as evidence that Muon is always faster than Euclidean momentum.

\begin{figure}[t]
\centering
\begin{minipage}{0.49\linewidth}
\centering
\includegraphics[width=\linewidth]{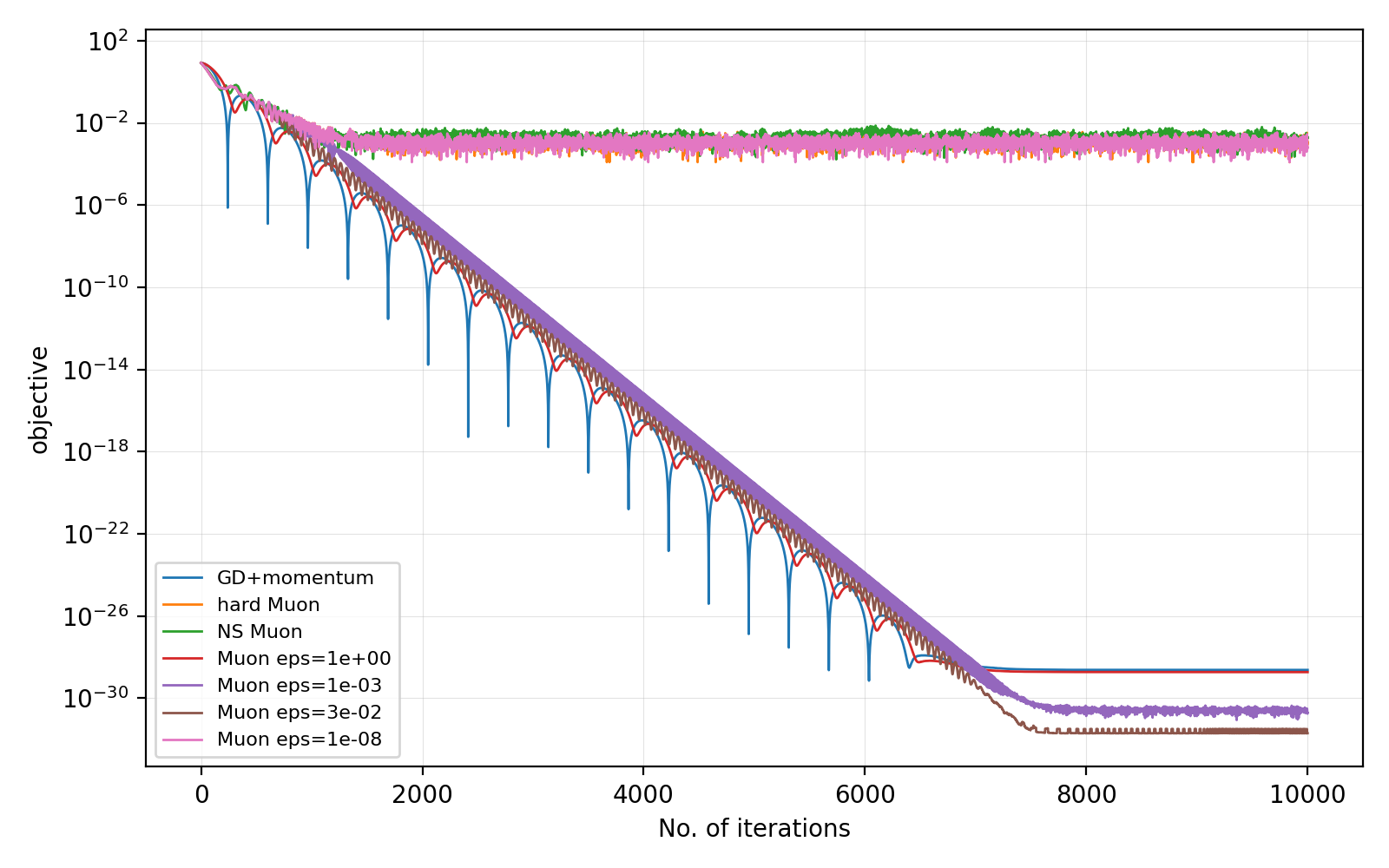}
\end{minipage}\hfill
\begin{minipage}{0.49\linewidth}
\centering
\includegraphics[width=\linewidth]{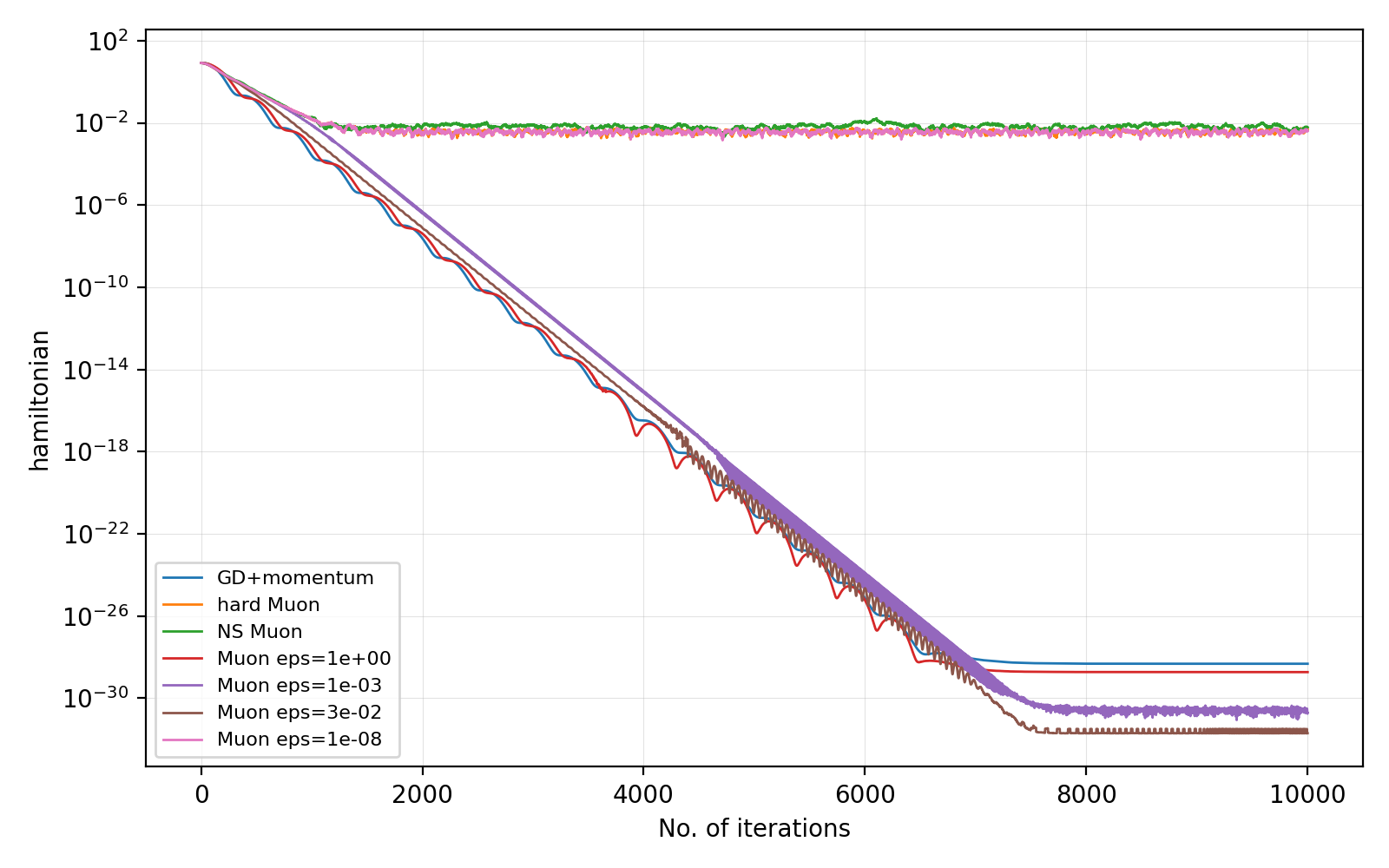}
\end{minipage}

\vspace{0.4em}
\begin{minipage}{0.49\linewidth}
\centering
\includegraphics[width=\linewidth]{Images/Exp1/4targetparticle/objective_experiment_exp1__matrix_shape_16x8__target_particles_4__particles_32__h_0.01__gamma_1.0.png}
\end{minipage}\hfill
\begin{minipage}{0.49\linewidth}
\centering
\includegraphics[width=\linewidth]{Images/Exp1/4targetparticle/hamiltonian_experiment_exp1__matrix_shape_16x8__target_particles_4__particles_32__h_0.01__gamma_1.0.png}
\end{minipage}
\caption{Experiment~1: matrix mean matching.  Top: $(M,N)=(1,10)$.  Bottom: $(M,N)=(4,32)$.  Left panels show $J_N$ and right panels show the Hamiltonian $K+\gamma J_N$ on a logarithmic scale.  Smooth regularized Muon with moderate $\varepsilon$ reaches numerical precision, while the hard and Newton-Schulz polar directions plateau at a finite-step residual floor.}
\label{fig:exp1-main}
\end{figure}

Figure~\ref{fig:exp2-main} shows the nonlinear product-space teacher-student experiment.  This setting is more representative of the product-space theory: the force depends on each particle, the parameter has two matrix blocks, and the loss is nonconvex.  The Euclidean momentum baseline barely decreases the objective over the plotted horizon.  Hard Muon and Newton-Schulz Muon reduce the loss by several orders of magnitude, confirming that spectral normalization is useful in this product-space particle model.  The regularized runs display the expected $\varepsilon$ tradeoff.  A large value such as $\varepsilon=10^{-1}$ is too smooth and behaves closer to a damped Euclidean update, while very small values approach the hard polar map.  Intermediate small values, especially $\varepsilon=10^{-5}$ in the plotted neural settings, give the best final losses among the displayed regularized runs.  The Hamiltonian panels follow the same qualitative decay pattern and provide the energy diagnostic predicted by the damped Hamiltonian identity.

\begin{figure}[t]
\centering
\begin{minipage}{0.49\linewidth}
\centering
\includegraphics[width=\linewidth]{Images/Exp2/3targetparticle/objective_experiment_exp2_ab_particles__target_particles_3__particles_12__h_0.01__gamma_1.0.png}
\end{minipage}\hfill
\begin{minipage}{0.49\linewidth}
\centering
\includegraphics[width=\linewidth]{Images/Exp2/3targetparticle/hamiltonian_experiment_exp2_ab_particles__target_particles_3__particles_12__h_0.01__gamma_1.0.png}
\end{minipage}

\vspace{0.4em}
\begin{minipage}{0.49\linewidth}
\centering
\includegraphics[width=\linewidth]{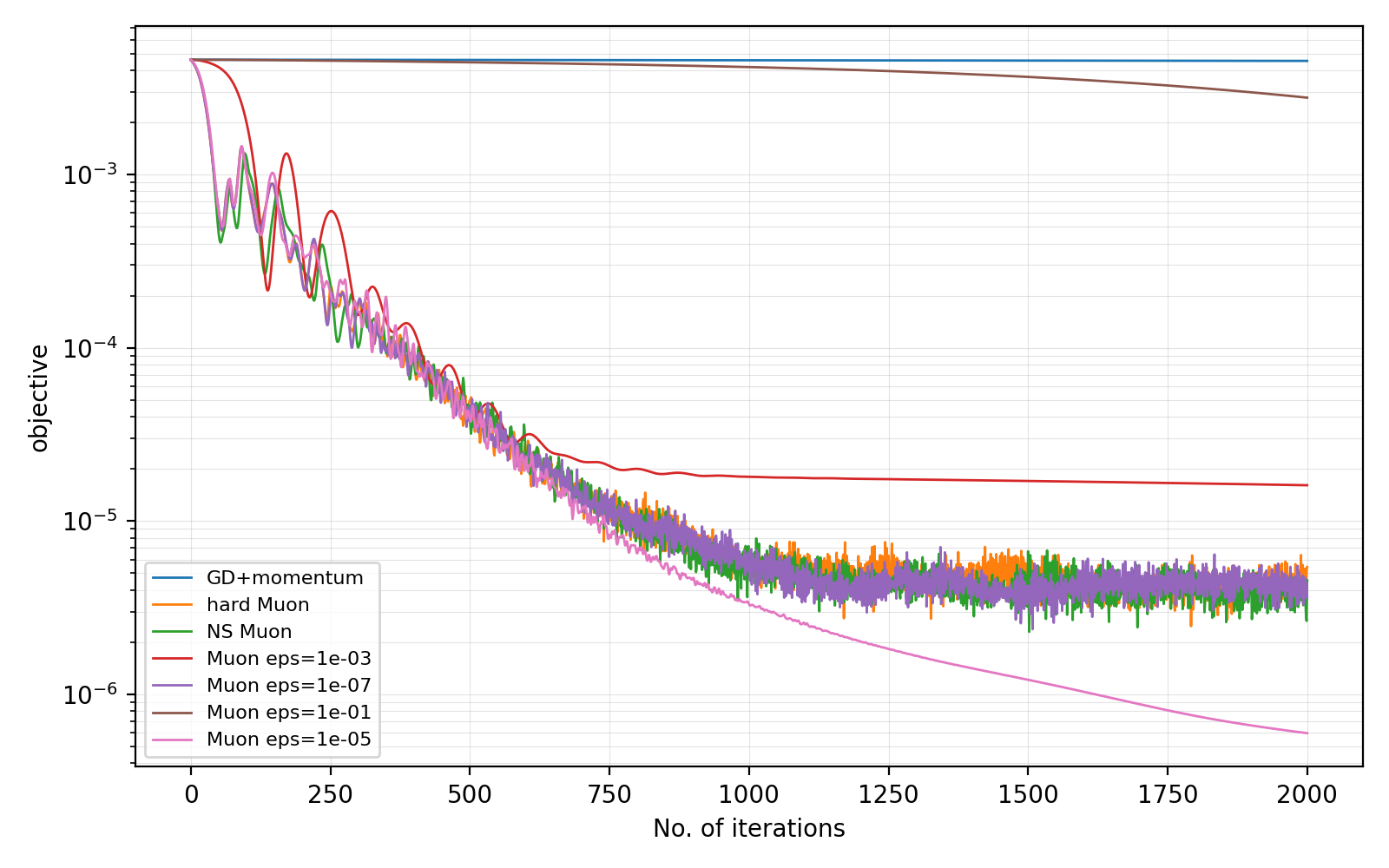}
\end{minipage}\hfill
\begin{minipage}{0.49\linewidth}
\centering
\includegraphics[width=\linewidth]{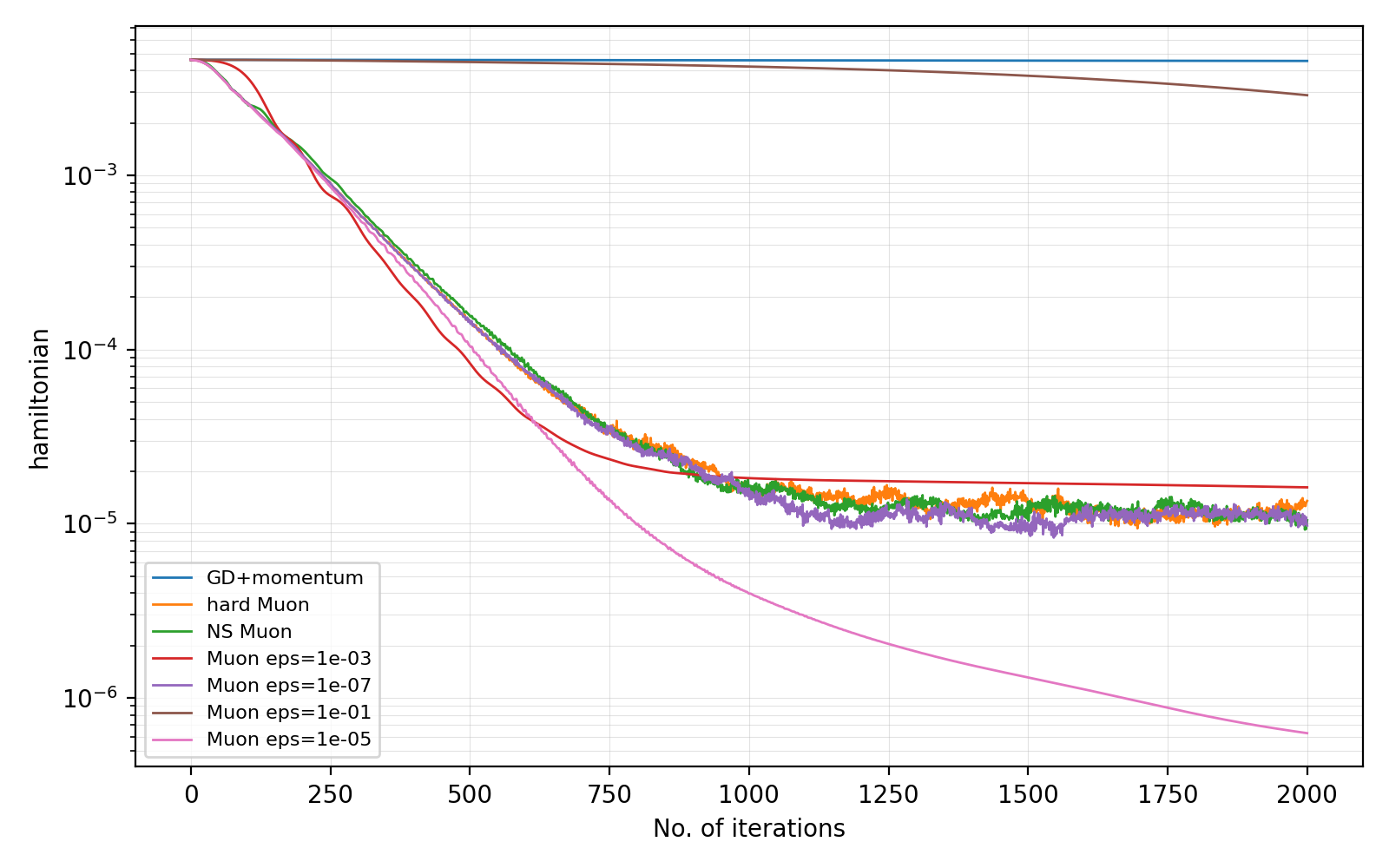}
\end{minipage}
\caption{Experiment~2: product-space teacher-student particles with $(d,r,p)=(10,6,4)$ and $S=320$ frozen inputs.  Top: overparameterized case $(M,N)=(3,12)$.  Bottom: matched-particle case $(M,N)=(10,10)$.  Left panels show $J_N$ and right panels show $K+\gamma J_N$.  Spectral Muon-type updates substantially outperform Euclidean momentum in this nonlinear product-space setting; the regularized map interpolates between overly smooth Euclidean-like behavior and the hard polar regime.}
\label{fig:exp2-main}
\end{figure}

\begin{table}[t]
\centering
\scriptsize
\caption{Final objective values at the last plotted iteration.  The ``best regularized'' column selects the best value among the plotted $\varepsilon$ values for that row.  These are single-seed deterministic diagnostics, so the table should be interpreted as a mechanistic comparison rather than a statistical benchmark.}
\label{tab:final-objectives}
\begin{tabular}{@{}lccccc@{}}
\toprule
Setting & GD+momentum & hard Muon & NS Muon & best regularized Muon & best $\varepsilon$ \\
\midrule
Exp.~1, $(M,N)=(1,10)$   & $2.4\cdot 10^{-29}$ & $1.1\cdot 10^{-3}$ & $1.9\cdot 10^{-3}$ & $2.0\cdot 10^{-32}$ & $3\cdot 10^{-2}$ \\
Exp.~1, $(M,N)=(4,32)$   & $8.5\cdot 10^{-31}$ & $1.1\cdot 10^{-3}$ & $1.4\cdot 10^{-3}$ & $1.5\cdot 10^{-33}$ & $10^{-2}$ \\
Exp.~2, $(M,N)=(3,12)$   & $1.1\cdot 10^{-2}$  & $4.4\cdot 10^{-6}$ & $2.8\cdot 10^{-6}$ & $1.3\cdot 10^{-7}$  & $10^{-5}$ \\
Exp.~2, $(M,N)=(10,10)$  & $4.5\cdot 10^{-3}$  & $5.4\cdot 10^{-6}$ & $4.6\cdot 10^{-6}$ & $6.0\cdot 10^{-7}$  & $10^{-5}$ \\
\bottomrule
\end{tabular}
\end{table}

The results support three conclusions.  First, the regularized operator is not merely a numerical perturbation of hard Muon: at finite step size it removes the non-vanishing update floor of the hard polar map near equilibrium.  Second, in the nonlinear product-space problem, the spectral Muon geometry is substantially more effective than raw Euclidean momentum.  Third, $\varepsilon$ has an interpretable role.  Large $\varepsilon$ over-smooths the map; extremely small $\varepsilon$ behaves like hard Muon; intermediate small $\varepsilon$ preserves the spectral acceleration while still allowing a smooth Hamiltonian interpretation.  These observations are consistent with the theoretical picture in which $\varepsilon>0$ supplies the smooth mirror map and the Hamiltonian dissipation diagnostic, while the hard Muon dynamics are recovered only as a singular limit.
\FloatBarrier
\subsection{Consequences of this probabilistic view and future directions}
The probability-flow perspective gives several concrete consequences that are difficult to see from the discrete update alone. First, it identifies the correct energy. For Euclidean gradient flow, the objective decreases directly. For Muon with momentum, the objective can increase transiently because momentum can be misaligned with the force. The Hamiltonian $K_t+\gamma J\left(\rho_t\right)$ is the quantity that dissipates exactly. This distinction is not a weakness of the method- it is the signature of acceleration.

Second, the theory explains the role of $\varepsilon$. A positive $\varepsilon$ gives a smooth, Lipschitz vector field and therefore a classical ODE/PDE theory. The hard optimizer is recovered as $\varepsilon \downarrow 0$, but the limiting equation is set-valued at rank-deficient momenta. Thus $\varepsilon$ has a dual interpretation - algorithmically it is a soft orthogonalization parameter, and analytically it is the regularity parameter that selects a smooth Hamiltonian flow before taking a nonsmooth limit.

Third, the particle formulation makes the mean-field limit explicit. The empirical law $N^{-1} \sum_i \delta_{\left(W_i, P_i\right)}$ is not an auxiliary construction, it is the object whose limit is the phase-space law $\mu_t$. Propagation of chaos then says that any fixed number of particles behaves asymptotically like independent nonlinear characteristics. This is the mathematical bridge between a finite collection of matrix blocks and a population optimization model.

Fourth, the alignment term $C_t$ gives an interpretable measure of whether momentum helps or hurts descent. Positive $C_t$ means the momentum is aligned with the force in the sense used by the Lyapunov proof. Negative $C_t$ means the system has kinetic energy but is initially moving against the current force. The convergence proof does not assume this term is always positive, instead, it subtracts a small multiple of $C_t$ and controls the resulting Lyapunov function.

Fifth, the transformer extension shows that input-dependent routing is compatible with probability measures over parameters. The route score does not have to be external to the measure. Once the particle is enlarged to $\theta=(\omega, \varphi)$, both expert outputs and router scores are ordinary components of a Hilbert-valued feature map. Softmax normalization across experts is represented by augmenting the feature moment with numerator and denominator components. This keeps the objective in the form $R\left(\int F \mathrm{~d} \rho\right)$ and lets the same first-variation and Hamiltonian machinery apply.

Several directions remain open. The most important is to verify or replace the PL and upper-gradient assumptions in concrete transformer regimes. A second direction is to prove uniqueness or selection principles for the hard-Muon differential inclusion at rank-deficient momenta. A third is to analyze stochastic mini-batch noise, finite Newton-Schulz approximation error, and adaptive choices of $\varepsilon$ within the Hamiltonian framework.


\end{document}